\documentclass{article}

\usepackage{arxiv}

\usepackage[utf8]{inputenc} 
\usepackage[T1]{fontenc}    
\usepackage{hyperref}       
\usepackage{url}            
\usepackage{booktabs}       
\usepackage{amsfonts}       
\usepackage{nicefrac}       
\usepackage{microtype}      
\usepackage{lipsum}
\usepackage{times}
\usepackage[ruled]{algorithm2e}
\usepackage{multicol}
\usepackage{amsmath,amssymb,amsthm}
\usepackage{natbib}
\setcitestyle{authoryear,open={(},close={)}}
\usepackage{dsfont}
\usepackage{xpatch}
\usepackage{graphicx}

\xpatchcmd{\proof}{\itshape}{\normalfont\proofnamefont}{}{}
\newcommand{\proofnamefont}{\bfseries}

\newtheorem{proposition}{Proposition}
\newtheorem{corollary}[proposition]{Corollary}
\newtheorem{lemma}[proposition]{Lemma}

\newtheorem{theorem}[proposition]{Theorem}

\newtheorem{assumption}{Assumption}

\DeclareMathOperator{\Tr}{Tr}

\makeatletter
\let\@fnsymbol\@arabic
\makeatother
    
\title{Stochastic Online Optimization using Kalman Recursion}

\author{%
 Joseph {de Vilmarest}
 \thanks{Sorbonne Université, CNRS, LPSM, F-75005 Paris, France} \ \textsuperscript{\textit{,}}\thanks{EDF R\&D, Palaiseau, France}
 \\ 
 \texttt{joseph.de\_vilmarest@upmc.fr}
\And Olivier Wintenberger
\footnotemark[1]\\
 \texttt{olivier.wintenberger@upmc.fr}
}

\newcommand{\R}{\mathbb{R}}

\begin{document}

\maketitle

\begin{abstract}%
We study the Extended Kalman Filter in constant dynamics, offering a bayesian perspective of stochastic optimization. We obtain high probability bounds on the cumulative excess risk in an unconstrained setting. In order to avoid any projection step we propose a two-phase analysis. First, for linear and logistic regressions, we prove that the algorithm enters a local phase where the estimate stays in a small region around the optimum. We provide explicit bounds with high probability on this convergence time. Second, for generalized linear regressions, we provide a martingale analysis of the excess risk in the local phase, improving existing ones in bounded stochastic optimization. The EKF appears as a parameter-free online  algorithm with $O(d^2)$ cost per iteration that optimally solves some unconstrained optimization problems.
\end{abstract}


\section{Introduction}
The optimization of convex functions is a long-standing problem with many applications. In supervised machine learning it frequently arises in the form of the prediction of an observation $y_t\in\R$ given explanatory variables $X_t\in\R^d$. The aim is to minimize a cost depending on the prediction and the observation. We focus in this article on linear predictors, hence the loss function is of the form $\ell(y_t,\theta^TX_t)$.

Two important settings have emerged in order to analyse learning algorithms. In the online setting $(X_t,y_t)$ may be set by an adversary. The assumption required is boundedness and the goal is to bound the regret (cumulative excess loss compared to the optimum). In the stochastic setting $(X_t,y_t)$ is i.i.d. thus allowing to define the risk $L(\theta)=\mathbb{E}[\ell(y,\theta^TX)]$. The goal is to bound the excess risk.
In this article we focus on the cumulative excess risk and we obtain non-asymptotic bounds holding with high probability. Our bounds hold simultaneously for any horizon, that is, we control the whole trajectory with high probability. Furthermore, our bounds on the cumulative risk all lead to a similar bound on the excess risk at any step for the averaged version of the algorithm.

Due to its low computational cost the Stochastic Gradient Descent of~\cite{robbins1951stochastic} has been widely used, along with its equivalent in the online setting, the Online Gradient Descent~\citep{zinkevich2003online} and a simple variant where the iterates are averaged~\citep{ruppert1988efficient,polyak1992acceleration}. More recently \cite{bach2013non} provided a sharp bound in expectation on the excess risk for a two step procedure that has been extended to the average of Stochastic Gradient Descent (SGD) with a constant step size~\citep{bach2014adaptivity}. 
Second-order methods based on stochastic versions of Newton-Raphson algorithm have been developed in order to converge faster in iterations, although with a bigger computational cost per iteration~\citep{hazan2007logarithmic}. 

In order to obtain a parameter-free second-order algorithm we apply a bayesian perspective, seeing the loss as a negative log-likelihood and approximating the maximum-likelihood estimator at each step. We get a state-space model interpretation of the optimization problem: in a well-specified setting the space equation is $y_t\sim p_{\theta_t}(\cdot\mid X_t)\propto \exp(-\ell(\cdot,\theta_t^TX_t))$ with $\theta_t\in\mathbb{R}^d$ and the state equation defines the dynamics of the state $\theta_t$. The stochastic convex optimization setting corresponds to a degenerate constant state-space model $\theta_t=\theta_{t-1}$ called static. As usual in State-Space models, the optimization is realized with the Kalman Filter~\citep{kalman1961new} for the quadratic loss and the Extended Kalman Filter~\citep{fahrmeir1992posterior} in a more general case.
A correspondence has recently been made by~\cite{ollivier2018online} between the static EKF and the online natural gradient~\citep{amari1998natural}. This motivates a risk analysis in order to enrich the link between Kalman Filtering and the optimization community. We may see the static EKF as the online approximation of Bayesian Model Averaging, and similarly to the analysis of BMA derived by~\cite{kakade2005online} our analysis is robust to misspecification, that is we don't assume the data to be generated by the probabilistic model.

The static EKF is very close to the Online Newton Step~\citep{hazan2007logarithmic} as both are second-order online algorithms and our results are of the same flavor as those obtained on the ONS~\citep{mahdavi2015lower}.
However the ONS requires the knowledge of the region in which the optimization is realized. It is involved in the choice of the gradient step size and a projection step is done to ensure that the search stays in the chosen region. On the other hand the EKF has no gradient step size parameter nor projection step and thus does not need additional information on the optimal localization, yielding two advantages at the cost of being less generic.

First, there is no costly projection step and each recursive update runs in $O(d^2)$ operations. Therefore, our comparison of the static EKF with the ONS provides a lead to the open question of~\cite{koren2013open}. Indeed, the problem of the ONS pointed out by~\cite{koren2013open} is to control the cost of the projection step and the question is whether it is possible to perform better than the ONS in the stochastic exp-concave setting. We don't answer the open question in the general setting. However, we suggest a general way to get rid of the projection by dividing the analysis between a convergence proof of the algorithm to the optimum and a second phase where the estimate stays in a small region around the optimum where no projection is required.

Second, the algorithm is (nearly) parameter-free. We believe that bayesian statistics is the reasonable approach in order to obtain parameter-free online algorithms in the unconstrained setting. Parameter-free is not exactly correct as there are initialization parameters, which we see as a smoothed version of the hard constraint imposed by bounded algorithm, but they have no impact on the leading terms of our bounds. Kalman Filter in constant dynamics is exactly ridge regression with a varying regularization parameter (see Section \ref{section:localized_regret}), and similarly the static EKF may be seen as the online approximation of a regularized version of the well-studied Empirical Risk Minimizer (see for instance~\cite{ostrovskii2018finite}).

\subsection{Contributions}
Our central contribution is a local analysis of the EKF under assumptions defined in Section \ref{section:assumptions}, and provided that consecutive steps stay in a small ball around the optimum $\theta^*$. We derive local bounds on the cumulative risk with high probability from a martingale analysis. Our analysis is similar to the one of~\cite{mahdavi2015lower} who obtained comparable results for the ONS, and we slightly refine their constants with an intermediate result (see Theorem \ref{th:refinement_mahdavi}). That is the aim of Section \ref{section:localized}.

We then focus  on linear regression and logistic regression as these two well-known problems are challenging in the unconstrained setting.
In linear regression, the gradient of the loss is not bounded globally.
In logistic regression, the loss is strictly convex, but neither strongly convex nor exp-concave in the unconstrained setting.
In Section \ref{section:logistic}, we develop a global bound in the logistic setting. However, in order to use our local result we first obtain the convergence of the algorithm to $\theta^*$, and for that matter we need a good control of $P_t$. We therefore modify slightly the algorithm in the fashion of~\cite{bercu2020efficient}. This modification is limited in time and thus our local analysis still applies.
In Section \ref{section:quadratic}, we apply our analysis to the quadratic setting. We rely on~\cite{hsu2012random} to obtain the convergence after exhibiting the correspondence between Kalman Filter in constant dynamics and Ridge Regression, and we therefore obtain similarly a global bound using our local analysis.

Finally, we demonstrate numerically the competitiveness of the static EKF for logistic regression in Section~\ref{sec:experiments}.

\section{Definitions and assumptions}\label{section:assumptions}
We consider loss functions that may be written as the negative log-likelihood of a Generalized Linear Model~\citep{mccullagh1989}. Formally, the loss is defined as $\ell(y,\theta^TX)=-\log{p_{\theta}(y\mid X)}$ where $\theta\in\R^d$, $(X,y)\in  \mathcal X\times\mathcal Y $ for some $\mathcal X \subset \R^d$ and  $\mathcal Y\subset \R$ and $p_{\theta}$ is of the form
\begin{equation}\label{eq:glmdistr}
    p_{\theta}(y\mid X) = h(y)\exp\left(\frac{y\,\theta^TX-b(\theta^TX)}{a}\right)\,,
\end{equation}
where $a$ is a constant and $h$ and $b$ are one-dimensional functions on which a few assumptions are required (Assumption \ref{ass:bounded}). This includes linear and logistic regression, see Sections \ref{section:logistic} and \ref{section:quadratic}. We display the static EKF in Algorithm \ref{alg:ekf} in this setting.
\begin{algorithm}[h]
{\caption{Static Extended Kalman Filter for Generalized Linear Model}
\label{alg:ekf}}
{
\begin{enumerate}
\item {\it Initialization}: $P_1$ is any positive definite matrix, $\hat{\theta}_1$ is any initial parameter in $\mathbb{R}^d$.
\item {\it Iteration}: at each time step $t=1,2,\ldots$
\begin{enumerate}
\item Update $P_{t+1}  = P_t  - \frac{P_tX_tX_t^TP_t}{1+X_t^TP_tX_t\alpha_t}\alpha_t$
with $\alpha_t=\frac{b''(\hat{\theta}_t^TX_t)}{a}$.
\item Update $\hat{\theta}_{t+1} = \hat{\theta}_t + P_{t+1}\frac{(y_t-b'(\hat{\theta}_t^TX_t))X_t}{a}$.
\end{enumerate}
\end{enumerate}
}
\end{algorithm}

Due to only matrix-vector and vector-vector multiplication, Algorithm \ref{alg:ekf} has a running-time complexity of $O(d^2)$ at each iteration and thus $O(nd^2)$ for $n$ iterations.

Note that although we need the loss function to be derived from a likelihood of the form \eqref{eq:glmdistr}, we do not need the data to be generated under this process.
We need two standard hypotheses on the data. The first one is the i.i.d. assumption:
\begin{assumption}\label{ass:iid}
The observations $(X_t,y_t)_t$ are i.i.d. copies of the pair $(X,y)\in  \mathcal X\times\mathcal Y $, $\mathbb{E}[XX^T]$ is positive definite and the diameter (for the Euclidian distance) of $\mathcal X$ is bounded by $D_X$. \end{assumption}
Working under Assumption \ref{ass:iid}, we define the risk function $L(\theta) = \mathbb{E}\left[\ell(y,\theta^TX)\right]$ and $\Lambda_{\rm min}$ the  smallest eigenvalue of $\mathbb{E}[XX^T]$. In order to work on a well-defined optimization problem we assume there exists a minimum:
\begin{assumption}\label{ass:existence}
There exists $\theta^*\in\R^d$ such that $L(\theta^*) = \inf\limits_{\theta\in\R^d} L(\theta)$.
\end{assumption}

We treat two different settings requiring different assumptions, summarized in Assumption \ref{ass:bounded} and \ref{ass:subgaussian} respectively. 
First, motivated by logistic regression we define:
\begin{assumption}\label{ass:bounded}
There exists $(\kappa_{\varepsilon})_{\varepsilon>0},(h_{\varepsilon})_{\varepsilon>0}$ and $\rho_{\varepsilon}\xrightarrow[\varepsilon\to 0]{} 1$ such that for any $\varepsilon>0$ and any $\theta,\theta'\in\R^d$ satisfying $max(\|\theta-\theta^*\|,\|\theta'-\theta^*\|)\le \varepsilon$, we have 
\begin{itemize}
\item
$\ell'(y,\theta^TX)^2\le \kappa_{\varepsilon} \ell''(y,\theta^TX)$ a.s.
\item 
$\ell''(y,\theta^TX)\le h_{\varepsilon}$ a.s.
\item
$\ell''(y,\theta^TX) \ge \rho_{\varepsilon} \ell''(y,\theta'^TX)$ a.s.
\end{itemize}
\end{assumption}
Assumption \ref{ass:bounded} requires local exp-concavity (around $\theta^*$) along with some regularity on $\ell''$ ($\ell''$ continuous and $\ell''(y,\theta^{*T}X)\ge\mu>0$ a.s. is sufficient). That setting implies $\mathcal Y$ bounded, because $\ell'$ depends on $y$ whereas $\ell''$ doesn't.
In logistic regression, $\mathcal Y= \{-1,+1\}$ and Assumption \ref{ass:bounded} is satisfied for $\kappa_{\varepsilon}=e^{D_X(\|\theta^*\|+\varepsilon)},h_{\varepsilon} = \frac14,\rho_{\varepsilon}=e^{-\varepsilon D_X}$.

Second, we consider the quadratic loss, corresponding to a gaussian model, and in order to include the well-specified model, we assume $y$ sub-gaussian conditionally to $X$, and not too far away from the model:
\begin{assumption}\label{ass:subgaussian}
The distribution of $(X,y)\in \mathcal X\times\mathcal Y$ satisfies
\begin{itemize}
\item
There exists $\sigma^2>0$ such that for any $s\in\R$, $\mathbb{E}\left[e^{s(y-\mathbb{E}[y\mid X])} \mid X\right]\le e^{\frac{\sigma^2s^2}{2}}$ a.s.,
\item
There exists $D_{\rm app}\ge 0$ such that $|\mathbb{E}[y\mid X]-\theta^{*T}X|\le D_{\rm app}$\, a.s.
\end{itemize}
\end{assumption}
Both conditions of Assumption \ref{ass:subgaussian} hold with $\mathcal Y=\R$ and $ D_{\rm app}=0$ for the well-specified sub-gaussian linear model with random bounded design. The second condition of Assumption  \ref{ass:subgaussian} is satisfied for $D_{\rm app}> 0$ in misspecified sub-gaussian linear model with a.s. bounded approximation error. 

\newpage

\section{The algorithm around the optimum}\label{section:localized}
In this section, we analyse the cumulative risk under a strong convergence assumption:
\begin{assumption}\label{ass:localized}
For any $\delta,\varepsilon>0$, there exists $\tau(\varepsilon,\delta)\in \mathbb{N}$ such that it holds for any $t>\tau(\varepsilon,\delta)$ simultaneously
\begin{equation*}
\|\hat{\theta}_t - \theta^*\| \le \varepsilon\,,
\end{equation*}
with probability at least $1-\delta$.
\end{assumption}
Assumption \ref{ass:localized} states that with high probability there exists a convergence time after which the algorithm stays trapped in a local region around the optimum.
Sections \ref{section:logistic} and \ref{section:quadratic} are devoted to define explicitly such a convergence time for logistic and linear regression.

\subsection{Main results}
We present our result in the bounded and sub-gaussian settings. The results and their proofs are very similar, but two crucial steps are different. First, Assumption \ref{ass:bounded} yields a bound on the gradient holding almost surely. We relax the boundedness condition for the quadratic loss with a sub-gaussian hypothesis, requiring a specific analysis with larger bounds. Second, our analysis is based on a second-order expansion. The quadratic loss satisfies an identity with its second-order Taylor expansion but we need Assumption \ref{ass:localized} along with the third point of Assumption \ref{ass:bounded} otherwise.

The following theorem is our result in the bounded setting. The constant $0.95$ may be chosen arbitrarily close to $0.5$ with growing constants in the bound on the cumulative risk. There is a hidden trade-off in $\varepsilon$: on the one hand, the smaller $\varepsilon$ the better our upper-bound, but on the other hand $\tau(\varepsilon,\delta)$ increases when $\varepsilon$ decreases, and thus our bound applies after a bigger convergence time.
\begin{theorem}
\label{th:localized_bounded}
If Assumptions \ref{ass:iid}, \ref{ass:existence}, \ref{ass:bounded}, \ref{ass:localized} are satisfied and if $\rho_{\varepsilon}>0.95$, for any $\delta>0$, it holds for any $n\ge 1$ simultaneously
\begin{align*}
    \sum\limits_{t=\tau(\varepsilon,\delta)+1}^{\tau(\varepsilon,\delta)+n} L(\hat{\theta}_t) - L(\theta^*) \le\ & \frac{5}{2} d \kappa_{\varepsilon} \ln\left(1 +n \frac{h_{\varepsilon}\lambda_{\rm max}(P_1)D_X^2}{d}\right)+ 5\lambda_{\rm max}\left(P_{\tau(\varepsilon,\delta)+1}^{-1}\right)\varepsilon^2 \\
    & + 30\left(2\kappa_{\varepsilon} + h_{\varepsilon}\varepsilon^2D_X^2\right) \ln\delta^{-1}\,,
\end{align*}
with probability at least $1-3\delta$.
\end{theorem}

For the quadratic loss, we obtain the following theorem under the sub-gaussian hypothesis. We observe a similar trade-off in $\varepsilon$.
\begin{theorem}
\label{th:localized_linear}
In the quadratic setting, if Assumptions \ref{ass:iid}, \ref{ass:existence}, \ref{ass:subgaussian} and \ref{ass:localized} are satisfied, for any $\delta>0$ and any $\varepsilon>0$, it holds for any $n\ge 1$ simultaneously
\begin{align*}
    \sum\limits_{t=\tau(\varepsilon,\delta)+1}^{\tau(\varepsilon,\delta)+n} L(\hat{\theta}_t) - L(\theta^*) \le\ & \frac{15}{2} d \left(8\sigma^2+ D_{\rm app}^2 + \varepsilon^2 D_X^2\right) \ln\left(1 +n \frac{\lambda_{\rm max}(P_{1})D_X^2}{d}\right) + 5 \lambda_{\rm max}\left(P_{\tau(\varepsilon,\delta)+1}^{-1}\right)\varepsilon^2 \\
    & + 115\left(\sigma^2(4+\frac{\lambda_{\rm max}(P_1)D_X^2}{4}) + D_{\rm app}^2 + 2\varepsilon^2D_X^2\right) \ln\delta^{-1}\,,
\end{align*}
with probability at least $1-5\delta$.
\end{theorem}

We display the parallel between the ONS and the static EKF in Algorithm \ref{alg:ons_ekf} through their recursive updates.
\begin{algorithm}[t]
{\caption{Recursive updates: the ONS and the static EKF}
\label{alg:ons_ekf}}
{
\begin{multicols}{2}
	\textbf{Online Newton Step}\\
	\begin{itemize}
	\item
	$P_{t+1}^{-1} = P_t^{-1} + \ell'(y_t,\hat{\theta}_t^TX_t)^2 X_tX_t^T$.
	\item
	$w_{t+1} = \prod\limits_{\mathcal{K}}^{P_{t+1}^{-1}}\left(w_t - \frac{1}{\gamma} P_{t+1} \nabla_t\right)$.
	\end{itemize}
	\textbf{Static Extended Kalman Filter}
	\begin{itemize}
	\item
	$P_{t+1}^{-1} = P_t^{-1} + \ell''(y_t,\hat{\theta}_t^TX_t) X_tX_t^T$.
	\item
	$\hat{\theta}_{t+1} = \hat{\theta}_t - P_{t+1} \nabla_t$.
	\end{itemize}
\end{multicols}
where $\nabla_t = \ell'(y_t,\hat{\theta}_t^TX_t) X_t$ and $\prod\limits_{\mathcal{K}}^{P_{t+1}^{-1}}$ is the projection on $\mathcal{K}$ for the norm $\|.\|_{P_{t+1}^{-1}}$.
}
\end{algorithm}
Our analysis is similar to the one of~\cite{mahdavi2015lower} and an intermediate result yields the following refinement on their bound on the risk of the averaged ONS:
\begin{theorem}\label{th:refinement_mahdavi}
Let $(w_t)_t$ be the ONS estimates starting from $P_1=\lambda I$ and using a step-size $\gamma=\frac12\min(\frac{1}{4GD},\alpha)$ with $\alpha$ the exp-concavity constant. Assume the gradients are bounded by $G$ and the optimization set $\mathcal{K}$ has diameter $D$. Then for any $\delta > 0$, it holds for any $n\ge 1$ simultaneously
\begin{equation*}
	\sum\limits_{t=1}^n L(w_t) - L(\theta^*) \le \frac{3}{2\gamma} d\ln\left(1+\frac{nG^2}{\lambda d}\right) + \frac{\lambda \gamma}{6} D^2 +  \left(\frac{12}{\gamma} + \frac{4\gamma G^2D^2}{3}\right) \ln\delta^{-1}\,,
\end{equation*}
with probability at least $1-2\delta$.
\end{theorem}
For consistency with the previous results we display Theorem \ref{th:refinement_mahdavi} as a bound on the cumulative risk, whereas Theorem 3 of~\cite{mahdavi2015lower} is a bound on the risk of the averaged ONS. The latter follows directly from Theorem \ref{th:refinement_mahdavi} by an application of Jensen's inequality. The proof of Theorem \ref{th:refinement_mahdavi} consists in replacing Theorem 4 of~\cite{mahdavi2015lower} with the following lemma:
\begin{lemma}\label{lemma:corobercu}
Let $k\ge 0$ and $(\Delta N_t)_{t> k}$ be any martingale difference adapted to the filtration $(\mathcal{F}_t)_{t\ge k}$ such that for any $t>k$, $\mathbb{E}[\Delta N_t^2\mid \mathcal{F}_{t-1}]<\infty$.
For any $\delta,\lambda>0$, we have the simultaneous property
\begin{equation*}
	\sum\limits_{t = k+1}^{k+n} \left(\Delta N_t - \frac{\lambda}{2} ((\Delta N_t)^2+\mathbb{E}[(\Delta N_t)^2\mid \mathcal{F}_{t-1}]) \right) \le \frac{\ln\delta^{-1}}{\lambda}, \qquad n\ge 1\,,
\end{equation*}
with probability at least $1-\delta$.
\end{lemma}
This result proved in Section \ref{app:corobercu} is a corollary of a martingale inequality from~\cite{bercu2008exponential} and a stopping time construction~\citep{freedman1975tail}.

The comparison of Theorem \ref{th:refinement_mahdavi} with Theorem \ref{th:localized_bounded} is difficult because we don't control in general $\tau(\varepsilon,\delta)$. We obtain similar constants, as $\kappa_{\varepsilon}$ is the inverse of the exp-concavity constant $\alpha$. However the static EKF is parameter-free  whereas $\alpha$ is an input of the ONS through the setting of the step-size $\gamma$. That is why we argue that the static EKF provides an optimal way to choose the step size, as does averaged SGD~\citep{bach2014adaptivity}. Indeed, as $\varepsilon$ is a parameter of the EKF analysis, we can improve the leading constant $\kappa_{\varepsilon}$ on local region arbitrarily small around $\theta^*$, at a cost for the $\tau(\varepsilon,\delta)$ first terms, whereas in the ONS the choice of a diameter $D>\|\theta^*\|$ makes the gradient step-size sub-optimal and impact the leading constant.
Similarly to the ONS analysis, the use of second-order methods learns adaptively the pre-conditioning matrix which is crucial in order to improve the leading constant $D_X^2/\Lambda_{\rm min}$ obtained for first-order methods to $d$. 

A similar comparison is possible between the result of Theorem \ref{th:localized_linear} and tight risk bounds obtained for the ordinary least-squares estimator and the ridge regression estimator~\citep{hsu2012random}. Up to numerical constants, the tight constant $d(\sigma^2+D^2_{\text{app}})$ is achieved by choosing $\varepsilon$ arbitrarily small, at a cost for the $\tau(\varepsilon,\delta)$ first terms.

We detail the key ideas of the proofs through intermediate results in Sections \ref{section:localized_regret} and \ref{section:localized_risk}.
We defer to Appendix \ref{app:localized} the proof of these intermediate results along with the detailed proof of Theorems \ref{th:localized_bounded} and \ref{th:localized_linear}.

\subsection{Comparison with Online Newton Step and Ridge Regression: a regret analysis}\label{section:localized_regret}
To begin our analysis, we formalize the strong links between the static EKF, the ONS and the Ridge Regression forecaster.
For the quadratic loss, the EKF becomes the Kalman Filter by plugging in Algorithm \ref{alg:ekf}  the identities $a=1,b'(\hat{\theta}_t^TX_t)=\hat{\theta}_t^TX_t,\alpha_t=1$.

The parallel with the Ridge Regression forecaster was evoked by~\cite{diderrich1985kalman}, and it is crucial that the static Kalman Filter is the Ridge regression estimator for a decaying regularization parameter. It highlights that the static EKF may be seen as an approximation of the regularized empirical risk minimization problem.

\begin{proposition}\label{prop:kf_ridge}
In the quadratic setting, for any sequence $(X_t,y_t)$ starting from any $\hat{\theta}_1\in\R^d$ and $P_1\succ 0$, the EKF satisfies the optimisation problem
\begin{equation*}
    \hat{\theta}_t = \arg\min\limits_{\theta\in\R^d} \left(\frac12\sum\limits_{s=1}^{t-1} (y_s-\theta^TX_s)^2 + \frac12 (\theta-\hat{\theta}_1)^TP_1^{-1}(\theta-\hat{\theta}_1)\right),\qquad t\ge 1 \,.
\end{equation*}
\end{proposition}

 Notice that the static Kalman Filter  provides automatically a right choice of the Ridge regularization parameter. This proposition is useful in the convergence phase of the quadratic setting.
 
 In order to get a bound that holds sequentially for any $t\ge1$, we adopt an approach similar as the one in \cite{hazan2007logarithmic} on the ONS   (Algorithm \ref{alg:ons_ekf}). The cornerstone of our local analysis is the derivation of a bound on the second-order Taylor expansion of $\ell$, from the recursive update formulae.
\begin{lemma}\label{lemma:recursive_bound}
For any sequence $(X_t,y_t)_t$, starting from $P_1\succ 0$ and $\hat{\theta}_1\in\R^d$, it holds for any $\theta^*\in\R^d$ and $n\in\mathbb{N}$ that
\begin{multline*}
    \sum\limits_{t=1}^{n} \left(\left(\ell'(y_t,\hat{\theta}_t^TX_t)X_t\right)^T(\hat{\theta}_t-\theta^*) - \frac12 (\hat{\theta}_t-\theta^*)^T \left(\ell''(y_t,\hat{\theta}_t^TX_t)X_tX_t^T\right)(\hat{\theta}_t-\theta^*) \right) \\
    \le \frac{1}{2} \sum\limits_{t=1}^{n} X_t^TP_{t+1}X_t\ell'(y_t,\hat{\theta}_t^TX_t)^2 + \frac{\|\hat{\theta}_{1}-\theta^*\|^2}{\lambda_{\rm min}(P_1)} \,.
\end{multline*}
\end{lemma}

In the quadratic setting  there is equality between the quadratic function and its second-order Taylor expansion and a logarithmic regret bound is derived (\cite{cesa2006prediction}, Theorem 11.7). However the factor before the logarithm is not easily bounded, unless we assume $(y_t-\hat{\theta}_t^TX_t)^2$ bounded.

In general, we cannot compare the excess loss with the second-order Taylor expansion, and we need a step size parameter.
In \cite{hazan2007logarithmic}, the regret analysis of the ONS is based on a very similar bound on
\begin{equation*}
	\left(\ell'(y_t,w_t^TX_t)X_t\right)^T(\hat{\theta}_t-\theta^*) - \frac{\gamma}{2} (w_t-\theta^*)^T \left(\ell'(y_t,w_t^TX_t)^2X_tX_t^T\right) (w_t-\theta^*) \,,
\end{equation*}
where $\gamma$ is a step size depending on the exp-concavity constant, a bound on the gradients and the diameter of the search region $\mathcal{K}$. Then the regret bound follows from the exp-concavity property, bounding the excess loss $\ell(y_t,w_t^TX_t)-\ell(y_t,\theta^{*T}X_t)$ with the previous quantity.


We follow a very different approach, to stay parameter-free and to avoid any additional cost in the leading constant. In the stochastic setting, we observe that we can upper-bound the excess risk with a second-order expansion, up to a multiplicative factor.

\subsection{From adversarial to stochastic: the cumulative risk}\label{section:localized_risk}
In order to compare the excess risk with a second-order expansion, we compare the first-order term with the second-order one.
\begin{proposition}
\label{prop:comparison12}
If Assumptions \ref{ass:iid}, \ref{ass:existence} and \ref{ass:bounded} are satisfied, for any $\theta\in\R^d$, it holds
\begin{equation*}
	\frac{\partial L}{\partial \theta}\Bigr|_{\substack{\theta}}^T(\theta-\theta^*) \ge \rho_{\|\theta-\theta^*\|} (\theta-\theta^*)^T \frac{\partial^2 L}{\partial \theta^2}\Bigr|_{\substack{\theta}}(\theta-\theta^*)\,.
\end{equation*}
\end{proposition}

This result leads immediately to the following proposition, using the first-order convexity property of $L$.
\begin{proposition}
\label{prop:secondorder}
If Assumptions \ref{ass:iid}, \ref{ass:existence} and \ref{ass:bounded} are satisfied, for any $\theta\in\R^d$, $0<c<\rho_{\|\theta-\theta^*\|}$, it holds
\begin{align*}
    L(\theta)-L(\theta^*)
    \le \frac{\rho_{\|\theta-\theta^*\|}}{\rho_{\|\theta-\theta^*\|}-c} \left(\frac{\partial L}{\partial \theta}\Bigr|_{\substack{\theta}}^T(\theta-\theta^*) - c (\theta-\theta^*)^T\frac{\partial^2 L}{\partial \theta^2}\Bigr|_{\substack{\theta}}(\theta-\theta^*)\right)\,.
\end{align*}
\end{proposition}

Lemma \ref{lemma:recursive_bound} motivates the use of $c>\frac12$, thus we need at least $\rho_{\|\theta-\theta^*\|}>\frac12$. In the quadratic setting, it holds as an equality with $\rho=1$ because the second derivative of the quadratic loss is constant. In the bounded setting we need to control the second derivative in a small range, and we can achieve that only locally. The natural condition becomes the third condition of Assumption \ref{ass:bounded}.

Then we are left to obtain a bound on the cumulative risk from Lemma \ref{lemma:recursive_bound}. In order to compare the derivatives of the risk and the losses, we need to control the martingale difference adapted to the natural filtration $(\mathcal{F}_t)$ and defined by
\begin{equation}
	\label{eq:martingale}
	\Delta M_t = \left(\frac{\partial L}{\partial \theta}\Bigr|_{\substack{\hat{\theta}_t}} - \nabla_t\right)^T(\hat{\theta}_t-\theta^*),\qquad \text{where}\ \nabla_t= \ell'(y_t,\hat{\theta}_t^TX_t)X_t \,.
\end{equation}
We thus apply Lemma \ref{lemma:corobercu} to the martingale difference defined in Equation \ref{eq:martingale}.
\begin{lemma}
\label{lemma:martingale}
If Assumptions \ref{ass:iid} and \ref{ass:existence} are satisfied, for any $k\ge 0$ and $\delta,\lambda>0$, it holds
\begin{equation*}
    \sum\limits_{t=k+1}^{k+n} \left(\Delta M_t - \lambda (\hat{\theta}_t-\theta^*)^T\left(\nabla_t\nabla_t^T + \frac32\mathbb{E}\left[\nabla_t\nabla_t^T\mid \mathcal{F}_{t-1}\right]\right)(\hat{\theta}_t-\theta^*) \right) \le \frac{\ln\delta^{-1}}{\lambda},\qquad n\ge 1\,,
\end{equation*}
with probability at least $1-\delta$.
\end{lemma}

Summing Lemma \ref{lemma:recursive_bound} and \ref{lemma:martingale}, the rest of the proof consists in the following two steps:
\begin{itemize}
    \item
    We derive poissonian bounds to control the quadratic terms in $\hat{\theta}_t-\theta^*$ in terms of the one of the second-order bound of Proposition \ref{prop:secondorder}.
    \item
    We upper-bound $\sum_t X_t^TP_{t+1}X_t\ell'(y_t,\hat{\theta}_t^TX_t)^2$ relying on techniques similar to the ridge analysis of the proof of Theorem 11.7 of~\cite{cesa2006prediction}.
\end{itemize}

\section{Logistic setting}\label{section:logistic}
Logistic regression is a widely used statistical model in order to predict a binary random variable $y\in\mathcal Y =\{-1,1\}$. It consists in estimating $\mathcal{L}(y\mid X)$ with
\begin{equation*}
	p_{\theta}(y\mid X) = \frac{1}{1+e^{-y\theta^TX}} = \exp\left(\frac{y\theta^TX-(2\ln(1+e^{\theta^TX})-\theta^TX)}{2}\right)\,.
\end{equation*}
In the GLM notations, it yields $a=2$ and $b(\theta^TX)=2\ln(1+e^{\theta^TX})-\theta^TX$.

\subsection{The truncated algorithm}\label{section:logistic_overview}
For checking Assumption \ref{ass:localized}, we follow a trick consisting in changing slightly the update on $P_t$~\citep{bercu2020efficient}.
Indeed, when the authors tried to prove the asymptotic convergence of the static EKF (which they named Stochastic Newton Step) using Robbins-Siegmund Theorem, they needed the convergence of $\sum_t \lambda_{\rm max}(P_t)^2$. This seems very likely to hold as we have intuitively $P_t\propto 1/t$. However, in order to obtain $\lambda_{\rm max}(P_t) = \mathcal{O}(1/t)$, one needs to lower-bound $\alpha_t$, that is, lower-bound $b''$, and that is impossible in the global logistic setting. Therefore, the idea is to force a lower-bound on $\alpha_t$ in its definition. We thus define, for some $0<\beta<1/2$,
\begin{equation*}
	\alpha_t = \max\left(\frac{1}{t^{\beta}},\frac{1}{(1+e^{\hat{\theta}_t^TX_t})(1+e^{-\hat{\theta}_t^TX_t})}\right),\qquad t\ge 1\,.
\end{equation*}

\begin{algorithm}[t]
{\caption{Truncated Extended Kalman Filter for Logistic Regression}
\label{alg:ekf_logistic}}
{
\begin{enumerate}
\item {\it Initialization}: $P_1$ is any positive definite matrix, $\hat{\theta}_1$ is any initial parameter in $\R^d$.
\item {\it Iteration}: at each time step $t=1,2,\ldots$
\begin{enumerate}
\item Update $P_{t+1}  = P_t  - \frac{P_tX_tX_t^TP_t}{1+X_t^TP_tX_t\alpha_t}\alpha_t$,
with $\alpha_t=\max\left(\frac{1}{t^{\beta}},\frac{1}{(1+e^{\hat{\theta}_t^TX_t})(1+e^{-\hat{\theta}_t^TX_t})}\right)$.
\item Update $\hat{\theta}_{t+1} = \hat{\theta}_t + P_{t+1}\frac{y_tX_t}{1+e^{y_t\hat{\theta}_t^TX_t}}$.
\end{enumerate}
\end{enumerate}
}
\end{algorithm}
This modification yields Algorithm \ref{alg:ekf_logistic}, where we keep the notations $\hat{\theta}_t,P_t$ with some abuse. We impose a decreasing threshold on $\alpha_t$ ($\beta>0$) so that the recursion coincides with Algorithm \ref{alg:ekf} after some steps. Then we apply our analysis of Section \ref{section:localized} after slightly changing Assumption \ref{ass:localized}:
\begin{assumption}\label{ass:localized_truncated}
For any $\delta,\varepsilon>0$, there exists $\tau(\varepsilon,\delta)\in \mathbb{N}$ such that it holds for any $t>\tau(\varepsilon,\delta)$ 
\begin{equation*}
\|\hat{\theta}_t - \theta^*\| \le \varepsilon \text{ and } \alpha_t = \frac{1}{(1+e^{\hat{\theta}_t^TX_t})(1+e^{-\hat{\theta}_t^TX_t})}
\end{equation*}
simultaneously with probability at least $1-\delta$.
\end{assumption}

The sensitivity of the algorithm to $\beta$ is discussed at the end of Section \ref{section:logistic_convergence}. Also, note that the threshold could be $c/t^{\beta}$, $c>0$, as in~\cite{bercu2020efficient}, it would not change the proofs nor the local result below.

We first state the result with $\tau(\varepsilon,\delta)$ in our upper-bound, for the choice $\varepsilon=1/(20D_X)$. We define its value in the next paragraph, and we discuss its dependence to parameters.
\begin{theorem}\label{th:result_logistic}
If Assumptions \ref{ass:iid}, \ref{ass:existence}  and \ref{ass:localized_truncated} are satisfied, for any $\delta>0$ it holds for any $n\ge 1$ simultaneously
\begin{align*}
    \sum\limits_{t=1}^{n} L(\hat{\theta}_t) - L(\theta^*) \le\ &
    3d e^{D_X\|\theta^*\|}  \ln\left(1 + n \frac{\lambda_{\rm max}(P_1) D_X^2}{4d}\right) + \frac{\lambda_{\rm max}(P_1^{-1})}{75D_X^2}  + 64 e^{D_X\|\theta^*\| } \ln\delta^{-1} \\
     & + \tau\left(\frac{1}{20 D_X},\delta\right) \left(\frac{1}{300} + D_X\|\hat{\theta}_1-\theta^*\| \right) + \tau\left(\frac{1}{20 D_X},\delta\right)^2\frac{\lambda_{\rm max}(P_1)D_X^2}{2}  \,,
\end{align*}
with probability at least $1-4\delta$.
\end{theorem}

\subsection{Definition of $\tau(\varepsilon,\delta)$ in Assumption \ref{ass:localized_truncated}}\label{section:logistic_convergence}
It is proved that $\|\hat{\theta}_n - \theta^*\|^2 = O\left(\ln n/n\right)$ almost surely (\cite{bercu2020efficient}, Theorem 4.2). We don't obtain a non-asymptotic version of this rate of convergence, but the aim of this paragraph is to check Assumption \ref{ass:localized_truncated} with an explicit value of $\tau(\varepsilon,\delta)$ for any $\delta,\varepsilon>0$.

The objective of the truncation introduced in the algorithm is to improve the control on $P_t$. We state that fact formally with a concentration result based on~\cite{Tropp2012}.
\begin{proposition}
\label{prop:boundP}
If Assumption \ref{ass:iid} is satisfied, for any $\delta>0$, it holds simultaneously
\begin{equation*}
    \forall t >  \left(\frac{20D_X^4}{\Lambda_{\rm min}^2}\ln\left(\frac{625dD_X^8}{\Lambda_{\rm min}^4\delta}\right)\right)^{1/(1-\beta)},\qquad \lambda_{\rm max}(P_t) \le \frac{4}{\Lambda_{\rm min}t^{1-\beta}}\,,
\end{equation*}
with probability at least $1-\delta$.
\end{proposition}
The limit $\beta<1/2$ thus corresponds to the condition $\sum_t \lambda_{\rm max}(P_t)^2<+\infty$ with high probability.
Motivated by Proposition \ref{prop:boundP}, we define, for $C>0$, the event
\begin{align*}
	A_{C}:=\bigcap\limits_{t=1}^{\infty} \Big(\lambda_{\rm max}(P_t)\le \frac{C}{t^{1-\beta}}\Big)\,.
\end{align*}
To obtain a control on $P_t$ holding for any $t$, we use the relation $\lambda_{\rm max}(P_t)\le \lambda_{\rm max}(P_1)$ holding almost surely. We thus define
\begin{align*}
    C_{\delta} = \max\left(\frac{4}{\Lambda_{\rm min}}, \lambda_{\rm max}(P_1)\left(\frac{20D_X^4}{\Lambda_{\rm min}^2}\ln\left(\frac{625dD_X^8}{\Lambda_{\rm min}^4\delta}\right)\right)\right) \,,
\end{align*}
and we obtain $\mathbb{P}\left(A_{C_{\delta}}\right) \ge 1-\delta$.
We obtain the following theorem under that condition.
\begin{theorem}
\label{th:convergence}
Provided that Assumptions \ref{ass:iid} and \ref{ass:existence} are satisfied, if $\hat{\theta}_1=0$ we have for any $\varepsilon>0$ and $t\ge \exp\left(\frac{2^8 D_X^8 C_{\delta}^2 (1+e^{D_X(\|\theta^*\|+\varepsilon)})^{3}}{\Lambda_{\rm min}^{3}(1-2\beta)^{3/2}\varepsilon^2}\right)$,
\begin{align*}
    \mathbb{P}(\|\hat{\theta}_t-\theta^*\| > \varepsilon \mid A_{C_{\delta}}) \le\ & (\sqrt{t}+1) \exp\left(- \frac{\Lambda_{\rm min}^6(1-2\beta)\varepsilon^4}{2^{16} D_X^{12}C_{\delta}^2(1+e^{D_X(\|\theta^*\|+\varepsilon)})^6}  \ln(t)^2\right) \\
    & + t \exp\left(-\frac{\Lambda_{\rm min}^2(1-2\beta)\varepsilon^4}{2^{11} D_X^4C_{\delta}^2 (1+e^{D_X(\|\theta^*\|+\varepsilon)})^2}  (\sqrt{t}-1)^{1-2\beta}\right)\,.
\end{align*}
\end{theorem}
The beginning of our convergence proof starts similarly as the analysis of~\cite{bercu2020efficient}: we obtain a recursive inequality ensuring that $(L(\hat{\theta}_t))_t$ is decreasing in expectation. However, in order to obtain a non-asymptotic result we cannot apply Robbins-Siegmund Theorem.
Instead we use the fact that the variations of the algorithm $\hat{\theta}_t$ are slow provided by the control on $P_t$. Thus, if the algorithm was far from the optimum, the last estimates were far too which contradicts the decrease in expectation of the risk. Consequently, we look at the last $k\le t$ such that $\|\hat{\theta}_k-\theta^*\| < \varepsilon/2$, if it exists. We decompose the probability of being outside the local region in two scenarii, yielding the two terms in Theorem \ref{th:convergence}. If $k<\sqrt{t}$, the recursive decrease in expectation makes it unlikely that the estimate stays far from the optimum for a long period. If $k>\sqrt{t}$, the control on $P_t$ allows a control on the probability that the algorithm moves fast, in $t-k$ steps, away from the optimum.

\newpage
The following corollary explicitly defines a guarantee for the convergence time.
\begin{corollary}
\label{coro:convergence}
Provided that Assumptions \ref{ass:iid} and \ref{ass:existence} are satisfied, if $\hat{\theta}_1=0$ we check Assumption \ref{ass:localized_truncated} for any $\varepsilon>0$, $\delta>0$ and
\begin{align*}
    \tau(\varepsilon,\delta) = \max\Bigg(\left(2(1+e^{D_X(\|\theta^*\|+\varepsilon)})\right)^{1/\beta}, \exp\left(\frac{3\cdot 2^{15} D_X^{12}C_{\delta/2}^2(1+e^{D_X(\|\theta^*\|+\varepsilon)})^6}{\Lambda_{\rm min}^6(1-2\beta)^{3/2}\varepsilon^4}\right), 6\delta^{-1} \Bigg)\,.
\end{align*}
\end{corollary}

This definition of $\tau(\varepsilon,\delta)$ allows a discussion on the dependence of the bound Theorem \ref{th:result_logistic} to the different parameters:
\begin{itemize}
	\item The truncation has introduced an extraparameter $\beta$, on which $\tau(\varepsilon,\delta)$ strongly depends with a trade-off. On the one hand, when $\beta$ is close to $0$, the algorithm is slow to coincide with the true Extended Kalman Filter, for which our fast rate holds. Precisely, we have $\tau(\varepsilon,\delta) = e^{O(1)/\beta}$. On the other hand, the truncation was introduced to control $P_t$. The larger $\beta$, the larger our control on $\lambda_{\rm max}(P_t)$ and thus we get $\tau(\varepsilon,\delta) = e^{O(1)/(1-2\beta)^{3/2}}$. 
	
	\item As Corollary \ref{coro:convergence} holds for any $\varepsilon>0$, 
	the compromise realized with $\varepsilon=1/(20D_X)$, made for simplifying constants, is totally arbitrary. The dependence of the convergence time is of the order $\tau(\varepsilon,\delta) = e^{O(1)/\varepsilon^4}$. However the $\log n$ term of the bound has a $e^{D_X\varepsilon}$ factor. Thus the best compromise should be an $\varepsilon>0$ decreasing with $n$.
	
	\item
	The dependence to $\delta$ is complex. The third constraint on $\tau(\varepsilon,\delta)$ is $O(\delta^{-1})$ which should not be sharp. 
\end{itemize}
To improve this lousy dependence of the bound, one needs a better control of $P_t$. It would follow from a specific analysis of the  $O(\ln\delta^{-1})$ first recursions in order to "initialize" the control on $P_t$. However the objective of Corollary \ref{coro:convergence} was to check Assumption \ref{ass:localized_truncated} and not to get an optimal value of $\tau(\varepsilon,\delta)$. Moreover practical considerations show that the truncation is artificial and can even deteriorate the performence of the EKF, see Section \ref{sec:experiments}. Thus~\cite{bercu2020efficient} suggest a threshold as low as possible ($10^{-10}/t^{0.49}$) so that the truncation makes no difference in numerical experiments. A tight probability bound on $\lambda_{\rm max}(P_t)$ of the EKF is a very important and challenging open question.

\section{Quadratic setting}
\label{section:quadratic}
We state our result for the quadratic loss where Algorithm \ref{alg:ekf} becomes the standard Kalman Filter. We first state our result with an upper-bound depending on $\tau(\varepsilon,\delta)$, then we define $\tau(\varepsilon,\delta)$ satisfying Assumption \ref{ass:localized}.

As for the logistic setting, we split the cumulative risk into two sums. The sum of the first terms is roughly bounded by a worst case analysis, and the sum of the last terms is estimated thanks to our local analysis (Theorem \ref{th:localized_linear}). However, as the loss and its gradient are not bounded we cannot obtain a similar almost sure upper-bound on the convergence phase. The sub-gaussian assumption provides a high probability bound instead.

\begin{theorem}\label{th:result_linear}
Provided that Assumptions \ref{ass:iid}, \ref{ass:existence}, \ref{ass:subgaussian} and \ref{ass:localized} are satisfied, for any $\varepsilon,\delta>0$, it holds simultaneously
\begin{align*}
    \sum\limits_{t=1}^{n} L(\hat{\theta}_t) - L(\theta^*) \le\ & \frac{15}{2} d \left(8\sigma^2+ D_{\rm app}^2 + \varepsilon^2 D_X^2\right) \ln\left(1 +n \frac{\lambda_{\rm max}(P_{1})D_X^2}{d}\right) + 5\lambda_{\rm max}(P_1^{-1})\varepsilon^2 \\
    & + 115\left(\sigma^2(4+\frac{\lambda_{\rm max}(P_1)D_X^2}{4}) + D_{\rm app}^2 + 2\varepsilon^2D_X^2\right) \ln\delta^{-1} \\
     & + D_X^2\left(5\varepsilon^2 + 2(\|\hat{\theta}_1 - \theta^*\|^2 + 3 \lambda_{\rm max}(P_1) D_X \sigma \ln\delta^{-1})^2\right) \tau(\varepsilon,\delta)\\
     & + \frac{2\lambda_{\rm max}(P_1)^2 D_X^4 (3\sigma+D_{\rm app})^2}{3} \tau(\varepsilon,\delta)^3,\qquad n\ge 1 \,,
\end{align*}
with probability at least $1-6\delta$.
\end{theorem}

As Kalman Filter estimator is exactly the Ridge estimator for a varying regularization parameter, we can use the regularized empirical risk minimization properties to control $\tau(\varepsilon,\delta)$. In particular, we apply the ridge analysis provided by~\cite{hsu2012random}, and we check Assumption \ref{ass:localized} by providing a non-asymptotic definition of $\tau(\varepsilon,\delta)$ in Appendix \ref{app:quadratic}, Corollary \ref{coro:tau_delta_linear}.
Up to universal constants, we get
\begin{align*}
    & \tau(\varepsilon,\delta) \lesssim h\Bigg( \frac{\varepsilon^{-1}}{\Lambda_{\rm min}}\Bigg(\frac{\|\hat{\theta}_1-\theta^*\|^2}{p_1} + \frac{D_X^2}{\Lambda_{\rm min}}(1+D_{\rm app}^2) \sqrt{\ln\delta^{-1}} + \sigma^2d \\
    & \qquad\qquad\qquad\qquad\qquad + \bigg(\frac{D_X}{\sqrt{\Lambda_{\rm min}}}(D_{\rm app}+D_X\|\theta^*\|)+\frac{\|\hat{\theta}_1-\theta^*\|}{\sqrt{p_1}} + \sigma^2\bigg)\ln\delta^{-1}\Bigg) \Bigg) \,,
\end{align*}
with $h(x)=x\ln x$.
We obtain a much less dramatic dependence in $\varepsilon$ than in the logistic setting. However we could not avoid an extra $\Lambda_{\rm min}^{-1}$ factor in the definition of $\tau(\varepsilon,\delta)$. It is not surprising since the convergence phase relies deeply on the behavior of $P_t$. 

\section{Experiments}\label{sec:experiments}
We experiment the static EKF for logistic regression. We first consider well-specified data generated by the same process as~\cite{bercu2020efficient}. Then we slightly change the simulation in order to obtain a misspecified setting.

The explanatory variables $X=(1,Z^T)^T$ are of dimension $d=11$  where $Z$ is a random vector composed of $10$ independent components uniformly generated in $[0,1]$. This yields $D_X=\sqrt{d}$. Then we define $\theta^*=(-9,0,3,-9,4,-9,15,0,-7,1,0)^T$, and at each iteration $t$, the variable $y_t\in\{-1,1\}$ is a Bernoulli variable of parameter $(1+e^{-\theta^{*T}X_t})^{-1}$.

We compare the following sequential algorithms that we all initialize at $\hat{\theta}_1=0$:
\begin{itemize}
	\item
	The EKF and the truncated version (Algorithm \ref{alg:ekf_logistic}). We take the default value $P_1=I_d$ along with the value $\beta=0.49$ suggested by~\cite{bercu2020efficient}. Note that a threshold $10^{-10}/t^{0.49}$ as recommended by~\cite{bercu2020efficient} would always coincide with the EKF. 
	\item
	The ONS and the averaged version. The convex region of search is a ball centered in $0$ and of radius $D=1.1\|\theta^*\|$, a setting where we have good knowledge of $\theta^*$. We implement two choices of the exp-concavity constant on which the ONS crucially relies. First, we use the optimal bound $e^{-D \sqrt{d}}$. Second, we use the minimum of the exp-concavity constants of $1000$ points of the sphere. This yields an optimistic constant and a bigger step size, though we do not prove that the exp-concavity is satisfied.
	\item
	Two Average Stochastic Gradient Descent as described by~\cite{bach2014adaptivity}. First we test the choice of the gradient step size $\gamma=1/(2d\sqrt{N})$ denoted by ASGD and a second version with $\gamma=\|\theta^*\|/(\sqrt{dN})$  denoted by ASGD oracle. Note that these algorithms are with fixed horizon, thus at each step $t$, we have to re-run the whole procedure.
\end{itemize}

We evaluate the different algorithms with the mean squared error $\mathbb{E}[\|\hat{\theta}_t-\theta^*\|^2]$ that we approximate by its empirical version on $100$ samples. We display the results in Figure~\ref{fig:comparison1} for $\theta^*=(-9,0,3,-9,4,-9,15,0,-7,1,0)^T$. 
\begin{figure}
	\centering
	\includegraphics[scale=0.38]{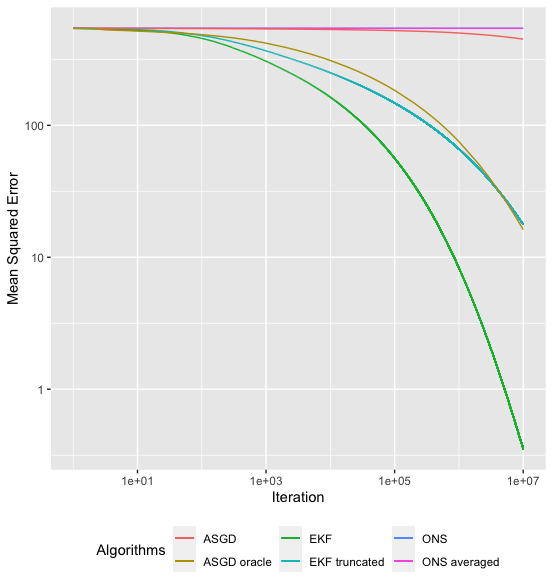}
	\caption{Mean squared error in log-log scale for $\theta^*=(-9,0,3,-9,4,-9,15,0,-7,1,0)^T$. The ONS is applied with our optimistic exp-concavity constant $1.5\cdot 10^{-14}$ instead of $e^{-D\sqrt{d}}\approx1.2\cdot 10^{-37}$. We observe that the algorithm still almost doesn't move.}
	\label{fig:comparison1}
\end{figure}
As this choice of $\theta^*$ yields a distribution of the Bernoulli parameter that is almost degenerated on the values $0$ with small mass at $1$ (cf Figure~\ref{fig:density}), we run the same experiments for $\theta^* = \frac{1}{10}(-9,0,3,-9,4,-9,15,0,-7,1,0)^T$. We display the results in Figure~\ref{fig:comparison10} for the second value of $\theta^*$.  
\begin{figure}
	\centering
	\includegraphics[scale=0.3]{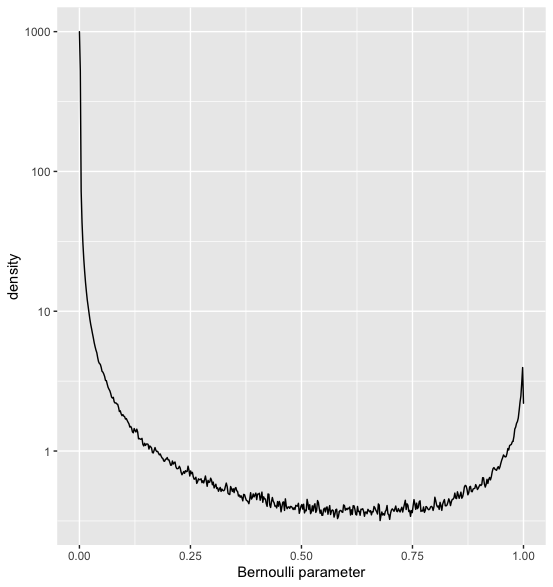}
	\includegraphics[scale=0.3]{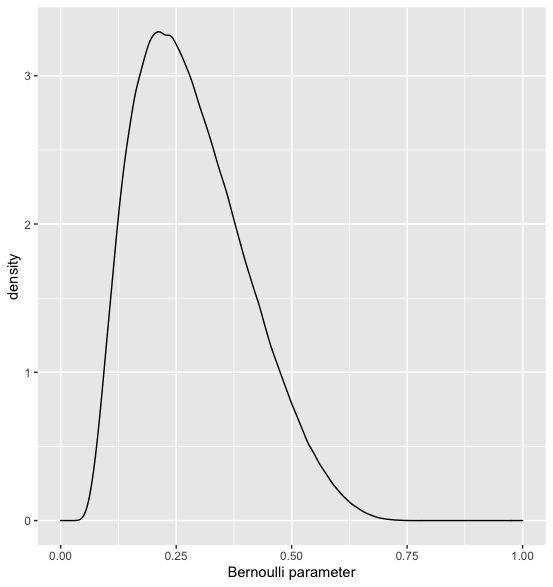}
	\caption{Density of $(1+e^{-\theta^{*T}X})^{-1}$ for $\theta^* = (-9,0,3,-9,4,-9,15,0,-7,1,0)^T$ (left, the ordinate is in log scale) and $\theta^* = \frac{1}{10}(-9,0,3,-9,4,-9,15,0,-7,1,0)^T$ (right) with $10^6$ samples.}
	\label{fig:density}
\end{figure}
\begin{figure}
	\centering
	\includegraphics[scale=0.38]{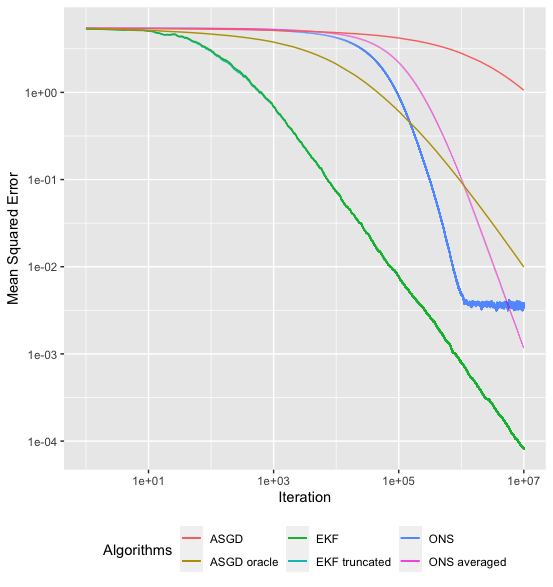}
	\includegraphics[scale=0.38]{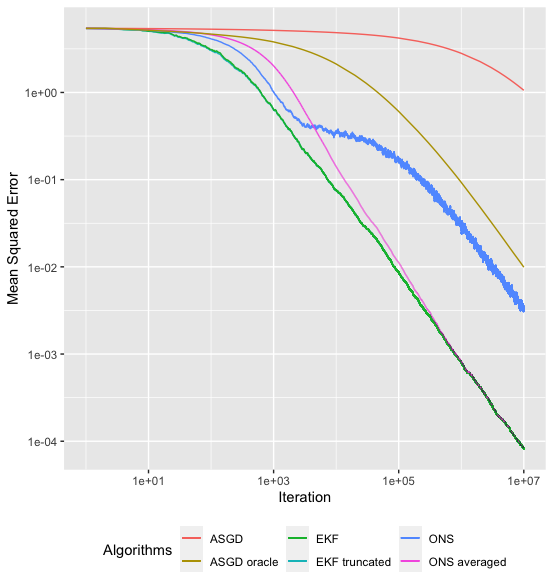}
	\caption{Mean squared error  in log-log scale for $\theta^*=\frac{1}{10}(-9,0,3,-9,4,-9,15,0,-7,1,0)^T$. The ONS is applied with the exp-concavity constant $e^{-D\sqrt{d}}\approx2.0\cdot 10^{-4}$ (left) and with our optimist  exp-concavity constant $2.5\cdot 10^{-2}$ (right).}
	\label{fig:comparison10}
\end{figure}

Finally, in order to demonstrate the robustness of the EKF we test the algorithms in a misspecified setting switching randomly between two well-specified logistic processes. We define $\theta_1=\frac{1}{10}(-9,0,3,-9,4,-9,15,0,-7,1,0)^T$ and  $\theta_2$ where we have only changed the first coefficient from $-9/10$ to $15/10$. Then $y$ is a Bernoulli random variable whose parameter is either $(1+e^{-\theta_1^TX_t})^{-1}$ or $(1+e^{-\theta_2^TX_t})^{-1}$ uniformly at random. We present the results Figure~\ref{fig:comparison_misspecified}.
\begin{figure}
	\centering
	\includegraphics[scale=0.38]{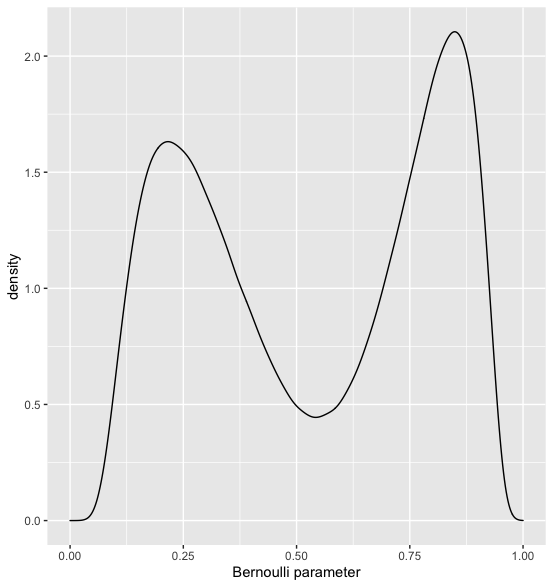}
	\includegraphics[scale=0.38]{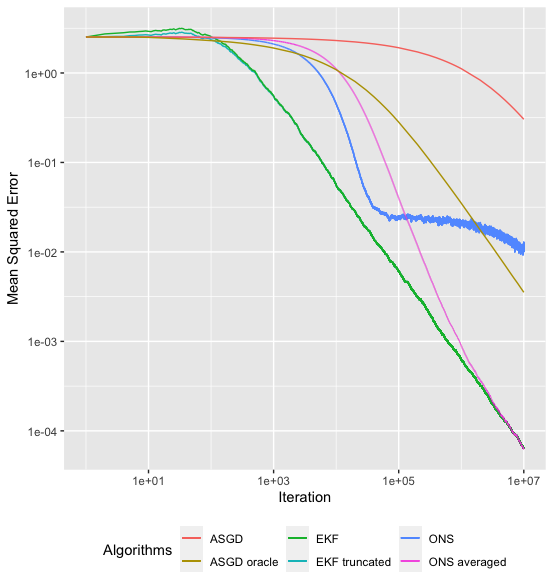}
	\caption{Misspecified setting. Density (left) of the Bernoulli parameter with two modes at $\mathbb{E}[(1+e^{-\theta_1^TX_t})^{-1}]\approx 0.28$ and $\mathbb{E}[(1+e^{-\theta_2^TX_t})^{-1}]\approx 0.79$.
Mean squared error (right) where the ONS is applied with the exp-concavity constant $e^{-D\sqrt{d}} \approx 3.0\cdot 10^{-3}$ and $\theta^*$ is estimated with $10^9$ iterations of the static EKF.}
	\label{fig:comparison_misspecified}
\end{figure}

Our experiments show the superiority of the EKF for logistic regression compared to the ONS or to averaged SGD in all the settings we tested.

It appears clear that low exp-concavity constants is responsible of the poor performances of the ONS. One may tune the gradient step size at the cost of losing the exp-concavity property and thus the regret guarantee of~\citep{hazan2007logarithmic} or its analogous for the cumulative risk~\citep{mahdavi2015lower}. Averaging is crucial in order to obtain a low mean squared error for the ONS, whereas it is useless for the static EKF. Indeed we chose not to plot the averaged version of the EKF for clarity, but the EKF performs better than its averaged version.

It is important to note that in the first setting the truncation deteriorates the performance of the EKF. \cite{bercu2020efficient} argue that the truncation is artificially introduced for the convergence property, they use the threshold $10^{-10}/t^{0.49}$ instead of $1/t^{0.49}$ and thus the truncated version almost coincides with the true EKF. We confirm here that the truncation may be damaging if the threshold is set too high and we recommend to use the EKF in practice, or equivalently the truncated version with the threshold suggested by~\cite{bercu2020efficient}. The results are similar for both versions with a smaller $\theta^*$, because our estimates $\hat{\theta}_t$ are smaller too so that the updates of the two versions coincide faster.

\section{Conclusion}
We have studied an efficient way to tackle some optimization problems, in which we get rid of the projection step of bounded algorithm such as the ONS. We presented a bayesian approach where we transformed the loss into a negative log-likelihood and we used the EKF to approximate the maximum-likelihood estimator. We demonstrate its robustness to misspecification on locally exp-concave losses which can be expressed as GLM log-likelihoods, and we illustrated our theoretical results with numerical experiments for logistic regression. It would be interesting to generalize our results to a larger class of optimization problems.

Finally, this article aimed at strengthening the bridge between Kalman Filtering and the optimization community therefore we made the i.i.d. assumption standard in the stochastic optimization literature. It may lead the way to a risk analysis of the EKF in a non i.i.d. setting, where it might be necessary to assume that the data follows a well-specified state-space model.

\bibliography{mybib.bib}
\appendix

\section*{Organization of the Appendix}
The Appendix follows the structure of the article:
\begin{itemize}
\item
Appendix \ref{app:localized} contains the proofs of Section \ref{section:localized}. Precisely, Lemma \ref{lemma:corobercu} is proved in Section \ref{app:corobercu}, the intermediate results of Sections \ref{section:localized_regret} and \ref{section:localized_risk} are proved in Sections \ref{app:localized_regret} and \ref{app:localized_risk}, then Theorem \ref{th:localized_bounded} is proved in Section \ref{app:localized_bounded} and Theorem \ref{th:localized_linear} in Section \ref{app:localized_quadratic}.
\item 
Appendix \ref{app:logistic} contains the proofs of Section \ref{section:logistic}. We derive the global bound (Theorem \ref{th:result_logistic}) in Section \ref{app:logistic_overview}, then we obtain the concentration result on $P_t$ in Section \ref{app:logistic_concentration}, and finally we prove the convergence of the truncated algorithm in Section \ref{app:logistic_convergence}.
\item 
Appendix \ref{app:quadratic} contains the proofs of Section \ref{section:quadratic}. We prove Theorem \ref{th:result_linear} in Section \ref{app:linear_overview} and then in Section \ref{app:linear_convergence} we prove the convergence of the algorithm, and we define an explicit value of $\tau(\varepsilon,\delta)$ satisfying Assumption \ref{ass:localized}.
\end{itemize}

\section{Proofs of Section \ref{section:localized}}\label{app:localized}
\subsection{Proof of Lemma \ref{lemma:corobercu}}\label{app:corobercu}
We prove the following Lemma inspired by the stopping time technique of~\cite{freedman1975tail} from which we derive Lemma \ref{lemma:corobercu}. We give a general form useful in several proofs.
\begin{lemma}\label{lemma:simultaneous_supermartingale}
Let $(\mathcal{F}_n)$ be a filtration, and we consider a sequence of events $(A_n)$ that is adapted to $(\mathcal{F}_n)$. Let $(V_n)$ be a sequence of random variables adapted to $(\mathcal{F}_n)$ satisfying  $V_0=1$, $V_n\ge 0$ almost surely for any $n$, and
\begin{align*}
	\mathbb{E}[V_n\mid \mathcal{F}_{n-1}, A_{n-1}] \le V_{n-1},\qquad n\ge 1.
\end{align*}
Then for any $\delta>0$, it holds
\begin{equation*}
	\mathbb{P}\left(\left(\bigcup\limits_{n=1}^{\infty} V_n > \delta^{-1}\right) \cup \left(\bigcup\limits_{n=0}^{\infty} \overline{A_n}\right) \right) \le \delta + \mathbb{P}\left(\bigcup\limits_{n=0}^{\infty} \overline{A_n}\right) \,.
\end{equation*}
\end{lemma}
An important particular case is when $(V_n)$ is a super-martingale adapted to the filtration $(\mathcal{F}_n)$ satisfying $V_0=1$ and $V_n\ge 0$ almost surely: then we have simultaneously $V_n\le \delta^{-1}$ for $n\ge 1$ with probability larger than $1-\delta$.

\begin{proof}
We define
\begin{equation*}
	E_k = \bigcup\limits_{n=1}^{k} \left(V_n > \delta^{-1} \cup \overline{A_{n-1}} \right)\,.
\end{equation*}
As $(E_k)$ is increasing, we have, for any $k\ge 1$,
\begin{align*}
	\mathbb{P}(E_k) & = \sum\limits_{n=1}^k \mathbb{P}\left(E_n \cap \overline{E_{n-1}}\right) \\
	 & = \sum\limits_{n=1}^k \mathbb{P}\left(\overline{A_{n-1}} \cap \overline{E_{n-1}}\right) + \sum\limits_{n=1}^k \mathbb{P}\left(V_n > \delta^{-1} \cap \overline{E_{n-1}} \cap A_{n-1}\right)\,.
\end{align*}
First, we have
\begin{align*}
	\sum\limits_{n=1}^k \mathbb{P}\left(\overline{A_{n-1}} \cap \overline{E_{n-1}}\right) \le \mathbb{P}\left(\bigcup\limits_{n=0}^{k-1} \overline{A_n}\right)\,.
\end{align*}
Second, we apply the Chernoff bound:
\begin{align*}
	\sum\limits_{n=1}^k \mathbb{P}\left(V_n > \delta^{-1} \cap \overline{E_{n-1}} \cap A_{n-1}\right) & = \sum\limits_{n=1}^k \mathbb{E}\left[\frac{V_n}{\delta^{-1}} \mathds{1}_{E_n \cap \overline{E_{n-1}}\cap A_{n-1}} \right] \\
	& \le \delta \sum\limits_{n=1}^k \mathbb{E}\left[V_n (\mathds{1}_{\overline{E_{n-1}}\cap A_{n-1}} - \mathds{1}_{\overline{E_n}}) \right] \\
	& = \delta  \sum\limits_{n=1}^k \left(\mathbb{E}\left[V_n \mathds{1}_{\overline{E_{n-1}}\cap A_{n-1}}\right] - \mathbb{E}\left[V_n \mathds{1}_{\overline{E_n}} \right] \right) \,.
\end{align*}
The second line is obtained since $\overline{E_n}\subset \left(\overline{E_{n-1}}\cap A_{n-1}\right)$.
According to the tower property and the super-martingale assumption,
\begin{equation*}
	\mathbb{E}\left[V_n \mathds{1}_{\overline{E_{n-1}}\cap A_{n-1}}\right] = \mathbb{E}\left[\mathbb{E}[V_n\mid \mathcal{F}_{n-1},A_{n-1}] \mathds{1}_{\overline{E_{n-1}}}\right] \le \mathbb{E}\left[V_{n-1} \mathds{1}_{\overline{E_{n-1}}}\right] \,.
\end{equation*}
Therefore, a telescopic argument along with $V_0=1$ and $V_k \mathds{1}_{\overline{E_k}}\ge 0$ yields
\begin{equation*}
	\sum\limits_{n=1}^k \mathbb{P}\left(V_n > \delta^{-1} \cap \overline{E_{n-1}} \cap A_{n-1}\right) \le \delta\,.
\end{equation*}
Finally, for any $k\ge 1$, we obtain
\begin{align*}
	\mathbb{P}\left(E_k\right) \le \mathbb{P}\left(\bigcup\limits_{n=0}^{k-1} \overline{A_n}\right) + \delta\,
\end{align*}
and the desired result follows by letting $k\to \infty$.
\end{proof}

\begin{proof}\textbf{of Lemma \ref{lemma:corobercu}.}
Let $\lambda>0$. For any $n\ge1$, we define
\begin{equation*}
    V_n = \exp\left(\sum\limits_{t=k+1}^{k+n} \left(\lambda \Delta N_t - \frac{\lambda^2}2((\Delta N_t)^2+\mathbb{E}[(\Delta N_t)^2\mid \mathcal{F}_{t-1}])\right)\right)\,.
\end{equation*}
Lemma B.1 of~\cite{bercu2008exponential} states that $(V_n)$ is a super-martingale adapted to the filtration $(\mathcal{F}_{k+n})$. Moreover $V_0=1$ and for any $n$, it holds $V_n\ge 0$ almost surely. Therefore we can apply Lemma \ref{lemma:simultaneous_supermartingale}.
\end{proof}

\subsection{Proofs of Sections \ref{section:localized_regret}}\label{app:localized_regret}
\begin{proof}\textbf{of Proposition \ref{prop:kf_ridge}.} 
The first order condition of the optimum yields 
\begin{equation*}
    \arg\min\limits_{\theta\in\R^d} \sum\limits_{s=1}^{t-1} (y_s-\theta^TX_s)^2 + \frac12 (\theta-\hat{\theta}_1)^TP_1^{-1}(\theta-\hat{\theta}_1) = \hat{\theta}_1 + P_t\sum\limits_{s=1}^{t-1} (y_s- \hat{\theta}_1^TX_s)X_s \,.
\end{equation*}
Therefore we prove recursively that $\hat{\theta}_t - \hat{\theta}_1 = P_t\sum_{s=1}^{t-1} (y_s- \hat{\theta}_1^TX_s)X_s$. It is clearly true at $t=1$. Assuming it is true for some $t\ge 1$, we use the update formula 
\begin{align*}
    \hat{\theta}_{t+1} - \hat{\theta}_1 & = (I-P_{t+1}X_tX_t^T)(\hat{\theta}_t-\hat{\theta}_1) +P_{t+1}y_tX_t - P_{t+1}X_tX_t^T\hat{\theta}_1 \\
    & = (I-P_{t+1}X_tX_t^T)P_t\sum\limits_{s=1}^{t-1} (y_s- \hat{\theta}_1^TX_s)X_s + P_{t+1}(y_t-\hat{\theta}_1^TX_t)X_t \,.
\end{align*}

We conclude with the following identity:
\begin{equation*}
    (I-P_{t+1}X_tX_t^T)P_t=P_t - P_tX_tX_t^TP_t + \frac{P_tX_tX_t^TP_tX_tX_t^TP_t}{X_t^TP_tX_t+1} = P_t - \frac{P_tX_tX_t^TP_t}{X_t^TP_tX_t+1} = P_{t+1} \,.
\end{equation*}
\end{proof}

\begin{proof}\textbf{of Lemma \ref{lemma:recursive_bound}.} 
We start from the update formula $\hat{\theta}_{t+1} = \hat{\theta}_t + P_{t+1}\frac{(y_t-b'(\hat{\theta}_t^TX_t))X_t}{a}$ yielding
\begin{multline*}
    (\hat{\theta}_{t+1}-\theta^*)^TP_{t+1}^{-1}(\hat{\theta}_{t+1}-\theta^*) = (\hat{\theta}_t-\theta^*)^TP_{t+1}^{-1}(\hat{\theta}_t-\theta^*) + 2\frac{(y_t-b'(\hat{\theta}_t^TX_t))X_t^T}{a}(\hat{\theta}_t-\theta^*) \\
    + X_t^TP_{t+1}X_t\left(\frac{y_t-b'(\hat{\theta}_t^TX_t)}{a}\right)^2\,.
\end{multline*}
With a summation argument, re-arranging terms, we obtain:
\begin{align*}
    \sum\limits_{t=1}^{n} &  \left(\frac{(b'(\hat{\theta}_t^TX_t)-y_t)X_t^T}{a}(\hat{\theta}_t-\theta^*) - \frac12 (\hat{\theta}_t-\theta^*)^T(P_{t+1}^{-1}-P_t^{-1})(\hat{\theta}_t-\theta^*)\right) \\
    = &\frac{1}{2} \sum\limits_{t=1}^{n} X_t^TP_{t+1}X_t\left(\frac{y_t-b'(\hat{\theta}_t^TX_t)}{a}\right)^2 \\
    &+ \frac{1}{2} \sum\limits_{t=1}^{n}
    \left((\hat{\theta}_t-\theta^*)^TP_t^{-1}(\hat{\theta}_t-\theta^*) - (\hat{\theta}_{t+1}-\theta^*)^TP_{t+1}^{-1}(\hat{\theta}_{t+1}-\theta^*)\right)\,.
\end{align*}

We bound the telescopic sum: as $P_{n+1}^{-1}\succcurlyeq 0$, we have
\begin{align*}
    \sum\limits_{t=\tau+1}^{\tau+n}&
    \left((\hat{\theta}_t-\theta^*)^TP_t^{-1}(\hat{\theta}_t-\theta^*) - (\hat{\theta}_{t+1}-\theta^*)^TP_{t+1}^{-1}(\hat{\theta}_{t+1}-\theta^*)\right)\\
    & \le (\hat{\theta}_{1}-\theta^*)^TP_{1}^{-1}(\hat{\theta}_{1}-\theta^*) \le \frac{\|\hat{\theta}_{1}-\theta^*\|^2}{\lambda_{\rm min}(P_1)} \,.
\end{align*}

The result follows from the identities
\begin{align*}
	 \frac{(b'(\hat{\theta}_t^TX_t)-y_t)X_t}{a} = \ell'(y_t,\hat{\theta}_t^TX_t)X_t \,,  \qquad
 P_{t+1}^{-1}-P_t^{-1} = \ell''(y_t,\hat{\theta}_t^TX_t)X_tX_t^T \,.
\end{align*}
\end{proof}

\subsection{Proofs of Section \ref{section:localized_risk}}\label{app:localized_risk}
\begin{proof}\textbf{of Proposition \ref{prop:comparison12}.} 
We recall that $\mathbb{E}_{y\sim p_{\theta^*}(y\mid X)}[y]=b'(\theta^{*T}X)$, therefore
\begin{align*}
    \mathbb{E}_{y\sim p_{\theta^*}(y\mid X)}\left[\frac{(b'(\theta^TX)-y)(\theta-\theta^*)^TX}{a}\right] & = \frac{(\theta-\theta^*)^TX}{a}\left(b'(\theta^TX)-b'(\theta^{*T}X)\right)\,.
\end{align*}
Thus, there exists $\lambda\in[0,1]$ such that
\begin{align*}
    \mathbb{E}_{y\sim p_{\theta^*}(y\mid X)}\left[\frac{(b'(\theta^TX)-y)(\theta-\theta^*)^TX}{a}\right] & = \frac{(\theta-\theta^*)^TX}{a}b''\left(\theta^TX+\lambda(\theta^*-\theta)^TX\right)(\theta-\theta^*)^TX\,.
\end{align*}
Then we use Assumption \ref{ass:bounded}:
\begin{align*}
	\frac{b''\left(\theta^TX+\lambda(\theta^*-\theta)^TX\right)}{b''\left(\theta^TX\right)}=\frac{\ell''\left(y_t,\theta^TX+\lambda(\theta^*-\theta)^TX\right)}{\ell''\left(y_t,\theta^TX\right)} \ge \rho_{\|\theta-\theta^*\|}\,,
\end{align*}
yielding
\begin{equation}
	\label{eq:ptheta*}
    \mathbb{E}_{y\sim p_{\theta^*}(y\mid X)}\left[\ell'(y,\theta^TX)X\right]^T(\theta-\theta^*) \ge \rho_{\|\theta-\theta^*\|} (\theta-\theta^*)^T \left(\ell''(y,\theta^TX)XX^T\right) (\theta-\theta^*)\,.
\end{equation}

The first-order condition satisfied by $\theta^*$ is
\begin{equation*}
    \mathbb{E}\left[-\frac{(y-b'(\theta^{*T}X))X}{a}\right] = 0\,,
\end{equation*}
which is re-written
\begin{equation*}
	\mathbb{E}\left[yX\right] = \mathbb{E}[b'(\theta^{*T}X)X] = \mathbb{E}[\mathbb{E}_{y\sim p_{\theta^*}(y\mid X)}[y]X]\,.
\end{equation*}
Plugging it into Equation \ref{eq:ptheta*}, we obtain
\begin{align*}
	\mathbb{E}[\ell'(y,\theta^TX)X]^T(\theta-\theta^*) \ge \rho_{\|\theta-\theta^*\|} (\theta-\theta^*)^T \mathbb{E}[\ell''(y,\theta^TX)XX^T] (\theta-\theta^*)\,.
\end{align*}
\end{proof}

\begin{proof}\textbf{of Proposition \ref{prop:secondorder}.} 
We first recall that $L(\theta) - L(\theta^*) \le \frac{\partial L}{\partial \theta}\Bigr|_{\substack{\theta}}^T(\theta-\theta^*)$, then Proposition \ref{prop:comparison12} yields
\begin{equation*}
    \frac{\partial L}{\partial \theta}\Bigr|_{\substack{\theta}}^T(\theta-\theta^*) - c (\theta-\theta^*)^T \frac{\partial^2 L}{\partial \theta^2}\Bigr|_{\substack{\theta}}(\theta-\theta^*) \ge (1-\frac{c}{\rho_{\|\theta-\theta^*\|} }) \frac{\partial L}{\partial \theta}\Bigr|_{\substack{\theta}}^T(\theta-\theta^*) \,,
\end{equation*}
and the result follows.
\end{proof}

\begin{proof}\textbf{of Lemma \ref{lemma:martingale}.}
We first develop $(\Delta M_t)^2$: 
\begin{align*}
    (\Delta M_t)^2 & = \left(\left(\mathbb{E}\left[\nabla_t \mid \mathcal{F}_{t-1}\right]-\nabla_t\right)^T(\hat{\theta}_t-\theta^*)\right)^2 \\
    & = (\hat{\theta}_t-\theta^*)^T\Big(\mathbb{E}[\nabla_t\mid \mathcal{F}_{t-1}]\mathbb{E}[\nabla_t\mid \mathcal{F}_{t-1}]^T+\nabla_t\nabla_t^T \\
    & \quad-\nabla_t\mathbb{E}[\nabla_t\mid \mathcal{F}_{t-1}]^T - \mathbb{E}[\nabla_t\mid \mathcal{F}_{t-1}] \nabla_t^T \Big)(\hat{\theta}_t-\theta^*) \\
    & \le 2 (\hat{\theta}_t-\theta^*)^T\Big(\mathbb{E}[\nabla_t\mid \mathcal{F}_{t-1}]\mathbb{E}[\nabla_t\mid\mathcal{F}_{t-1}]^T+\nabla_t\nabla_t^T\Big)(\hat{\theta}_t-\theta^*) \\
    & \le 2 (\hat{\theta}_t-\theta^*)^T\Big(\mathbb{E}[\nabla_t\nabla_t^T\mid \mathcal{F}_{t-1}]+\nabla_t\nabla_t^T\Big)(\hat{\theta}_t-\theta^*) \,.
\end{align*}
The third line holds because if $U,V\in\R^d$, it holds $-UV^T-VU^T\preccurlyeq UU^T + VV^T$. The last one comes from $\mathbb{E}\Big[(\nabla_t-\mathbb{E}[\nabla_t\mid \mathcal{F}_{t-1}])(\nabla_t-\mathbb{E}[\nabla_t\mid \mathcal{F}_{t-1}])^T\mid \mathcal{F}_{t-1}\Big]\succcurlyeq 0$.

Also, we have the relation
\begin{align*}
    \mathbb{E}[(\Delta M_t)^2\mid \mathcal{F}_{t-1}] \le (\hat{\theta}_t-\theta^*)^T\mathbb{E}[\nabla_t\nabla_t^T\mid \mathcal{F}_{t-1}](\hat{\theta}_t-\theta^*)\,.
\end{align*}

It yields
\begin{equation*}
    (\Delta M_t)^2+\mathbb{E}[(\Delta M_t)^2\mid \mathcal{F}_{t-1}] \le (\hat{\theta}_t-\theta^*)^T\left(3\mathbb{E}[\nabla_t\nabla_t^T\mid \mathcal{F}_{t-1}]+2\nabla_t\nabla_t^T\right)(\hat{\theta}_t-\theta^*) \,,
\end{equation*}
and the result follows from Lemma \ref{lemma:corobercu}.
\end{proof}

We derive the following Lemma in order to control the right-hand side of Lemma \ref{lemma:recursive_bound}, in both settings.
\begin{lemma}\label{lemma:sumtrace}
Assume the second point of Assumption \ref{ass:bounded} holds. For any $k,n\ge 1$, if $\|\hat{\theta}_t-\theta^*\|^2\le \varepsilon$ for any $k<t \le k+n$ then we have
\begin{equation*}
	\sum\limits_{t=k+1}^{k+n} \Tr\left(P_{t+1}(P_{t+1}^{-1}-P_t^{-1})\right) \le d \ln\left(1 +n \frac{h_{\varepsilon}\lambda_{\rm max}(P_{k+1})D_X^2}{d}\right) \,.
\end{equation*}
\end{lemma}

\begin{proof}
We apply Lemma 11.11 of~\cite{cesa2006prediction}:
\begin{align*}
    \sum\limits_{t=k+1}^{k+n} \Tr\left(P_{t+1}(P_{t+1}^{-1}-P_t^{-1})\right) & = \sum\limits_{t=k+1}^{k+n} \left(1-\frac{\det(P_t^{-1})}{\det(P_{t+1}^{-1})}\right)\\
    & \le \sum\limits_{t=k+1}^{k+n} \ln\left(\frac{\det(P_{t+1}^{-1})}{\det(P_t^{-1})}\right) \\
    & = \ln\left(\frac{\det(P_{k+n+1}^{-1})}{\det(P_{k+1}^{-1})}\right) \\
    & \le \ln\det\left(I+ \sum\limits_{t=k+1}^{k+n} \ell''(y_t,\hat{\theta}_t^TX_t) (P_{k+1}^{1/2}X_t)(P_{k+1}^{1/2}X_t)^T \right) \\
    & = \sum\limits_{i=1}^d \ln(1+\lambda_i) \,,
\end{align*}
where $\lambda_1,...,\lambda_d$ are the eigenvalues of $\sum\limits_{t=k+1}^{k+n} \ell''(y_t,\hat{\theta}_t^TX_t) (P_{k+1}^{1/2}X_t)(P_{k+1}^{1/2}X_t)^T$. Therefore we have
\begin{align*}
    \sum\limits_{t=k+1}^{k+n} \Tr\left(P_{t+1}(P_{t+1}^{-1}-P_t^{-1})\right) & \le d\ln\left(1+\frac{1}{d}\sum\limits_{i=1}^d \lambda_i\right) \\
    & \le d\ln\left(1+\frac1d n  h_{\varepsilon}\lambda_{\rm max}(P_{k+1})D_X^2 \right) \,.
\end{align*}
\end{proof}

\subsection{Bounded setting (Assumption \ref{ass:bounded})}\label{app:localized_bounded}
\begin{proof}\textbf{of Theorem \ref{th:localized_bounded}.}
Let $\delta>0$.
On the one hand, we sum Lemma \ref{lemma:recursive_bound} and \ref{lemma:martingale}. We obtain, for any $\lambda>0$,
\begin{align}
	\nonumber & \sum\limits_{t=\tau(\varepsilon,\delta)+1}^{\tau(\varepsilon,\delta)+n} \left(\mathbb{E}[\nabla_t\mid \mathcal{F}_{t-1}]^T(\hat{\theta}_t-\theta^*) - \frac12 Q_t - \lambda (\hat{\theta}_t-\theta^*)^T\left(\nabla_t\nabla_t^T +\frac32 \mathbb{E}\left[\nabla_t\nabla_t^T\mid \mathcal{F}_{t-1}\right] \right)(\hat{\theta}_t-\theta^*) \right) \\
	\label{eq:without_localized}& \qquad \le \frac{1}{2} \sum\limits_{t=\tau(\varepsilon,\delta)+1}^{\tau(\varepsilon,\delta)+n} X_t^TP_{t+1}X_t \ell'(y_t,\hat{\theta}_t^TX_t)^2 + \frac{\|\hat{\theta}_{1}-\theta^*\|^2}{\lambda_{\rm min}(P_{\tau(\varepsilon,\delta)+1})} + \frac{\ln\delta^{-1}}{\lambda},\qquad n\ge 1\,,
\end{align}
with probability at least $1-\delta$, where we define $Q_t = (\hat{\theta}_t-\theta^*)^T \left(\ell''(y_t,\hat{\theta}_t^TX_t)X_tX_t^T\right) (\hat{\theta}_t-\theta^*)$ for any $t$.

On the other hand, thanks to Assumption \ref{ass:bounded}, we can apply Proposition \ref{prop:secondorder} with $c=0.75$ to obtain, for any $t\ge 1$,
\begin{align}
	\nonumber 
	\|\hat{\theta}_t-\theta^*\|\le \varepsilon & \implies L(\hat{\theta}_t)-L(\theta^*) \le \frac{\rho_{\varepsilon}}{\rho_{\varepsilon}-0.75} \Big(\frac{\partial L}{\partial \theta}\Bigr|_{\substack{\hat{\theta}_{t}}}^T(\hat{\theta}_t-\theta^*) - 0.75 (\hat{\theta}_t-\theta^*)^T\frac{\partial^2 L}{\partial \theta^2}\Bigr|_{\substack{\hat{\theta}_{t}}}(\hat{\theta}_t-\theta^*)\Big) \,, \\
	 \label{eq:with_localized}
	 & \implies L(\hat{\theta}_t)-L(\theta^*) \le 5 \Big(\mathbb{E}[\nabla_t\mid \mathcal{F}_{t-1}]^T(\hat{\theta}_t-\theta^*) - 0.75 \mathbb{E}[Q_t\mid \mathcal{F}_{t-1}]\Big) \,,
\end{align}
because $\rho_{\varepsilon}>0.95$.

In order to bridge the gap between Equations \eqref{eq:without_localized} and \eqref{eq:with_localized}, we need to control the quadratic terms of Equation \eqref{eq:without_localized} with $\mathbb{E}[Q_t\mid \mathcal{F}_{t-1}]$. First, for any $t$, if $\|\hat{\theta}_t-\theta^*\|\le \varepsilon$, we have $Q_t\in[0, h_{\varepsilon}\varepsilon^2D_X^2]$, and we apply Lemma A.3 of~\cite{cesa2006prediction} to the random variable $\frac{1}{h_{\varepsilon}\varepsilon^2D_X^2}Q_t\in[0,1]$: for any $s>0$,
\begin{equation*}
	\mathbb{E}\left[\exp\left(\frac{s}{h_{\varepsilon}\varepsilon^2D_X^2} Q_t - \frac{e^s-1}{h_{\varepsilon}\varepsilon^2D_X^2}\mathbb{E}\left[Q_t \mid \mathcal{F}_{t-1}\right]\right) \mid \mathcal{F}_{t-1}, \|\hat{\theta}_t-\theta^*\|\le \varepsilon \right] \le 1\,.
\end{equation*}
We fix $s=0.1$ and we define
\begin{align*}
	V_n = \exp\left( \sum\limits_{t=\tau(\varepsilon,\delta)+1}^{\tau(\varepsilon,\delta)+n}\left(\frac{0.1}{h_{\varepsilon}\varepsilon^2D_X^2} Q_t - (e^{0.1}-1)\mathbb{E}\left[\frac{1}{h_{\varepsilon}\varepsilon^2D_X^2} Q_t \mid \mathcal{F}_{t-1}\right]\right)\right)\,.
\end{align*}
The sequence $(V_n)$ is adapted to $(\mathcal{F}_{\tau(\varepsilon,\delta)+n})$, almost surely we have $V_0=1$ and $V_n\ge 0$. Finally,
\begin{align*}
	\mathbb{E}[V_n\mid \mathcal{F}_{\tau(\varepsilon,\delta)+n-1}, \|\hat{\theta}_{\tau(\varepsilon,\delta)+n}-\theta^*\|\le \varepsilon |\le V_{n-1}\,,
\end{align*}
and $(\|\hat{\theta}_{\tau(\varepsilon,\delta)+n}-\theta^*\|\le \varepsilon)$ belongs to $ \mathcal{F}_{\tau(\varepsilon,\delta)+n-1}$.
We apply Lemma \ref{lemma:simultaneous_supermartingale}:
\begin{equation*}
	\mathbb{P}\left(\left(\bigcup\limits_{n=1}^{\infty} V_n > \delta^{-1}\right) \cup \left(\bigcup\limits_{n=1}^{\infty} (\|\hat{\theta}_{\tau(\varepsilon,\delta)+n}-\theta^*\|> \varepsilon)\right) \right) \le \delta + \mathbb{P}\left(\bigcup\limits_{n=1}^{\infty} (\|\hat{\theta}_{\tau(\varepsilon,\delta)+n}-\theta^*\|> \varepsilon)\right) \,.
\end{equation*}
We define $A_k^{\varepsilon}=\bigcap\limits_{n=k+1}^{\infty} (\|\hat{\theta}_{n}-\theta^*\|\le \varepsilon)$ for any $k$. The last inequality is equivalent to
\begin{align}
	\label{eq:stopping_time_localized}
	& \mathbb{P}\left(\bigcup\limits_{n=1}^{\infty}\left(\sum\limits_{t=\tau(\varepsilon,\delta)+1}^{\tau(\varepsilon,\delta)+n}Q_t > 10(e^{0.1}-1) \sum\limits_{t=\tau(\varepsilon,\delta)+1}^{\tau(\varepsilon,\delta)+n} \mathbb{E}\left[Q_t \mid \mathcal{F}_{t-1}\right] + 10h_{\varepsilon}\varepsilon^2D_X^2 \ln\delta^{-1}\right) \cap A_{\tau(\varepsilon,\delta)}^{\varepsilon}\right) \le\delta\,.
\end{align}

We then bound the two quadratic terms coming from Lemma \ref{lemma:martingale}: using Assumption \ref{ass:bounded} we have the implications
\begin{align*}
	& \|\hat{\theta}_t-\theta^*\| \le \varepsilon \implies (\hat{\theta}_t-\theta^*)^T\nabla_t\nabla_t^T(\hat{\theta}_t-\theta^*) \le \kappa_{\varepsilon} Q_t \,, \\
	& \|\hat{\theta}_t-\theta^*\| \le \varepsilon \implies (\hat{\theta}_t-\theta^*)^T\mathbb{E}\left[\nabla_t\nabla_t^T\mid \mathcal{F}_{t-1}\right](\hat{\theta}_t-\theta^*) \le \kappa_{\varepsilon} \mathbb{E}\left[Q_t \mid \mathcal{F}_{t-1} \right] \,.
\end{align*}
Therefore, we get from \eqref{eq:stopping_time_localized}
\begin{align*}
	& \mathbb{P}\left(\bigcup\limits_{n=1}^{\infty}\left(\sum\limits_{t=\tau(\varepsilon,\delta)+1}^{\tau(\varepsilon,\delta)+n}\left(\frac12 Q_t + \lambda(\hat{\theta}_t-\theta^*)^T\nabla_t\nabla_t^T(\hat{\theta}_t-\theta^*) + \frac32\lambda (\hat{\theta}_t-\theta^*)^T\mathbb{E}\left[\nabla_t\nabla_t^T\mid \mathcal{F}_{t-1}\right](\hat{\theta}_t-\theta^*) \right)> \right.\right. \\
	& \quad  \left. \left. \left(10(e^{0.1}-1)(\frac12 + \lambda \kappa_{\varepsilon})+ \frac32 \lambda \kappa_{\varepsilon}\right) \sum\limits_{t=\tau(\varepsilon,\delta)+1}^{\tau(\varepsilon,\delta)+n} \mathbb{E}\left[Q_t\mid \mathcal{F}_{t-1} \right] + 10(\frac12 + \lambda \kappa_{\varepsilon}) h_{\varepsilon}\varepsilon^2D_X^2 \ln\delta^{-1}\right) \cap A_{\tau(\varepsilon,\delta)}^{\varepsilon}\right) \\
	& \qquad \le \delta \,.
\end{align*}

We set $\lambda = \frac{0.75-5(e^{0.1}-1)}{(10(e^{0.1}-1)+\frac32)\kappa_{\varepsilon}}$, so that
\begin{align*}
	& 10(e^{0.1}-1)(\frac12 + \lambda \kappa_{\varepsilon})+ \frac32 \lambda \kappa_{\varepsilon} = 0.75\,, \\
	& \frac12 + \lambda \kappa_{\varepsilon} = \frac12 + \frac{0.75-5(e^{0.1}-1)}{10(e^{0.1}-1)+\frac32} \approx 0.59 \le 0.6\,,
\end{align*}
and consequently
\begin{align*}
	& \mathbb{P}\Bigg(\bigcup\limits_{n=1}^{\infty}\Bigg(\sum\limits_{t=\tau(\varepsilon,\delta)+1}^{\tau(\varepsilon,\delta)+n}\left(\mathbb{E}[\nabla_t\mid \mathcal{F}_{t-1}]^T(\hat{\theta}_t-\theta^*) - 0.75 \mathbb{E}[Q_t\mid \mathcal{F}_{t-1}] \right) > 6h_{\varepsilon}\varepsilon^2D_X^2 \ln\delta^{-1} \\
	& \quad  + \sum\limits_{t=\tau(\varepsilon,\delta)+1}^{\tau(\varepsilon,\delta)+n} \left(\mathbb{E}[\nabla_t\mid \mathcal{F}_{t-1}]^T(\hat{\theta}_t-\theta^*) - \frac12 Q_t - \lambda (\hat{\theta}_t-\theta^*)^T\left(\nabla_t\nabla_t^T +\frac32 \mathbb{E}\left[\nabla_t\nabla_t^T\mid \mathcal{F}_{t-1}\right] \right)(\hat{\theta}_t-\theta^*) \right)\Bigg) \\
	& \qquad \cap A_{\tau(\varepsilon,\delta)}^{\varepsilon}\Bigg) \le \delta \,.
\end{align*}
We plug Equation \eqref{eq:with_localized} in the last inequality:
\begin{align*}
	& \mathbb{P}\Bigg(\bigcup\limits_{n=1}^{\infty}\Bigg( \sum\limits_{t=\tau(\varepsilon,\delta)+1}^{\tau(\varepsilon,\delta)+n} (L(\hat{\theta}_t)-L(\theta^*)) > 30 h_{\varepsilon}\varepsilon^2D_X^2 \ln\delta^{-1}  \\
	& \qquad\qquad+ 5\sum\limits_{t=\tau(\varepsilon,\delta)+1}^{\tau(\varepsilon,\delta)+n} \bigg(\mathbb{E}[\nabla_t\mid \mathcal{F}_{t-1}]^T(\hat{\theta}_t-\theta^*) - \frac12 Q_t \\
	& \qquad\qquad\qquad\qquad\qquad - \lambda (\hat{\theta}_t-\theta^*)^T\left(\nabla_t\nabla_t^T +\frac32 \mathbb{E}\left[\nabla_t\nabla_t^T\mid \mathcal{F}_{t-1}\right] \right)(\hat{\theta}_t-\theta^*) \bigg)\Bigg) \cap A_{\tau(\varepsilon,\delta)}^{\varepsilon}\Bigg) \le \delta \,.
\end{align*}
We then use Equation \eqref{eq:without_localized} with $\frac{1}{\lambda} = \frac{(10(e^{0.1}-1)+\frac32)\kappa_{\varepsilon}}{0.75-5(e^{0.1}-1)} \approx 11.4\kappa_{\varepsilon} \le 12 \kappa_{\varepsilon}$. It yields
\begin{align*}
	& \mathbb{P}\Bigg(\bigcup\limits_{n=1}^{\infty}\Bigg( \sum\limits_{t=\tau(\varepsilon,\delta)+1}^{\tau(\varepsilon,\delta)+n} (L(\hat{\theta}_t)-L(\theta^*)) > \frac{5}{2} \sum\limits_{t=\tau(\varepsilon,\delta)+1}^{\tau(\varepsilon,\delta)+n} X_t^TP_{t+1}X_t \ell'(y_t,\hat{\theta}_t^TX_t)^2 \\
	& \qquad\qquad\qquad\qquad + \frac{5\|\hat{\theta}_{1}-\theta^*\|^2}{\lambda_{\rm min}(P_{\tau(\varepsilon,\delta)+1})} + 30(2\kappa_{\varepsilon} + h_{\varepsilon}\varepsilon^2D_X^2) \ln\delta^{-1} \Bigg) \cap A_{\tau(\varepsilon,\delta)}^{\varepsilon}\Bigg) \le 2\delta \,.
\end{align*}

Thanks to Assumption \ref{ass:bounded}, we have
\begin{equation*}
    X_t^TP_{t+1}X_t\ell'(y_t,\hat{\theta}_t^TX_t)^2 \le \kappa_{\varepsilon} \Tr\left(P_{t+1}(P_{t+1}^{-1}-P_t^{-1})\right),\qquad  t > \tau(\varepsilon,\delta) \,,
\end{equation*}
therefore we apply Lemma \ref{lemma:sumtrace}: for any $n$, it holds
\begin{align*}
	\sum\limits_{t=\tau(\varepsilon,\delta)+1}^{\tau(\varepsilon,\delta)+n} X_t^TP_{t+1}X_t\ell'(y_t,\hat{\theta}_t^TX_t)^2 \le d \kappa_{\varepsilon} \ln\left(1 +n \frac{h_{\varepsilon}\lambda_{\rm max}(P_{\tau(\varepsilon,\delta)+1})D_X^2}{d}\right)\,.
\end{align*}
As $P_{\tau(\varepsilon,\delta)+1}\preccurlyeq P_1$, we obtain
\begin{align*}
	& \mathbb{P}\Bigg(\bigcup\limits_{n=1}^{\infty}\Bigg( \sum\limits_{t=\tau(\varepsilon,\delta)+1}^{\tau(\varepsilon,\delta)+n} (L(\hat{\theta}_t)-L(\theta^*)) > \frac{5}{2} d \kappa_{\varepsilon} \ln\left(1 +n \frac{h_{\varepsilon}\lambda_{\rm max}(P_{\tau(\varepsilon,\delta)+1})D_X^2}{d}\right) \\
	& \qquad\qquad\qquad\qquad + \frac{5\|\hat{\theta}_{1}-\theta^*\|^2}{\lambda_{\rm min}(P_{\tau(\varepsilon,\delta)+1})} + 30(2\kappa_{\varepsilon} + h_{\varepsilon}\varepsilon^2D_X^2) \ln\delta^{-1} \Bigg) \cap A_{\tau(\varepsilon,\delta)}^{\varepsilon}\Bigg) \le 2\delta \,.
\end{align*}

To conclude, we use Assumption \ref{ass:localized}.
\end{proof}

\subsection{Quadratic setting (Assumption \ref{ass:subgaussian})}\label{app:localized_quadratic}
We recall two definitions introduced in the previous subsection:
\begin{align*}
	A_k^{\varepsilon}=\bigcap\limits_{n=k+1}^{\infty} (\|\hat{\theta}_{n}-\theta^*\|\le \varepsilon)&,\qquad k\ge 1\,, \\
	Q_t=(\hat{\theta}_t-\theta^*)^TX_tX_t^T(\hat{\theta}_t-\theta^*)&,\qquad t\ge 1\,.
\end{align*}
The sub-gaussian hypothesis requires a different treatment of several steps in the proof. In the following proofs, we use a consequence of the first points of Assumption \ref{ass:subgaussian}. We apply Lemma 1.4 of~\cite{rigollet2015high}: for any $X\in\R^d$,
\begin{align}
	\label{eq:rigollet}
	\mathbb{E}[(y-\mathbb{E}[y\mid X])^{2i} \mid X] \le 2 i (2\sigma^2)^i \Gamma(i) = 2 (2\sigma^2)^i i!,\qquad i\in\mathbb{N}^* \,.
\end{align}

First, we control the quadratic terms in $\nabla_t=-(y_t-\hat{\theta}_t^TX_t)X_t$ in the following lemma. 

\begin{lemma}\label{lemma:quadratic_linear}
\begin{enumerate}
    \item
    For any $k\in\mathbb{N}$ and $\delta>0$, we have
    \begin{align*}
    		& \mathbb{P}\Bigg(\bigcup\limits_{n=1}^{\infty} \Bigg(\sum\limits_{t=k+1}^{k+n}(\hat{\theta}_t-\theta^*)^T\nabla_t\nabla_t^T(\hat{\theta}_t-\theta^*) \\
    		& \qquad\qquad > 3 \left(8\sigma^2 + D_{\rm app}^2 + \varepsilon^2 D_X^2 \right) \sum\limits_{t=k+1}^{k+n} Q_t + 12 \varepsilon^2 D_X^2 \sigma^2\ln\delta^{-1}\Bigg) \cap A_{k}^{\varepsilon}\Bigg) \le \delta \,.
    	\end{align*}
    
    \item
    For any $t$, it holds almost surely
    \begin{align*}
    		(\hat{\theta}_t-\theta^*)^T\mathbb{E}[\nabla_t\nabla_t^T\mid \mathcal{F}_{t-1}](\hat{\theta}_t-\theta^*) \le 3\left(\sigma^2 + D_{\rm app}^2 + \|\hat{\theta}_t-\theta^*\|^2 D_X^2\right) \mathbb{E}[Q_t\mid \mathcal{F}_{t-1}] \,.
    	\end{align*}
\end{enumerate}
\end{lemma}

\begin{proof} 
\begin{enumerate}
\item
We recall that for any $a,b,c$, we have $(a+b+c)^2\le 3(a^2+b^2+c^2)$. Thus
\begin{align}
    \nonumber
    (\hat{\theta}_t-\theta^*)^T\nabla_t\nabla_t^T(\hat{\theta}_t-\theta^*) & = Q_t (y_t-\hat{\theta}_t^TX_t)^2 \\
    \nonumber& \le 3 Q_t \Big((y_t - \mathbb{E}[y_t\mid X_t])^2 + (\mathbb{E}[y_t\mid X_t]-\theta^{*T}X_t)^2 + ((\theta^*-\hat{\theta}_t)^TX_t)^2 \Big) \\
    \label{eq:decompZZT}
    & \le 3 Q_t \left((y_t - \mathbb{E}[y_t\mid X_t])^2 + D_{\rm app}^2 + \|\hat{\theta}_t-\theta^*\|^2D_X^2 \right)\,.
\end{align}
To obtain the last inequality, we use the second point of Assumption \ref{ass:subgaussian} to bound the middle term.
Then we use Taylor series for the exponential, and we apply Equation \eqref{eq:rigollet}. For any $t$ and any $\mu$ satisfying $0<\mu \le \frac{1}{4 Q_t \sigma^2}$, we have
\begin{align*}
    \mathbb{E}\left[\exp\left(\mu Q_t (y_t-\mathbb{E}[y_t\mid X_t])^2\right)\mid \mathcal{F}_{t-1}, X_t \right] & = 1 + \sum\limits_{i\ge 1} \frac{\mu^i Q_t^{i} \mathbb{E}[(y_t-\mathbb{E}\left[y_t\mid X_t])^{2i} \mid X_t\right]}{i!} \\
    & \le 1 + 2 \sum\limits_{i\ge 1} \frac{\mu^i Q_t^i i! (2\sigma^2)^i}{i!} \\
    & \le 1 + 2 \sum\limits_{i\ge 1} \left(2\mu Q_t \sigma^2\right)^i \\
    & \le 1 + 8 \mu Q_t \sigma^2, \qquad 2\mu Q_t \sigma^2\le \frac{1}{2} \\
    & \le \exp\left(8\mu Q_t \sigma^2\right)\,.
\end{align*}
Therefore, for any $t$,
\begin{align*}
	\mathbb{E}\left[\exp\left(\frac{1}{4\varepsilon^2D_X^2\sigma^2} Q_t \left( (y_t-\mathbb{E}[y_t\mid X_t])^2 - 8\sigma^2\right) \right)\mid \mathcal{F}_{t-1}, X_t, \|\hat{\theta}_t-\theta^*\|\le \varepsilon \right] \le 1\,.
\end{align*}

We define the random variable
\begin{equation*}
	V_n = \exp\left(\frac{1}{4 \varepsilon^2D_X^2 \sigma^2} \sum\limits_{t=k+1}^{k+n} Q_t \left((y_t-\mathbb{E}[y_t\mid X_t])^2-8\sigma^2\right)\right),\qquad n\in \mathbb{N} \,.
\end{equation*}
$(V_n)_n$ is adapted to the filtration $(\sigma(X_1,y_1,...,X_{k+n},y_{k+n},X_{k+n+1})_n$, moreover $V_0=1$ and $V_n\ge 0$ almost surely, and
\begin{align*}
	\mathbb{E}[V_n\mid X_1,y_1,...,X_{k+n-1},y_{k+n-1},X_{k+n},\|\hat{\theta}_{k+n}-\theta^*\|\le \varepsilon]\le V_{n-1}\,.
\end{align*}
Therefore we apply Lemma \ref{lemma:simultaneous_supermartingale}: for any $\delta>0$,
\begin{align*}
	\mathbb{P}\left(\bigcup\limits_{n=1}^{\infty}(V_n > \delta^{-1}) \cap A_{k}^{\varepsilon}\right) \le \delta \,,
\end{align*}
which is equivalent to
\begin{align*}
    \mathbb{P}\Bigg(\bigcup\limits_{n=1}^{\infty} \Bigg(\sum\limits_{t=k+1}^{k+n} Q_t (y_t-\mathbb{E}[y_t\mid X_t])^2 > 8\sigma^2 \sum\limits_{t=k+1}^{k+n} Q_t + 4 \varepsilon^2D_X^2 \sigma^2\ln\delta^{-1}\Bigg) \cap A_{k}^{\varepsilon}\Bigg) \le \delta \,.
\end{align*}

Substituting in Equation \eqref{eq:decompZZT}, we obtain the desired result.

\item
We apply the same decomposition as for Equation \ref{eq:decompZZT}: for any $t$,
\begin{align*}
    & (\hat{\theta}_t-\theta^*)^T\mathbb{E}[\nabla_t\nabla_t^T\mid \mathcal{F}_{t-1}](\hat{\theta}_t-\theta^*) \\
    & \qquad \le 3 (\hat{\theta}_t-\theta^*)^T\mathbb{E}\bigg[X_tX_t^T \Big((y_t - \mathbb{E}[y_t\mid X_t])^2 + D_{\rm app}^2 + \|\theta^*-\hat{\theta}_t\|^2D_X^2 \Big)\mid \mathcal{F}_{t-1}\bigg](\hat{\theta}_t-\theta^*) \,.
\end{align*}
Assumption \ref{ass:subgaussian} implies that for any $X_t$, $\mathbb{E}[(y_t-\mathbb{E}[y_t\mid X_t])^2 \mid X_t]\le \sigma^2$. Thus, the tower property yields
\begin{align*}
    & (\hat{\theta}_t-\theta^*)^T\mathbb{E}[\nabla_t\nabla_t^T\mid \mathcal{F}_{t-1}](\hat{\theta}_t-\theta^*) \\
    & \qquad \le 3 \left(\sigma^2 + D_{\rm app}^2 + \|\hat{\theta}_t-\theta^*\|^2 D_X^2\right) (\hat{\theta}_t-\theta^*)^T\mathbb{E}[X_tX_t^T\mid \mathcal{F}_{t-1}](\hat{\theta}_t-\theta^*) \,.
\end{align*}
\end{enumerate}
\end{proof}

Second, we bound the right-hand side of Lemma \ref{lemma:recursive_bound}, that is the objective of the following lemma.
\begin{lemma}
\label{lemma:sumtrace_linear}
Let $k\in\mathbb{N}$. For any $\delta>0$, we have
\begin{align*}
	\mathbb{P}\Bigg(\bigcup\limits_{n=1}^{\infty}\Bigg( \sum\limits_{t=k+1}^{k+n} X_t^TP_{t+1}X_t (y_t-\hat{\theta}_t^TX_t)^2 >\ & 3 \left(8\sigma^2+ D_{\rm app}^2 + \varepsilon^2 D_X^2\right) d \ln\left(1 +n \frac{\lambda_{\rm max}(P_{k+1})D_X^2}{d}\right) \\
	& + 12 \lambda_{\rm max}(P_{1})D_X^2\sigma^2 \ln\delta^{-1} \Bigg) \cap A_{k}^{\varepsilon}\Bigg) \le \delta \,.
\end{align*}
\end{lemma}

\begin{proof}
We apply a similar analysis as in the proof of Lemma \ref{lemma:quadratic_linear} in order to use the sub-gaussian assumption, and then we apply the telescopic argument as in the bounded setting. We decompose $y_t-\hat{\theta}_t^TX_t$:
\begin{align}
    \nonumber X_t^TP_{t+1}X_t (y_t-\hat{\theta}_t^TX_t)^2 & \le 3 X_t^TP_{t+1}X_t \Big((y_t-\mathbb{E}[y_t\mid X_t])^2 + (\mathbb{E}[y_t\mid X_t]-b'(\theta^{*T}X_t))^2 + ((\theta^*-\hat{\theta}_t)^TX_t)^2\Big)\\
    \label{eq:decompXPX}& \le 3 X_t^TP_{t+1}X_t \left((y_t-\mathbb{E}[y_t\mid X_t])^2+ D_{\rm app}^2 + \|\hat{\theta}_t-\theta^*\|^2 D_X^2\right)  \,.
\end{align}
To control $(y_t - \mathbb{E}[y_t\mid X_t])^2 X_t^TP_{t+1}X_t$, we use its positivity along with Equation \eqref{eq:rigollet}. Precisely, for any $t$ and any $\mu>0$ satisfying $0<\mu\le \frac{1}{4 X_t^TP_{t+1}X_t \sigma^2}$, we have
\begin{align*}
    \mathbb{E}\left[\exp\left(\mu (y_t - \mathbb{E}[y_t\mid X_t])^2 X_t^TP_{t+1}X_t\right)\mid \mathcal{F}_{t-1}, X_t \right] & = 1 + \sum\limits_{i\ge 1} \frac{\mu^i (X_t^TP_{t+1}X_t)^i \mathbb{E}\left[(y_t-\mathbb{E}[y_t\mid X_t])^{2i} \mid X_t\right]}{i!} \\
    & \le 1 + 2 \sum\limits_{i\ge 1} \frac{\mu^i (X_t^TP_{t+1}X_t)^i i! (2\sigma^2)^i}{i!} \\
    & = 1 + 2 \sum\limits_{i\ge 1} \left(2\mu X_t^TP_{t+1}X_t \sigma^2\right)^i \\
    & \le 1 + 8\mu X_t^TP_{t+1}X_t \sigma^2, \qquad 0<2\mu X_t^TP_{t+1}X_t \sigma^2\le \frac{1}{2} \\
    & \le \exp\left(8\mu X_t^TP_{t+1}X_t \sigma^2\right)\,.
\end{align*}
We apply the previous bound with a uniform $\mu=\frac{1}{4 \lambda_{\rm max}(P_1) D_X^2 \sigma^2}$, and as $\lambda_{\rm max}(P_{t+1}) \le \lambda_{\rm max}(P_1)$ for any $t$, we get $\mu\le \frac{1}{4 X_t^TP_{t+1}X_t \sigma^2}$. Thus, we define
\begin{align*}
    V_n = \exp\left(\frac{1}{4 \lambda_{\rm max}(P_1) D_X^2 \sigma^2}\sum\limits_{t=k+1}^{k+n} \left((y_t - \mathbb{E}[y_t\mid X_t])^2 - 8\sigma^2\right) X_t^TP_{t+1}X_t \right),\qquad n\in\mathbb{N}\,.
\end{align*}
$(V_n)$ is a super-martingale adapted to the filtration $(\sigma(X_1,y_1,...,X_{k+n-1},y_{k+n-1},X_{k+n}))_n$ satisfying almost surely $V_0=1,V_n\ge 0$, thus we apply Lemma \ref{lemma:simultaneous_supermartingale}:
\begin{align*}
	\mathbb{P}\left(\bigcup\limits_{n=1}^{\infty} (V_n > \delta^{-1})\right) \le \delta \,,
\end{align*}
or equivalently
\begin{align*}
	& \mathbb{P}\Bigg(\bigcup\limits_{n=1}^{\infty} \Bigg( \sum\limits_{t=k+1}^{k+n} X_t^TP_{t+1}X_t (y_t-\mathbb{E}[y_t\mid X_t])^2 >  8\sigma^2 \sum\limits_{t=k+1}^{k+n} X_t^TP_{t+1}X_t + 4 \lambda_{\rm max}(P_{1})D_X^2\sigma^2 \ln\delta^{-1} \Bigg)\Bigg) \le \delta \,.
\end{align*}

Combining it with Equation \eqref{eq:decompXPX}, we get
\begin{align*}
	\mathbb{P}\Bigg(\bigcup\limits_{n=1}^{\infty}\Bigg( \sum\limits_{t=k+1}^{k+n} X_t^TP_{t+1}X_t (y_t-\hat{\theta}_t^TX_t)^2 >\ & 3 \left(8\sigma^2+ D_{\rm app}^2 + \varepsilon^2 D_X^2\right) \sum\limits_{t=k+1}^{k+n}X_t^TP_{t+1}X_t \\
	& + 12 \lambda_{\rm max}(P_{1})D_X^2\sigma^2 \ln\delta^{-1} \Bigg) \cap A_{k}^{\varepsilon}\Bigg) \le \delta \,.
\end{align*}

Then we apply Lemma \ref{lemma:sumtrace}: the second point of Assumption \ref{ass:bounded} holds with $h_{\varepsilon}=1$, thus
\begin{equation*}
	\sum\limits_{t=k+1}^{k+n} \Tr\left(P_{t+1}(P_{t+1}^{-1}-P_t^{-1})\right) \le d \ln\left(1 +n \frac{\lambda_{\rm max}(P_{k+1})D_X^2}{d}\right),\qquad n\ge 1.
\end{equation*}
We conclude with $X_t^TP_{t+1}X_t = \Tr(P_{t+1}(P_{t+1}^{-1}-P_t^{-1}))$.
\end{proof}

We sum up our findings and we prove the result for the quadratic loss. The structure of the proof is the same as the one of Theorem \ref{th:localized_bounded}.
\begin{proof}\textbf{of Theorem \ref{th:localized_linear}.}
On the one hand, we sum Lemma \ref{lemma:recursive_bound} and Lemma \ref{lemma:martingale}: for any $\lambda,\delta>0$,
\begin{multline}\label{eq:martingale_linear}
	\sum\limits_{t=\tau(\varepsilon,\delta)+1}^{\tau(\varepsilon,\delta)+n} \left(\mathbb{E}[\nabla_t\mid \mathcal{F}_{t-1}]^T(\hat{\theta}_t-\theta^*) - \frac12 Q_t - \lambda (\hat{\theta}_t-\theta^*)^T\left(\nabla_t\nabla_t^T +\frac32 \mathbb{E}\left[\nabla_t\nabla_t^T\mid \mathcal{F}_{t-1}\right] \right)(\hat{\theta}_t-\theta^*) \right) \\
	\le \frac{1}{2} \sum\limits_{t=\tau(\varepsilon,\delta)+1}^{\tau(\varepsilon,\delta)+n} X_t^TP_{t+1}X_t (y_t-\hat{\theta}_t^TX_t)^2 + \frac{\|\hat{\theta}_{\tau(\varepsilon,\delta)+1}-\theta^*\|^2}{\lambda_{\rm min}(P_{\tau(\varepsilon,\delta)+1})} + \frac{\ln\delta^{-1}}{\lambda},\qquad n\ge 1\,,
\end{multline}
with probability at least $1-\delta$.
On the other hand, we have
\begin{align}
	\label{eq:expansion_linear}
	\sum\limits_{t=\tau(\varepsilon,\delta)+1}^{\tau(\varepsilon,\delta)+n} (L(\hat{\theta}_t)-L(\theta^*)) \le \frac{1}{1-0.8} \sum\limits_{t=\tau(\varepsilon,\delta)+1}^{\tau(\varepsilon,\delta)+n} \left(\mathbb{E}[\nabla_t\mid \mathcal{F}_{t-1}]^T(\hat{\theta}_t-\theta^*) - 0.8 \mathbb{E}[Q_t\mid \mathcal{F}_{t-1}]\right) \,.
\end{align}
We aim to relate Equations \eqref{eq:martingale_linear} and \eqref{eq:expansion_linear} as in the proof of Theorem \ref{th:localized_bounded}.
To that end, we apply Lemma \ref{lemma:quadratic_linear}:
\begin{align*}
	& \mathbb{P}\Bigg(\bigcup\limits_{n=1}^{\infty}\Bigg(\sum\limits_{t=\tau(\varepsilon,\delta)+1}^{\tau(\varepsilon,\delta)+n} \left(\frac12 Q_t + \lambda (\hat{\theta}_t-\theta^*)^T\left(\nabla_t\nabla_t^T +\frac32 \mathbb{E}\left[\nabla_t\nabla_t^T\mid \mathcal{F}_{t-1}\right] \right)(\hat{\theta}_t-\theta^*)\right) \\
	& \qquad\qquad > \Big(\frac12+3\lambda(8\sigma^2 + D_{\rm app}^2 + \varepsilon^2 D_X^2)\Big) \sum\limits_{t=\tau(\varepsilon,\delta)+1}^{\tau(\varepsilon,\delta)+n} Q_t \\
	& \qquad\qquad\qquad + \frac92 \lambda \Big(\sigma^2 + D_{\rm app}^2 + \varepsilon^2 D_X^2\Big) \sum\limits_{t=\tau(\varepsilon,\delta)+1}^{\tau(\varepsilon,\delta)+n} \mathbb{E}\left[Q_t \mid \mathcal{F}_{t-1}\right] + 12 \lambda \varepsilon^2 D_X^2 \sigma^2 \ln\delta^{-1}\Bigg) \cap A_{\tau(\varepsilon,\delta)}^{\varepsilon}\Bigg) \le \delta\,.
\end{align*}

As in the proof of Theorem \ref{th:localized_bounded} we apply Lemma A.3 of \citep{cesa2006prediction} and Lemma \ref{lemma:simultaneous_supermartingale}: for any $\delta>0$,
\begin{align*}
	& \mathbb{P}\left(\bigcup\limits_{n=1}^{\infty}\left(\sum\limits_{t=\tau(\varepsilon,\delta)+1}^{\tau(\varepsilon,\delta)+n}Q_t > 10(e^{0.1}-1) \sum\limits_{t=\tau(\varepsilon,\delta)+1}^{\tau(\varepsilon,\delta)+n} \mathbb{E}\left[Q_t \mid \mathcal{F}_{t-1}\right] + 10\varepsilon^2D_X^2 \ln\delta^{-1}\right) \cap A_{\tau(\varepsilon,\delta)}^{\varepsilon}\right) \le\delta\,.
\end{align*}

We combine the last two inequalities:
\begin{align}
	\nonumber & \mathbb{P}\Bigg(\bigcup\limits_{n=1}^{\infty}\Bigg(\sum\limits_{t=\tau(\varepsilon,\delta)+1}^{\tau(\varepsilon,\delta)+n} \left(\frac12 Q_t + \lambda (\hat{\theta}_t-\theta^*)^T\left(\nabla_t\nabla_t^T +\frac32 \mathbb{E}\left[\nabla_t\nabla_t^T\mid \mathcal{F}_{t-1}\right] \right)(\hat{\theta}_t-\theta^*)\right) \\
	\nonumber & \qquad > \bigg(10(e^{0.1}-1)\Big(\frac12+3\lambda(8\sigma^2 + D_{\rm app}^2 + \varepsilon^2 D_X^2)\Big) + \frac92 \lambda (\sigma^2 + D_{\rm app}^2 + \varepsilon^2 D_X^2)\bigg) \sum\limits_{t=\tau(\varepsilon,\delta)+1}^{\tau(\varepsilon,\delta)+n} \mathbb{E}\left[Q_t \mid \mathcal{F}_{t-1}\right] \\
	\label{eq:quadratic_linear} & \qquad \qquad + \left(10\varepsilon^2 D_X^2 \Big(\frac12+3\lambda(8\sigma^2 + D_{\rm app}^2 + \varepsilon^2 D_X^2)\Big) + 12 \lambda \varepsilon^2 D_X^2 \sigma^2\right) \ln\delta^{-1}\Bigg) \cap A_{\tau(\varepsilon,\delta)}^{\varepsilon}\Bigg) \le 2\delta\,.
\end{align}
We set
\begin{align*}
	\lambda = \left(0.8 - 5(e^{0.1}-1)\right)\left(30(e^{0.1}-1)(8\sigma^2 + D_{\rm app}^2 + \varepsilon^2 D_X^2) + \frac{9}{2} (\sigma^2 + D_{\rm app}^2 + \varepsilon^2 D_X^2)\right)^{-1}
\end{align*}
in order to obtain
\begin{align*}
	& 10(e^{0.1}-1)\Big(\frac12+3\lambda(8\sigma^2 + D_{\rm app}^2 + \varepsilon^2 D_X^2)\Big) + \frac92 \lambda (\sigma^2 + D_{\rm app}^2 + \varepsilon^2 D_X^2) = 0.8\,, \\
	& \frac{1}{109 \sigma^2 + 28 D_{\rm app}^2 + 28 \varepsilon^2 D_X^2}< \lambda < \frac{1}{108 \sigma^2 + 27 D_{\rm app}^2 + 27 \varepsilon^2 D_X^2} \,, \\
	& 10\varepsilon^2 D_X^2 \Big(\frac12+3\lambda(8\sigma^2 + D_{\rm app}^2 + \varepsilon^2 D_X^2)\Big) + 12 \lambda D_X^2\varepsilon^2 \sigma^2 \le 8\varepsilon^2D_X^2 \\
	& \frac{1}{\lambda} \le 28 (4\sigma^2 + D_{\rm app}^2 + \varepsilon^2 D_X^2) \,.
\end{align*}

Combining Equations \eqref{eq:martingale_linear}, \eqref{eq:expansion_linear} and \eqref{eq:quadratic_linear}, we obtain
\begin{align*}
	& \mathbb{P}\Bigg(\bigcup\limits_{n=1}^{\infty}\Bigg(0.2\sum\limits_{t=\tau(\varepsilon,\delta)+1}^{\tau(\varepsilon,\delta)+n} (L(\hat{\theta}_t)-L(\theta^*)) > \frac{1}{2} \sum\limits_{t=\tau(\varepsilon,\delta)+1}^{\tau(\varepsilon,\delta)+n} X_t^TP_{t+1}X_t (y_t-\hat{\theta}_t^TX_t)^2 + \frac{\varepsilon^2}{\lambda_{\rm min}(P_{\tau(\varepsilon,\delta)+1})} \\
	& \qquad + 28(4\sigma^2 + D_{\rm approx}^2 + \varepsilon^2 D_X^2) \ln\delta^{-1} + 8\varepsilon^2 D_X^2 \ln\delta^{-1}\Bigg) \cap A_{\tau(\varepsilon,\delta)}^{\varepsilon}\Bigg) \le 3\delta \,.
\end{align*}
Finally, we apply Lemma \ref{lemma:sumtrace_linear} with $P_{\tau(\varepsilon,\delta)+1}\preccurlyeq P_1$ and we use Assumption \ref{ass:localized}: it holds simultaneously
\begin{align*}
    \sum\limits_{t=\tau(\varepsilon,\delta)+1}^{\tau(\varepsilon,\delta)+n} L(\hat{\theta}_t) - L(\theta^*) \le 5 & \bigg(\frac32 \left(8\sigma^2+ D_{\rm app}^2 + \varepsilon^2 D_X^2\right) d \ln\left(1 +n \frac{\lambda_{\rm max}(P_{1})D_X^2}{d}\right) + \lambda_{\rm max}\left(P_{\tau(\varepsilon,\delta)+1}^{-1}\right)\varepsilon^2 \\
    & \qquad + 28(4\sigma^2 + D_{\rm approx}^2 + \varepsilon^2 D_X^2) \ln\delta^{-1} + 8\varepsilon^2 D_X^2 \ln\delta^{-1} \\
    & \qquad + 6 \lambda_{\rm max}(P_1)D_X^2\sigma^2 \ln\delta^{-1}\bigg),\qquad n\ge 1 \,,
\end{align*}
with probability at least $1-5\delta$. To conclude, we write
\begin{align*}
    28(4\sigma^2 + D_{\rm approx}^2 + \varepsilon^2 D_X^2) +
   8\varepsilon^2 D_X^2 + 6 \lambda_{\rm max}(P_1)D_X^2\sigma^2\le  28\left(\sigma^2(4+\frac{\lambda_{\rm max}(P_1)D_X^2}{4}) + D_{\rm app}^2 + 2\varepsilon^2D_X^2\right) \,.
\end{align*}
\end{proof}

\section{Proofs of Section \ref{section:logistic}}\label{app:logistic}
\subsection{Proof of Theorem \ref{th:result_logistic}}\label{app:logistic_overview}
\begin{proof}\textbf{of Theorem \ref{th:result_logistic}.}
We check Assumption \ref{ass:bounded} with $\kappa_{\varepsilon}=e^{D_X(\|\theta^*\|+\varepsilon)},h_{\varepsilon} = \frac14$ and $\rho_{\varepsilon}=e^{-\varepsilon D_X}>0.95$.
We can thus apply Theorem \ref{th:localized_bounded} with
\begin{align*}
	\begin{gathered}
    \lambda_{\rm max}(P_{\tau(\varepsilon,\delta)+1}^{-1})\le \lambda_{\rm max}(P_1^{-1})+\frac14\sum\limits_{t=1}^{\tau(\varepsilon,\delta)}\|X_t\|^2 \,, \\
	\frac{5\kappa_{\varepsilon}}{2}< 3e^{D_X\|\theta^*\|},\qquad
	30\Big(2\kappa_{\varepsilon} + \frac{\varepsilon^2D_X^2}{4}\Big)<64e^{D_X\|\theta^*\|},\qquad
	5\varepsilon^2D_X^2 \le 1/75\,.
	\end{gathered}
\end{align*}

We then control the first terms. To that end, we use a rough bound at any time $t\ge 1$:
\begin{align*}
    L(\hat{\theta}_t)-L(\theta^*) & \le \mathbb{E}\left[\frac{yX}{1+e^{y\hat{\theta}_t^TX}}\mid \hat{\theta}_t\right]^T(\hat{\theta}_t-\theta^*) \\
    & \le D_X\|\hat{\theta}_t-\theta^*\|\\
    & \le D_X(\|\hat{\theta}_1-\theta^*\|+(t-1)\lambda_{\rm max}(P_1)D_X)\,,
\end{align*}
because for any $s\ge 1$, we have $P_s\preccurlyeq P_1$ and therefore $\|\hat{\theta}_{s+1}-\hat{\theta}_s\| \le \lambda_{\rm max}(P_1)D_X$.
Summing from $1$ to $\tau(\varepsilon,\delta)\le \tau(\frac{1}{20D_X},\delta)$ yields the result.
\end{proof}

\subsection{Concentration of $P_t$}\label{app:logistic_concentration}
We prove a concentration result based on~\cite{Tropp2012}, which will be used on the inverse of $P_t$.

\begin{lemma}\label{lemma:concentration_tropp}
If Assumption \ref{ass:iid} is satisfied, then for any $0\le\beta<1$ and $t\ge 4^{1/(1-\beta)}$, it holds
\begin{align*}
    \mathbb{P}\left(\lambda_{\rm min}\left(\sum\limits_{s=1}^{t-1} \frac{X_sX_s^T}{s^{\beta}}\right) < \frac{\Lambda_{\rm min}t^{1-\beta}}{4(1-\beta)}\right) & \le d \exp\left(- t^{1-\beta}\frac{\Lambda_{\rm min}^2}{10D_X^4}\right) \,.
\end{align*}
\end{lemma}

\begin{proof}
We wish to center the matrices $X_sX_s^T$ by subtracting their (common) expected value. We use that if $A$ and $B$ are symmetric, $\lambda_{\rm min}(A-B)\le\lambda_{\rm min}(A)-\lambda_{\rm min}(B)$. Indeed, denoting by $v$ any eigenvector of $A$ associated with its smallest eigenvalue,
\begin{align*}
    \lambda_{\rm min}(A-B) & = \min\limits_{x} \frac{x^T(A-B)x}{\|x\|^2} \\
    & \le \frac{v^T(A-B)v}{\|v\|^2} \\
    & = \lambda_{\rm min}(A) - \frac{v^TBv}{\|v\|^2} \\
    & \le \lambda_{\rm min}(A) - \min\limits_{x} \frac{x^TBx}{\|x\|^2} \\
    & = \lambda_{\rm min}(A) - \lambda_{\rm min}(B)\,.
\end{align*}
We obtain:
\begin{align*}
    \lambda_{\rm min}\left(\sum\limits_{s=1}^{t-1} \frac{X_sX_s^T}{s^{\beta}} - \sum\limits_{s=1}^{t-1} \mathbb{E}\left[\frac{X_sX_s^T}{s^{\beta}}\right]\right) & \le \lambda_{\rm min}\left(\sum\limits_{s=1}^{t-1} \frac{X_sX_s^T}{s^{\beta}}\right) - \lambda_{\rm min}\left(\sum\limits_{s=1}^{t-1} \mathbb{E}\left[\frac{X_sX_s^T}{s^{\beta}}\right]\right) \\
    & = \lambda_{\rm min}\left(\sum\limits_{s=1}^{t-1} \frac{X_sX_s^T}{s^{\beta}}\right) - \Lambda_{\rm min} \sum\limits_{s=1}^{t-1} \frac{1}{s^{\beta}}\\
    & \le \lambda_{\rm min}\left(\sum\limits_{s=1}^{t-1} \frac{X_sX_s^T}{s^{\beta}}\right) - \Lambda_{\rm min} \frac{t^{1-\beta}-1}{1-\beta} \,.
\end{align*}
Therefore, we obtain
\begin{align*}
    & \mathbb{P}\left(\lambda_{\rm min}\left(\sum\limits_{s=1}^{t-1} \frac{X_sX_s^T}{s^{\beta}}\right) < \frac{\Lambda_{\rm min}(t^{1-\beta}-2)}{2(1-\beta)}\right) \\
    & \qquad \le \mathbb{P}\left(\lambda_{\rm min}\left(\sum\limits_{s=1}^{t-1} \left(\frac{X_sX_s^T}{s^{\beta}} - \mathbb{E}\left[\frac{X_sX_s^T}{s^{\beta}}\right]\right)\right) < \frac{\Lambda_{\rm min}(t^{1-\beta}-2)}{2(1-\beta)} - \Lambda_{\rm min} \frac{t^{1-\beta}-1}{1-\beta}\right) \\
    & \qquad = \mathbb{P}\left(\lambda_{\rm max}\left(\sum\limits_{s=1}^{t-1} \left(\mathbb{E}\left[\frac{X_sX_s^T}{s^{\beta}}\right] - \frac{X_sX_s^T}{s^{\beta}}\right)\right) > \frac{\Lambda_{\rm min}t^{1-\beta}}{2(1-\beta)}\right)\,.
\end{align*}

We check the assumptions of Theorem 1.4 of~\cite{Tropp2012}:
\begin{itemize}
    \item
    Obviously $\mathbb{E}\left[\frac{X_sX_s^T}{s^{\beta}}\right] - \frac{X_sX_s^T}{s^{\beta}}$ is centered,
    \item
    $\lambda_{\rm max}\left(\mathbb{E}\left[\frac{X_sX_s^T}{s^{\beta}}\right] - \frac{X_sX_s^T}{s^{\beta}}\right)\le \lambda_{\rm max}\left(\mathbb{E}\left[\frac{X_sX_s^T}{s^{\beta}}\right]\right) \le D_X^2$ almost surely.
\end{itemize}

As $0\preccurlyeq \mathbb{E}\left[\left(\mathbb{E}\left[\frac{X_sX_s^T}{s^{\beta}}\right] - \frac{X_sX_s^T}{s^{\beta}}\right)^2\right]\preccurlyeq \mathbb{E}\left[\left(\frac{X_sX_s^T}{s^{\beta}}\right)^2\right]\preccurlyeq \frac{D_X^4}{s^{2\beta}} I \preccurlyeq \frac{D_X^4}{s^{\beta}} I $, we get
\begin{equation*}
    0\preccurlyeq \sum\limits_{s=1}^{t-1} \mathbb{E}\left[\left(\mathbb{E}\left[\frac{X_sX_s^T}{s^{\beta}}\right] - \frac{X_sX_s^T}{s^{\beta}}\right)^2\right] \preccurlyeq \left(\sum\limits_{s=1}^{t-1} \frac{D_X^4}{s^{\beta}}\right) I \preccurlyeq \left(D_X^4 \frac{t^{1-\beta}}{1-\beta}\right) I\,.
\end{equation*}

Therefore we can apply Theorem 1.4 of~\cite{Tropp2012}:
\begin{align*}
    & \mathbb{P}\left(\lambda_{\rm max}\left(\sum\limits_{s=1}^{t-1} \left(\mathbb{E}\left[\frac{X_sX_s^T}{s^{\beta}}\right] - \frac{X_sX_s^T}{s^{\beta}}\right)\right) > \frac{\Lambda_{\rm min}t^{1-\beta}}{2(1-\beta)} \right) \\
    & \qquad \le d \exp\left(-\frac{\Lambda_{\rm min}^2t^{2(1-\beta)}/(8(1-\beta)^2)}{D_X^4t^{1-\beta}/(1-\beta)+D_X^2\Lambda_{\rm min}t^{1-\beta}/(6(1-\beta))}\right) \\
    & \qquad = d \exp\left(- t^{1-\beta} \frac{\Lambda_{\rm min}^2}{8D_X^4}\frac{1/(1-\beta)^2}{1/(1-\beta)+\Lambda_{\rm min}/(6D_X^2(1-\beta))}\right) \\
    & \qquad = d \exp\left(- t^{1-\beta} \frac{\Lambda_{\rm min}^2}{8D_X^4}\left(1-\beta+\frac{\Lambda_{\rm min}(1-\beta)}{6D_X^2}\right)^{-1}\right)\,.
\end{align*}
Using $\Lambda_{\rm min}/D_X^2\le 1$ and $\beta\ge 0$, we obtain $8(1-\beta+\frac{\Lambda_{\rm min}(1-\beta)}{6D_X^2}) \le 8(1+1/6)=28/3\le 10$, therefore
\begin{align*}
    \mathbb{P}\left(\lambda_{\rm min}\left(\sum\limits_{s=1}^{t-1} \frac{X_sX_s^T}{s^{\beta}}\right) < \frac{\Lambda_{\rm min}(t^{1-\beta}-2)}{2(1-\beta)}\right) & \le d \exp\left(- t^{1-\beta}\frac{\Lambda_{\rm min}^2}{10D_X^4}\right) \,.
\end{align*}
The result follows from $\frac12 t^{1-\beta}-2>0$ for $t\ge 4^{1/(1-\beta)}$.
\end{proof}

We can now do a union bound to obtain Proposition \ref{prop:boundP}.
\begin{proof}\textbf{of Proposition \ref{prop:boundP}.}
We reduce our problem to the deviations of a sum of centered independent random matrices:
\begin{align*}
    \lambda_{\rm max}(P_t) & = \lambda_{\rm min}\left(P_1^{-1}+\sum\limits_{s=1}^{t-1} X_sX_s^T \alpha_s\right)^{-1} \\
    & \le \lambda_{\rm min}\left(P_1^{-1}+\sum\limits_{s=1}^{t-1} \frac{X_sX_s^T}{s^{\beta}}\right)^{-1}\,,
\end{align*}
because $\alpha_s\ge 1/s^{\beta}$. Therefore, for $t\ge 8\ge 4^{1/(1-\beta)}$,
\begin{align*}
    \mathbb{P}\left(\lambda_{\rm max}(P_t) > \frac{4}{\Lambda_{\rm min}t^{1-\beta}} \right) & \le \mathbb{P}\left(\lambda_{\rm min}\left(P_1^{-1}+\sum\limits_{s=1}^{t-1} \frac{X_sX_s^T}{s^{\beta}}\right)^{-1} > \frac{4}{\Lambda_{\rm min}t^{1-\beta}} \right)\\
    & = \mathbb{P}\left(\lambda_{\rm min}\left(P_1^{-1}+\sum\limits_{s=1}^{t-1} \frac{X_sX_s^T}{s^{\beta}}\right) < \frac{\Lambda_{\rm min}t^{1-\beta}}{4}\right) \\
    & \le \mathbb{P}\left(\lambda_{\rm min}\left(\sum\limits_{s=1}^{t-1} \frac{X_sX_s^T}{s^{\beta}}\right) < \frac{\Lambda_{\rm min}t^{1-\beta}}{4}\right)\\
    & \le d \exp\left(- t^{1-\beta}\frac{\Lambda_{\rm min}^2}{10D_X^4}\right) \,,
\end{align*}
where we applied Lemma \ref{lemma:concentration_tropp} to obtain the last line.
We take a union bound to obtain, for any $k\ge 7$,
\begin{align*}
    \mathbb{P}\left(\exists t>k, \lambda_{\rm max}(P_t) > \frac{4}{\Lambda_{\rm min}t^{1-\beta}} \right) & \le \sum\limits_{t>k} d \exp\left(- t^{1-\beta}\frac{\Lambda_{\rm min}^2}{10D_X^4}\right) \\
    & \le d \sum\limits_{t>k} \exp\left(- \lfloor t^{1-\beta}\rfloor\frac{\Lambda_{\rm min}^2}{10D_X^4}\right) \\
    & = d \sum\limits_{m\ge 1} \exp\left(- m\frac{\Lambda_{\rm min}^2}{10D_X^4}\right) \sum\limits_{t>k} \mathds{1}_{\lfloor t^{1-\beta}\rfloor = m}
\end{align*}
We bound $\sum\limits_{t>k} \mathds{1}_{\lfloor t\rfloor = m}$: for any $m$
\begin{align*}
    \lfloor t^{1-\beta} \rfloor = m & \implies m^{1/(1-\beta)}\le t < (m+1)^{1/(1-\beta)}\,,
\end{align*}
then using $e^x\le 1+2x$ for any $0\le x\le 1$, we have
\begin{align*}
    (m+1)^{1/(1-\beta)} & = m^{1/(1-\beta)}(1+1/m)^{1/(1-\beta)}\\
    & = m^{1/(1-\beta)}\exp(\ln(1+1/m)/(1-\beta)) \\
    & \le m^{1/(1-\beta)}\exp(1/(m(1-\beta))) \\
    & \le m^{1/(1-\beta)}(1+2/(m(1-\beta)))\,,
\end{align*}
as long as $m\ge 2 \ge 1/(1-\beta)$. Therefore
\begin{align*}
    (m+1)^{1/(1-\beta)}-m^{1/(1-\beta)}+1 \le 2m^{1/(1-\beta)-1}/(1-\beta)+1 \le 4m + 1\le 4(m+1)\,,
\end{align*}
and that is true for $m=1$ too. Hence
\begin{align*}
    \mathbb{P}\left(\exists t>k, \lambda_{\rm max}(P_t) > \frac{4}{\Lambda_{\rm min}t^{1-\beta}}\right) & \le 4d \sum\limits_{m\ge \lfloor k^{1-\beta}\rfloor} (m+1)\exp\left(-m\frac{\Lambda_{\rm min}^2}{10D_X^4}\right) \\
    & = 4d \frac{\exp\left(-\frac{\Lambda_{\rm min}^2}{10D_X^4}\right)^{\lfloor k^{1-\beta}\rfloor}}{1-\exp\left(-\frac{\Lambda_{\rm min}^2}{10D_X^4}\right)}(\lfloor k^{1-\beta}\rfloor+1+\frac{\exp\left(-\frac{\Lambda_{\rm min}^2}{10D_X^4}\right)}{1-\exp\left(-\frac{\Lambda_{\rm min}^2}{10D_X^4}\right)})\\
    & \le 4d \frac{\exp\left(\frac{\Lambda_{\rm min}^2}{10D_X^4}\right)}{1-\exp\left(-\frac{\Lambda_{\rm min}^2}{10D_X^4}\right)}(k^{1-\beta}+\frac{1}{1-\exp\left(-\frac{\Lambda_{\rm min}^2}{10D_X^4}\right)}) \exp\left(-\frac{\Lambda_{\rm min}^2}{10D_X^4}\right)^{k^{1-\beta}}\,,
\end{align*}
where the second line is obtained deriving both sides of $\sum\limits_{m\ge \lfloor k^{1-\beta}\rfloor} r^{m+1} = \frac{r^{\lfloor k^{1-\beta}\rfloor+1}}{1-r}$ with respect to r. Also, as $1-e^{-x}\ge xe^{-x}$ for any $x\in\R$, we get
\begin{align*}
    & \mathbb{P}\left(\exists t>k, \lambda_{\rm max}(P_t) > \frac{4}{\Lambda_{\rm min}t^{1-\beta}}\right) \\
    &\qquad \le 4d\frac{10D_X^4}{\Lambda_{\rm min}^2} \exp\left(2\frac{\Lambda_{\rm min}^2}{10D_X^4}\right)(k^{1-\beta}+\frac{10D_X^4}{\Lambda_{\rm min}^2}\exp\left(\frac{\Lambda_{\rm min}^2}{10D_X^4}\right)) \exp\left(-\frac{\Lambda_{\rm min}^2}{10D_X^4}\right)^{k^{1-\beta}} \,.
\end{align*}

Also, as $xe^{-x}\le e^{-1}$ for any $x\ge0$, we get for any $k\ge 7$:
\begin{align*}
    \left(k^{1-\beta}+\frac{10D_X^4}{\Lambda_{\rm min}^2}\exp\left(\frac{\Lambda_{\rm min}^2}{10D_X^4}\right)\right) \exp\left(-k^{1-\beta}\frac{\Lambda_{\rm min}^2}{20D_X^4}\right) & \le \frac{20D_X^4e^{-1}}{\Lambda_{\rm min}^2} \exp\left(\frac{10D_X^4}{\Lambda_{\rm min}^2}\exp\left(\frac{\Lambda_{\rm min}^2}{10D_X^4}\right)\frac{\Lambda_{\rm min}^2}{20D_X^4}\right)\\
    & = \frac{20D_X^4e^{-1}}{\Lambda_{\rm min}^2} \exp\left(\frac12\exp\left(\frac{\Lambda_{\rm min}^2}{10D_X^4}\right)\right) \,.
\end{align*}

Combining the last two inequalities, we obtain
\begin{align*}
    \mathbb{P}\left(\exists t>k, \lambda_{\rm max}(P_t) > \frac{4}{\Lambda_{\rm min}t^{1-\beta}}\right) & \le d\frac{800D_X^8e^{-1}}{\Lambda_{\rm min}^4} \exp\left(2\frac{\Lambda_{\rm min}^2}{10D_X^4} + \frac12\exp\left(\frac{\Lambda_{\rm min}^2}{10D_X^4}\right)\right) \exp\left(-k^{1-\beta}\frac{\Lambda_{\rm min}^2}{20D_X^4}\right) \\
    & \le d\frac{625D_X^8}{\Lambda_{\rm min}^4} \exp\left(-k^{1-\beta}\frac{\Lambda_{\rm min}^2}{20D_X^4}\right) \,,
\end{align*}
and the result follows. The last line comes from $\Lambda_{\rm min}\le D_X^2$ and consequently
\begin{align*}
    800e^{-1}\exp\left(2\frac{\Lambda_{\rm min}^2}{10D_X^4} + \frac12\exp\left(\frac{\Lambda_{\rm min}^2}{10D_X^4}\right)\right)\le 800e^{-1+0.2+0.5e^{0.1}} \approx 624.7\le 625\,.
\end{align*}
The condition $k\ge7$ is not necessary because
\begin{equation*}
    \left(\frac{20D_X^4}{\Lambda_{\rm min}^2}\ln\left(\frac{625dD_X^8}{\Lambda_{\rm min}^4\delta}\right)\right)^{1/(1-\beta)} \ge 20\ln(625\delta^{-1})\,,
\end{equation*}
and either $\delta\ge1$ and the result is trivial, either $\delta<1$ and $20\ln(625\delta^{-1})\ge 128$.
\end{proof}

\subsection{Convergence of the truncated algorithm}\label{app:logistic_convergence}
In order to prove Theorem \ref{th:convergence}, we state and prove an intermediate lemma.
\begin{lemma}
\label{lemma:equivalence_Lnorm}
Let $\theta\in\R^d$.
\begin{enumerate}
    \item
    \label{point:LtogradL}
    For any $\eta>0$, we have
    \begin{equation*}
        L(\theta) - L(\theta^*) > \eta \implies
        \left\|\frac{\partial L}{\partial \theta}\Bigr|_{\substack{\theta}}\right\| \ge D_\eta \,
    \end{equation*}
    for $D_\eta=\frac{\Lambda_{\rm min}\sqrt \eta}{\sqrt{2}D_X(1+e^{D_X(\|\theta^*\|+\sqrt{8\eta/D_X^2})})}$.
    \item
    \label{point:Ltotheta}
    For any $\varepsilon > 0$, we have
    \begin{equation*}
    \|\theta-\theta^*\| > \varepsilon \implies L(\theta) - L(\theta^*) > \frac{\Lambda_{\rm min}}{4(1+e^{D_X(\|\theta^*\|+\varepsilon)})}\varepsilon^2\,.
    \end{equation*}
\end{enumerate}
\end{lemma}

\begin{proof}
Both points derive from a second-order identity, turned in an upper-bound in the one case and in a lower-bound in the other. Using $\frac{\partial L}{\partial \theta}(\theta^*)=0$, there exists $0\le \lambda \le 1$ such that
\begin{equation*}
    L(\theta) = L(\theta^*) + \frac12(\theta - \theta^*)^T\mathbb{E}\left[\frac{1}{(1+e^{(\lambda\theta+(1-\lambda)\theta^*)^TX})(1+e^{-(\lambda\theta+(1-\lambda)\theta^*)^TX})}XX^T\right](\theta - \theta^*)\,.
\end{equation*}

\begin{enumerate}
\item
We first have
\begin{equation*}
    L(\theta) - L(\theta^*) \le \frac{D_X^2}{8}\|\theta - \theta^*\|^2\,.
\end{equation*}
Assume $L(\theta) - L(\theta^*) > \eta$. Then $\|\theta - \theta^*\|\ge \sqrt{8\eta/D_X^2}$. Also, using the Taylor expansion of $\theta^*$ around some $\theta_0\in\R^d$, we get
\begin{equation*}
    L(\theta^*) \ge L(\theta_0) + \frac{\partial L}{\partial \theta}\Bigr|_{\substack{\theta_0}}^T(\theta^*-\theta_0) + \frac{1}{4(1+e^{D_X(\|\theta^*\|+\|\theta_0-\theta^*\|)})}(\theta_0 - \theta^*)^T\mathbb{E}\left[XX^T\right](\theta_0 - \theta^*)\,,
\end{equation*}
and that yields
\begin{align*}
    \frac{\partial L}{\partial \theta}\Bigr|_{\substack{\theta_0}}^T(\theta_0-\theta^*) \ge
    L(\theta_0) - L(\theta^*) + \frac{\Lambda_{\rm min}}{4(1+e^{D_X(\|\theta^*\|+\|\theta_0-\theta^*\|)})} \|\theta_0 - \theta^*\|^2\,.
\end{align*}
Therefore, as $L(\theta_0) - L(\theta^*)\ge 0$,
\begin{align*}
    \left\|\frac{\partial L}{\partial \theta}\Bigr|_{\substack{\theta_0}}\right\| \ge
    \frac{\Lambda_{\rm min}}{4(1+e^{D_X(\|\theta^*\|+\|\theta_0-\theta^*\|)})}\|\theta_0 - \theta_{\text{true}}\| \,.
\end{align*}
Finally, as $L$ is convex of minimum $\theta^*$,
\begin{align*}
    \left\|\frac{\partial L}{\partial \theta}\Bigr|_{\substack{\theta}}\right\| & \ge \min\limits_{\|\theta_0-\theta^*\| = \sqrt{8\eta/D_X^2}} \left\|\frac{\partial L}{\partial \theta}\Bigr|_{\substack{\theta_0}}\right\| \\
    & \ge \frac{\Lambda_{\rm min}}{4(1+e^{D_X(\|\theta^*\|+\sqrt{8\eta/D_X^2})})}\sqrt{8\eta/D_X^2}  \\
    & \ge \frac{\Lambda_{\rm min}}{\sqrt{2}D_X(1+e^{D_X(\|\theta^*\|+\sqrt{8\eta/D_X^2})})}\sqrt{\eta} \,.
\end{align*}

\item 
On the other hand we have
\begin{align*}
    L(\theta) & \ge L(\theta^*) + \frac{\Lambda_{\rm min}}{4(1+e^{D_X(\|\theta^*\|+\|\theta-\theta^*\|)})}\|\theta - \theta^*\|^2 \,.
\end{align*}
Thus, as $L$ is convex of minimum $\theta^*$, if $\|\theta - \theta^*\|> \varepsilon$ it holds
\begin{equation*}
    L(\theta) - L(\theta^*) > \min\limits_{\|\theta_0 - \theta^* \| = \varepsilon} L(\theta_0) - L(\theta^*) \ge \frac{\Lambda_{\rm min}}{4(1+e^{D_X(\|\theta^*\|+\varepsilon)})}\varepsilon^2\,.
\end{equation*}
\end{enumerate}
\end{proof}

\begin{proof}\textbf{of Theorem \ref{th:convergence}.}
We prove the convergence of $(L(\hat{\theta}_{t}))_t$ to $L(\theta^*)$ and then the convergence of $(\hat{\theta}_{t})_t$ to $\theta^*$ follows. The convergence of $(L(\hat{\theta}_{t}))_t$ comes from the first point of Lemma \ref{lemma:equivalence_Lnorm}. The link between the two convergences is stated in the second point.

To study the evolution of $L(\hat{\theta}_{t})$ we first apply a second-order Taylor expansion: for any $t\ge 1$ there exists $0\le\alpha_t\le 1$ such that
\begin{equation}
	\label{eq:taylorLincr}
    L(\hat{\theta}_{t+1}) = L(\hat{\theta}_{t}) + \frac{\partial L}{\partial \theta}\Bigr|_{\substack{\hat{\theta}_{t}}}^T(\hat{\theta}_{t+1}-\hat{\theta}_{t}) + \frac12 (\hat{\theta}_{t+1}-\hat{\theta}_{t})^T \frac{\partial^2 L}{\partial \theta^2}\Bigr|_{\substack{\hat{\theta}_{t}+\alpha_t(\hat{\theta}_{t+1}-\hat{\theta}_{t})}} (\hat{\theta}_{t+1}-\hat{\theta}_{t})\,.
\end{equation}

We have $\frac{\partial^2 L}{\partial \theta^2} \preccurlyeq \frac14\mathbb{E}[XX^T]$, therefore, using the update formula on $\hat{\theta}$, the second-order term is bounded with
\begin{align*}
    (\hat{\theta}_{t+1}-\hat{\theta}_{t})^T \frac{\partial^2 L}{\partial \theta^2}\Bigr|_{\substack{\hat{\theta}_{t}+\alpha_t(\hat{\theta}_{t+1}-\hat{\theta}_{t})}} (\hat{\theta}_{t+1}-\hat{\theta}_{t}) & \le \frac{1}{(1+e^{y_t\hat{\theta}_t^TX_t})^2}X_t^TP_{t+1}^T\frac{\mathbb{E}[XX^T]}{4}P_{t+1}X_t \\
    & \le \frac14 D_X^4\lambda_{\rm max}(P_{t+1})^2 \le \frac14 D_X^4\lambda_{\rm max}(P_t)^2\,.
\end{align*}
The first-order term is controlled using the definition of the algorithm:
\begin{align*}
    \hat{\theta}_{t+1}-\hat{\theta}_{t} & = \left(P_t  - \frac{P_tX_tX_t^TP_t}{1+X_t^TP_tX_t\alpha_t}\alpha_t\right) \frac{y_tX_t}{1+e^{y_t\hat{\theta}_t^TX_t}}\,,
\end{align*}
and as $\alpha_t\le 1$,
\begin{equation*}
    \left\| - \alpha_t \frac{P_tX_tX_t^TP_t}{1+X_t^TP_tX_t\alpha_t} \frac{y_tX_t}{1+e^{y_t\hat{\theta}_t^TX_t}} \right\| \le D_X^3 \lambda_{\rm max}(P_t)^2\,.
\end{equation*}
Also, $\left\|\frac{\partial L}{\partial \theta}\right\|\le D_X$. Substituting our findings in Equation \eqref{eq:taylorLincr}, we obtain
\begin{equation}
    \label{eq:bound_incr}
    L(\hat{\theta}_{t+1}) \le L(\hat{\theta}_{t}) + \frac{\partial L}{\partial \theta}\Bigr|_{\substack{\hat{\theta}_{t}}}^TP_t \frac{y_tX_t}{1+e^{y_t\hat{\theta}_t^TX_t}} + 2D_X^4 \lambda_{\rm max}(P_t)^2\,.
\end{equation}

We define
\begin{align*}
    M_t & =\frac{\partial L}{\partial \theta}\Bigr|_{\substack{\hat{\theta}_t}}^TP_t \frac{y_tX_t}{1+e^{y_t\hat{\theta}_t^TX_t}} - \mathbb{E}\left[\frac{\partial L}{\partial \theta}\Bigr|_{\substack{\hat{\theta}_t}}^TP_t\frac{y_tX_t}{1+e^{y_t\hat{\theta}_t^TX_t}}\mid X_1,y_1,...,X_{t-1},y_{t-1}\right]\\
    & = \frac{\partial L}{\partial \theta}\Bigr|_{\substack{\hat{\theta}_{t}}}^TP_t \frac{y_tX_t}{1+e^{y_t\hat{\theta}_t^TX_t}} + \frac{\partial L}{\partial \theta}\Bigr|_{\substack{\hat{\theta}_{t}}}^TP_t\frac{\partial L}{\partial \theta}\Bigr|_{\substack{\hat{\theta}_{t}}} \,.
\end{align*}
Hence we have
\begin{align*}
    \frac{\partial L}{\partial \theta}\Bigr|_{\substack{\hat{\theta}_{t}}}^TP_t\frac{y_tX_t}{1+e^{y_t\hat{\theta}_t^TX_t}} \le M_t - \lambda_{\rm min}(P_t) \left\|\frac{\partial L}{\partial \theta}\Bigr|_{\substack{\hat{\theta}_{t}}}\right\|^2 \le M_t - \frac{1}{tD_X^2} \left\|\frac{\partial L}{\partial \theta}\Bigr|_{\substack{\hat{\theta}_{t}}}\right\|^2\,,
\end{align*}
because $P_s\succcurlyeq \frac{I}{sD_X^2}$. Combining it with Equation \eqref{eq:bound_incr} and summing consecutive terms, we obtain, for any $k<t$,
\begin{equation}
	\label{eq:recursive_incrL}
    L(\hat{\theta}_t)-L(\hat{\theta}_k) \le \sum\limits_{s=k}^{t-1} \left(M_s- \frac{1}{sD_X^2} \left\|\frac{\partial L}{\partial \theta}\Bigr|_{\substack{\hat{\theta}_{s}}}\right\|^2 + 2D_X^4 \lambda_{\rm max}(P_s)^2\right) \,.
\end{equation}

We recall that there exists $C_\delta$ such that $\mathbb{P}(A_{C_\delta})\ge 1-\delta$ where
\begin{align*}
	A_{C_\delta}:=\bigcap\limits_{t=1}^{\infty} \Big(\lambda_{\rm max}(P_t)\le \frac{C_\delta}{t^{1-\beta}}\Big)\,.
\end{align*}
On the previous inequality, we see that the left-hand side is the sum of a martingale and a term which is negative for $s$ large enough, under the event $A_{C_{\delta}}$.

We are then interested in $\mathbb{P}((L(\hat{\theta}_t)-L(\theta^*) > \eta) \mid A_{C_{\delta}})$ for some $\eta>0$.
For $0\le k\le t$, we define $B_{k,t}$ be the event $(\forall k < s < t, L(\hat{\theta}_s)-L(\theta^*)>\eta/2)$. Then we use the law of total probability:
\begin{align}
    \mathbb{P}(L(\hat{\theta}_t)-L(\theta^*) > \eta \mid A_{C_{\delta}}) =\ &  \mathbb{P}\left((L(\hat{\theta}_t)-L(\theta^*) > \eta) \cap B_{0,t} \mid A_{C_{\delta}}\right) \nonumber \\
    & + \sum\limits_{k=1}^{t-1} \mathbb{P}\left((L(\hat{\theta}_t)-L(\theta^*) > \eta) \cap \big(L(\hat{\theta}_{k})-L(\theta^*) \le \frac{\eta}{2}\big) \cap B_{k,t} \mid A_{C_{\delta}}\right) \label{eq:total_proba}\\
    \le\ & \mathbb{P}\left((L(\hat{\theta}_t)-L(\theta^*) > \eta) \cap B_{0,t} \mid A_{C_{\delta}}\right) \nonumber \\
    & + \sum\limits_{k=1}^{t-1} \mathbb{P}\left(\big(L(\hat{\theta}_t)-L(\hat{\theta}_{k}) > \frac{\eta}{2}\big) \cap B_{k,t} \mid A_{C_{\delta}}\right) \nonumber \,.
\end{align}

Lemma \ref{lemma:equivalence_Lnorm} yields
\begin{align*}
	L(\hat{\theta}_s) - L(\theta^*) > \frac{\eta}{2} \implies \left\|\frac{\partial L}{\partial \theta}\Bigr|_{\substack{\hat{\theta}_s}}\right\| \ge D_{\eta}\,.
\end{align*}

We combine the last equation, along with Equation \eqref{eq:recursive_incrL} and the definition of $A_{C_{\delta}}$ to get, for any $1\le k<t$,
\begin{align*}
    \mathbb{P}\left((L(\hat{\theta}_t)-L(\hat{\theta}_{k}) > \eta /2) \cap B_{k,t} \mid A_{C_{\delta}}\right)
    & \le \mathbb{P}\left(\Big(\sum\limits_{s=k}^{t-1} M_s > f(k,t)\Big) \cap B_{k,t} \mid A_{C_{\delta}}\right) \\
    & \le \mathbb{P}\left(\sum\limits_{s=k}^{t-1} M_s > f(k,t) \mid A_{C_{\delta}}\right) \,,
\end{align*}
where $f(k,t) = \frac{\eta}{2} + \frac{D_{\eta}^2}{D_X^2}\sum\limits_{s=k}^{t-1}\frac{1}{s} - 2D_X^4C_{\delta}^2\sum\limits_{s=k}^{t-1} \frac{1}{s^{2(1-\beta)}}$ for any $1 \le k < t$.

Similarly, we get
\begin{equation*}
    \mathbb{P}\left((L(\hat{\theta}_t)-L(\theta^*) > \eta) \cap B_{0,t} \mid A_C\right) \le \mathbb{P}\left(\sum\limits_{s=1}^{t-1} M_s > f_0(t) \mid A_C \right)\,,
\end{equation*}
with $f_0(t)=\eta - (L(\hat{\theta}_1)-L(\theta^*)) + \frac{D_{\eta}^2}{D_X^2}\sum\limits_{s=1}^{t-1}\frac{1}{s} - 2D_X^4C_{\delta}^2\sum\limits_{s=1}^{t-1} \frac{1}{s^{2(1-\beta)}}$ for any $t\ge 1$.

We have $\mathbb{E}[M_s\mid X_1,y_1,...,X_{s-1},y_{s-1}] = 0$, and almost surely $| M_s | \le 2D_X^2 \lambda_{\rm max}(P_s)$.
We can therefore apply Azuma-Hoeffding inequality: for $t,k$ such that $f(k,t)>0$,
\begin{equation*}
    \mathbb{P}\left(\sum\limits_{s=k}^{t-1} M_s > f(k,t) \mid A_{C_{\delta}}\right)
    \le \exp\left(-f(k,t)^2 \frac{(1-2\beta)\max\left(1/2,(k-1)^{1-2\beta}\right)}{8 D_X^4 C_{\delta}^2}\right)\,,
\end{equation*}
because $\sum\limits_{s=k}^{+\infty} \frac{1}{s^{2(1-\beta)}} \le \frac{1}{(1-2\beta)\max\left(1/2,(k-1)^{1-2\beta}\right)}$.
Similarly, for $t$ such that $f_0(t)>0$,
\begin{equation*}
    \mathbb{P}\left(\sum\limits_{s=1}^{t-1} M_s > f_0(t) \mid A_{C_{\delta}}\right) \le \exp\left(-f_0(t)^2 \frac{1-2\beta}{16 D_X^4 C_{\delta}^2}\right)\,.
\end{equation*}

We need to control $f(k,t),f_0(t)$. We see that for $t$ large enough, when $k$ is small compared to $t$, $f(k,t)$ is driven by $\frac{D_{\eta}^2}{D_X^2}\ln(t)$ and when $k \approx t$, $f(k,t)$ is driven by $\eta/2$. The following Lemma formally states these approximations as lower-bounds. We prove it right after the end of this proof.
\begin{lemma}
\label{lemma:bound_f}
For $t\ge\max\left(e^{\frac{16 D_X^6C_{\delta}^2}{D_{\eta}^2(1-2\beta)}}, \left(1+\left(\frac{8 D_X^4C_{\delta}^2}{\eta(1-2\beta)}\right)^{\frac{1}{1-2\beta}}\right)^2\right)$, it holds
\begin{align*}
    f(k,t) &\ge \frac{D_{\eta}^2}{4D_X^2}\ln(t), &1\le k< \sqrt{t}, \\
    f(k,t) &\ge \frac{\eta}{4}, &\sqrt{t}\le k < t\,.
\end{align*}
Similarly, for $t\ge e^{\frac{2D_X^2}{D_{\eta}^2}\left(L(\hat{\theta}_1)-L(\theta^*) + \frac{4 D_X^4C_{\delta}^2}{1-2\beta}\right)}$, we have
\begin{equation*}
	f_0(t) \ge \frac{D_{\eta}^2}{2D_X^2}\ln(t)\,.
\end{equation*}
\end{lemma}

Then, defining $C_1=\frac{D_{\eta}^4(1-2\beta)}{256 D_X^8C_{\delta}^2}$ and $C_2=\frac{\eta^2(1-2\beta)}{128 D_X^4C_{\delta}^2}$, we finally get for $t$ large enough:
\begin{align*}
    \mathbb{P}\left((L(\hat{\theta}_t)-L(\theta^*) > \eta) \cap B_{0,t} \mid A_{C_{\delta}}\right) \le \exp\left(-4C_1\ln(t)^2\right),& \\
    \mathbb{P}\left((L(\hat{\theta}_t)-L(\theta^*) > \eta) \cap (L(\hat{\theta}_{k})-L(\theta^*) \le \frac{\eta}{2}) \cap B_{k,t} \mid A_{C_{\delta}}\right) \le \exp\left(-C_1\ln(t)^2\right),&\qquad 1\le k< \sqrt{t} \\
    \mathbb{P}\left((L(\hat{\theta}_t)-L(\theta^*) > \eta) \cap (L(\hat{\theta}_{k})-L(\theta^*) \le \frac{\eta}{2}) \cap B_{k,t} \mid A_{C_{\delta}}\right) \le \exp\left(-C_2(k-1)^{1-2\beta}\right),&\qquad \sqrt{t}\le k < t
\end{align*}

Substituting in Equation \eqref{eq:total_proba} yields:
\begin{align*}
    \mathbb{P}(L(\hat{\theta}_t)-L(\theta^*) > \eta \mid A_C) & \le \exp\left(-4C_1\ln(t)^2\right) + \sum\limits_{k=1}^{\lceil\sqrt{t}\rceil-1} \exp\left(-C_1\ln(t)^2\right) + \sum\limits_{k=\lceil\sqrt{t}\rceil}^{t-1} \exp\left(-C_2(k-1)^{1-2\beta}\right) \\
    & \le (\sqrt{t}+1) \exp\left(-C_1\ln(t)^2\right) + t \exp\left(-C_2(\sqrt{t}-1)^{1-2\beta}\right)\,.
\end{align*}

Finally, Point \ref{point:Ltotheta} of Lemma \ref{lemma:equivalence_Lnorm} allows to obtain the result: defining $\eta=\frac{\Lambda_{\rm min}\varepsilon^2}{4(1+e^{D_X(\|\theta^*\|+\varepsilon)})}$, we obtain
\begin{align*}
    \mathbb{P}(\|\hat{\theta}_t-\theta^*\| > \varepsilon \mid A_{C_{\delta}}) & \le \mathbb{P}(L(\hat{\theta}_t)-L(\theta^*) > \eta \mid A_{C_{\delta}}) \\
    & \le (\sqrt{t}+1) \exp\left(-C_1\ln(t)^2\right) + t \exp\left(-C_2(\sqrt{t}-1)^{1-2\beta}\right)\,.
\end{align*}

In order to obtain the constants involved in the Theorem, we write
\begin{align*}
    & D_{\eta} = \frac{\Lambda_{\rm min}\sqrt{\frac{\Lambda_{\rm min}\varepsilon^2}{4(1+e^{D_X(\|\theta^*\|+\varepsilon)})}}}{2D_X(1+\exp\left(D_X(\|\theta^*\|+\sqrt{\frac{\Lambda_{\rm min}\varepsilon^2}{D_X^2(1+e^{D_X(\|\theta^*\|+\varepsilon)})}})\right))} \ge \left(\frac{\Lambda_{\rm min}}{1+e^{D_X(\|\theta^*\|+\varepsilon)}}\right)^{3/2} \frac{\varepsilon}{4D_X}\,, \\
    & C_1 \ge \frac{\Lambda_{\rm min}^6(1-2\beta)\varepsilon^4}{2^{16} D_X^{12}C_{\delta}^2(1+e^{D_X(\|\theta^*\|+\varepsilon)})^6} \,, \\
    & C_2 \ge \frac{\Lambda_{\rm min}^2(1-2\beta)\varepsilon^4}{2^{11} D_X^4C_{\delta}^2 (1+e^{D_X(\|\theta^*\|+\varepsilon)})^2} \,,
\end{align*}
and the conditions of Lemma \ref{lemma:bound_f} become
\begin{align*}
    & t \ge \exp\left(\frac{2^8 D_X^8 C_{\delta}^2 (1+e^{D_X(\|\theta^*\|+\varepsilon)})^{3}}{\Lambda_{\rm min}^{3}(1-2\beta)\varepsilon^2}\right) \,, \\
    & t \ge \left(1+\left(\frac{32 D_X^4C_{\delta}^2 (1+e^{D_X(\|\theta^*\|+\varepsilon)})}{(1-2\beta)\Lambda_{\rm min}\varepsilon^2}\right)^{\frac{1}{1-2\beta}}\right)^2 \,, \\
    & t \ge \exp\left(\frac{32 D_X^4(1+e^{D_X(\|\theta^*\|+\varepsilon)})^{3}}{\Lambda_{\rm min}^{3}\varepsilon^2}\left(L(\hat{\theta}_1)-L(\theta^*) + \frac{4 D_X^4C_{\delta}^2}{1-2\beta}\right)\right) \,.
\end{align*}

We would like to obtain a single condition on $t$, thus we write
\begin{align*}
    \left(1+\left(\frac{32 D_X^4C_{\delta}^2 (1+e^{D_X(\|\theta^*\|+\varepsilon)})}{(1-2\beta)\Lambda_{\rm min}\varepsilon^2}\right)^{\frac{1}{1-2\beta}}\right)^2 & = \exp\left(2\ln\left(1+\left(\frac{32 D_X^4C_{\delta}^2 (1+e^{D_X(\|\theta^*\|+\varepsilon)})}{(1-2\beta)\Lambda_{\rm min}\varepsilon^2}\right)^{\frac{1}{1-2\beta}}\right)\right) \\
    & \le \exp\left(\frac{2}{1-2\beta}\ln\left(1+\frac{32 D_X^4C_{\delta}^2 (1+e^{D_X(\|\theta^*\|+\varepsilon)})}{(1-2\beta)\Lambda_{\rm min}\varepsilon^2}\right)\right) \\
    & \le \exp\left(\frac{2}{1-2\beta}\sqrt{\frac{32 D_X^4C_{\delta}^2 (1+e^{D_X(\|\theta^*\|+\varepsilon)})}{(1-2\beta)\Lambda_{\rm min}\varepsilon^2}}\right) \\
    & \le \exp\left(\frac{2^8 D_X^8 C_{\delta}^2 (1+e^{D_X(\|\theta^*\|+\varepsilon)})^{3}}{\Lambda_{\rm min}^{3}(1-2\beta)^{3/2}\varepsilon^2}\right) \,,
\end{align*}
The third line is obtained with the inequality $\ln(1+x)\le \sqrt{x}$ for any $x>0$. Obviously, as $0<1-2\beta<1$, the first threshold on $t$ is bounded by:
\begin{equation*}
	\exp\left(\frac{2^8 D_X^8 C_{\delta}^2 (1+e^{D_X(\|\theta^*\|+\varepsilon)})^{3}}{\Lambda_{\rm min}^{3}(1-2\beta)\varepsilon^2}\right) \le \exp\left(\frac{2^8 D_X^8 C_{\delta}^2 (1+e^{D_X(\|\theta^*\|+\varepsilon)})^{3}}{\Lambda_{\rm min}^{3}(1-2\beta)^{3/2}\varepsilon^2}\right) \,.
\end{equation*}
To handle the third one, we use $D_X^2 C_{\delta}\ge \frac{4D_X^2}{\Lambda_{\rm min}}\ge 4$ and as $\hat{\theta}_1=0$ we obtain $L(\hat{\theta}_1)-L(\theta^*)\le \ln 2\le \frac{4D_X^4C_{\delta}^2}{1-2\beta}$, hence
\begin{equation*}
    \exp\left(\frac{32 D_X^4(1+e^{D_X(\|\theta^*\|+\varepsilon)})^{3}}{\Lambda_{\rm min}^{3}\varepsilon^2}\left(L(\hat{\theta}_1)-L(\theta^*) + \frac{4 D_X^4C_{\delta}^2}{1-2\beta}\right)\right) \le \exp\left(\frac{2^8 D_X^8 C_{\delta}^2 (1+e^{D_X(\|\theta^*\|+\varepsilon)})^{3}}{\Lambda_{\rm min}^{3}(1-2\beta)^{3/2}\varepsilon^2}\right) \,.
\end{equation*}
\end{proof}

\begin{proof}\textbf{of Lemma \ref{lemma:bound_f}.} 
We recall that for any $k\ge 1$,
\begin{equation*}
	\sum\limits_{s=k}^{t-1}\frac{1}{s}\ge \ln t - \ln k \,,\qquad  \sum\limits_{s=k}^{t-1}\frac{1}{s^{2(1-\beta)}} \le \frac{1}{1-2\beta}\frac{1}{\max(1/2,(k-1)^{1-2\beta})}\,.
\end{equation*}
Therefore:
\begin{align*}
	&f(k,t) \ge \frac{\eta}{2} + \frac{D_{\eta}^2}{D_X^2}(\ln t - \ln k) - \frac{2D_X^4C_{\delta}^2}{1-2\beta}\frac{1}{\max(1/2,(k-1)^{1-2\beta})} \,, \\
	&f_0(t) \ge \eta - (L(\hat{\theta}_1)-L(\theta^*) + \frac{D_{\eta}^2}{D_X^2}\ln t - \frac{4D_X^4C_{\delta}^2}{1-2\beta} \,.
\end{align*}
\begin{itemize}
    \item
    For any $1\le k<\sqrt{t}$, $\ln k \le \frac12\ln t$, and we have
    \begin{equation*}
        f(k,t) \ge \frac{D_{\eta}^2}{2D_X^2}\ln(t) -\frac{4D_X^4C_{\delta}^2}{1-2\beta} \,,
    \end{equation*}
    and taking $t\ge e^{\frac{16 D_X^6C_{\delta}^2}{D_{\eta}^2(1-2\beta)}}$ yields $f(k,t) \ge \frac{D_{\eta}^2}{4D_X^2}\ln(t)$.
    \item
    For $t\ge 2$ and any $k\ge\sqrt{t}$, we have
    \begin{align*}
        f(k,t) & \ge \frac{\eta}{2} - \frac{2 D_X^4C_{\delta}^2}{(1-2\beta)(k-1)^{1-2\beta}} \ge \frac{\eta}{2} - \frac{2 D_X^4C_{\delta}^2}{(1-2\beta)(\sqrt{t}-1)^{1-2\beta}}\,.
    \end{align*}
    Then if $t\ge \left(1+\left(\frac{8 D_X^4C_{\delta}^2}{\eta(1-2\beta)}\right)^{\frac{1}{1-2\beta}}\right)^2$, we get $f(k,t)\ge \frac{\eta}{4}$.
    
    \item
    Last point comes from $f_0(t) \ge \frac{D_{\eta}^2}{D_X^2}\ln t - (L(\hat{\theta}_1)-L(\theta^*) - \frac{4D_X^4C_{\delta}^2}{1-2\beta}$.
\end{itemize}
\end{proof}

\begin{proof}\textbf{of Corollary \ref{coro:convergence}.} 
We apply Theorem \ref{th:convergence}: for any $t\ge \exp\left(\frac{2^8 D_X^8 C_{\delta/2}^2 (1+e^{D_X(\|\theta^*\|+\varepsilon)})^{3}}{\Lambda_{\rm min}^{3}(1-2\beta)^{3/2}\varepsilon^2}\right)$,
\begin{align*}
    \mathbb{P}(\|\hat{\theta}_t-\theta^*\| > \varepsilon \mid A_{C_{\delta/2}}) \le\ & (\sqrt{t}+1) \exp\left(- C_1  \ln(t)^2\right) + t \exp\left(-C_2 (\sqrt{t}-1)^{1-2\beta}\right)\,,
\end{align*}
where
\begin{align*}
    C_1 = \frac{\Lambda_{\rm min}^6(1-2\beta)\varepsilon^4}{2^{16} D_X^{12}C_{\delta/2}^2(1+e^{D_X(\|\theta^*\|+\varepsilon)})^6},\qquad C_2 = \frac{\Lambda_{\rm min}^2(1-2\beta)\varepsilon^4}{2^{11} D_X^4C_{\delta/2}^2 (1+e^{D_X(\|\theta^*\|+\varepsilon)})^2} \,.
\end{align*}
We use a union bound: for any $\tau\ge \exp\left(\frac{2^8 D_X^8 C_{\delta/2}^2 (1+e^{D_X(\|\theta^*\|+\varepsilon)})^{3}}{\Lambda_{\rm min}^{3}(1-2\beta)^{3/2}\varepsilon^2}\right)$,
\begin{align*}
    \mathbb{P}\left(\bigcup\limits_{t=\tau+1}^{\infty} (\|\hat{\theta}_t-\theta^*\| > \varepsilon) \mid A_{C_{\delta/2}} \right) \le \sum\limits_{t>\tau} (\sqrt{t}+1) \exp\left(- C_1  \ln(t)^2\right) + \sum\limits_{t> \tau} t \exp\left(-C_2 (\sqrt{t}-1)^{1-2\beta}\right)\,.
\end{align*}
\begin{itemize}
\item
If $\tau \ge e^{\frac{3}{2C_1}}$, we have 
\begin{align*}
	\sum\limits_{t>\tau} (\sqrt{t}+1) \exp\left(- C_1  \ln(t)^2\right) \le \sum\limits_{t>\tau} (\sqrt{t}+1) \frac{1}{t^{5/2}} \le 2/\tau\,,
\end{align*}
\item
For $t\ge 4$, $1-1/\sqrt{t}\ge 1/2$, then for $t\ge\left(\frac{12}{C_2(1-2\beta)}\right)^{4/(1-2\beta)}$,
\begin{align*}
	t^3 \exp\left(-C_2(\sqrt{t}-1)^{1-2\beta}\right) & \le \exp\left(3\ln(t)-\frac{C_2}{2}t^{(1-2\beta)/2}\right) \\
	& \le \exp\left(\frac{12}{1-2\beta}\ln\left(\frac{12}{C_2(1-2\beta)}\right)-\frac{6}{1-2\beta}\left(\frac{12}{C_2(1-2\beta)}\right)\right) \\
	& \le 1 \,,
\end{align*}
because for any $x>0$, we have $\ln x\le x/2$.

Thus for $\tau\ge\left(\frac{12}{C_2(1-2\beta)}\right)^{4/(1-2\beta)}$
\begin{align*}
	\sum\limits_{t> \tau} t \exp\left(-C_2 (\sqrt{t}-1)^{1-2\beta}\right) & \le 1/\tau \,.
\end{align*}
\end{itemize}

Finally, for $\tau$ big enough, we obtain
\begin{align*}
    \mathbb{P}\left(\bigcup\limits_{t=\tau+1}^{\infty} (\|\hat{\theta}_t-\theta^*\| > \varepsilon) \mid A_{C_{\delta/2}} \right) \le 3 / \tau \le \delta / 2\,,
\end{align*}
if $\tau \ge 6\delta^{-1}$.
We now compare the constants involved. As long as $\varepsilon D_X\le 1$, we have
\begin{align*}
	\exp\left(\frac{2^8 D_X^8 C_{\delta/2}^2 (1+e^{D_X(\|\theta^*\|+\varepsilon)})^{3}}{\Lambda_{\rm min}^{3}(1-2\beta)^{3/2}\varepsilon^2}\right) \le \exp\left(\frac{3\cdot 2^{15} D_X^{12}C_{\delta/2}^2(1+e^{D_X(\|\theta^*\|+\varepsilon)})^6}{\Lambda_{\rm min}^6(1-2\beta)^{3/2}\varepsilon^4}\right) \,.
\end{align*}
Furthermore, as $1-2\beta\le 1$, we have
\begin{align*}
	\exp\left(\frac{3}{2C_1}\right) & = \exp\left(\frac{3\cdot 2^{15} D_X^{12}C_{\delta/2}^2(1+e^{D_X(\|\theta^*\|+\varepsilon)})^6}{\Lambda_{\rm min}^6(1-2\beta)\varepsilon^4}\right) \le \exp\left(\frac{3\cdot 2^{15} D_X^{12}C_{\delta/2}^2(1+e^{D_X(\|\theta^*\|+\varepsilon)})^6}{\Lambda_{\rm min}^6(1-2\beta)^{3/2}\varepsilon^4}\right) \,.
\end{align*}
Finally,
\begin{align*}
	\left(\frac{12}{C_2(1-2\beta)}\right)^{4/(1-2\beta)} & = \exp\left(\frac{4}{1-2\beta}\ln \frac{12}{C_2(1-2\beta)} \right) \\
	& = \exp\left(\frac{4}{1-2\beta}\ln \frac{12\cdot 2^{11} D_X^4C_{\delta/2}^2 (1+e^{D_X(\|\theta^*\|+\varepsilon)})^2}{\Lambda_{\rm min}^2(1-2\beta)^2\varepsilon^4} \right) \\
	& = \exp\left(\frac{8}{1-2\beta}\ln \frac{12\cdot 2^{11} D_X^4C_{\delta/2}^2 (1+e^{D_X(\|\theta^*\|+\varepsilon)})^2}{\Lambda_{\rm min}^2(1-2\beta)\varepsilon^4} \right) \\
	& \le \exp\left(\frac{8}{1-2\beta}\sqrt{ \frac{3\cdot 2^{13} D_X^4C_{\delta/2}^2 (1+e^{D_X(\|\theta^*\|+\varepsilon)})^2}{\Lambda_{\rm min}^2(1-2\beta)\varepsilon^4}}\right) \\
	& = \exp\left(\frac{\sqrt{6}2^9 D_X^2C_{\delta/2} (1+e^{D_X(\|\theta^*\|+\varepsilon)})}{\Lambda_{\rm min}(1-2\beta)^{3/2}\varepsilon^2}\right) \\
	& \le \exp\left(\frac{3\cdot 2^{15} D_X^{12}C_{\delta/2}^2(1+e^{D_X(\|\theta^*\|+\varepsilon)})^6}{\Lambda_{\rm min}^6(1-2\beta)^{3/2}\varepsilon^4}\right) \,.
\end{align*}
\end{proof}

\section{Proofs of Section \ref{section:quadratic}}\label{app:quadratic}
\subsection{Proof of Theorem \ref{th:result_linear}}\label{app:linear_overview}
We first prove a result controlling the first estimates of the algorithm.
\begin{lemma}
\label{lemma:quadratic_first}
Provided that assumptions \ref{ass:iid}, \ref{ass:existence} and \ref{ass:subgaussian} are satisfied, starting from any $\hat{\theta}_1\in\R^d$ and $P_1\succ 0$, for any $\delta>0$, it holds simultaneously
\begin{equation*}
	\| \hat{\theta}_t - \theta^* \| \le \|\hat{\theta}_1 - \theta^*\| + \lambda_{\rm max}(P_1) D_X \left((3\sigma+D_{\rm approx}) (t-1) + 3 \sigma \ln\delta^{-1} \right),\qquad t\ge 1,
\end{equation*}
with probability at least $1-\delta$.
\end{lemma}

\begin{proof} 
From Proposition \ref{prop:kf_ridge}, we obtain, for any $t\ge1$, $\hat{\theta}_t - \hat{\theta}_1 = P_t\sum_{s=1}^{t-1} (y_s- \hat{\theta}_1^TX_s)X_s$. Consequently,
\begin{align*}
    \hat{\theta}_t - \theta^* & = P_t\sum_{s=1}^{t-1} (y_s- \hat{\theta}_1^TX_s)X_s - P_t \left(P_1^{-1} + \sum\limits_{s=1}^{t-1} X_sX_s^T\right) (\theta^*-\hat{\theta}_1) \\
    & = P_t \sum\limits_{s=1}^{t-1} (y_s - \theta^{*T}X_s)X_s  + P_t P_1^{-1}(\hat{\theta}_1-\theta^*) \,,
\end{align*}
and using $P_tP_1^{-1} \preccurlyeq I$, we obtain
\begin{align}
    \nonumber \| \hat{\theta}_t - \theta^* \| &  \le \|\hat{\theta}_1-\theta^*\| + \lambda_{\rm max}(P_t) D_X \sum\limits_{s=1}^{t-1} |y_s - \theta^{*T}X_s| \\
    \label{eq:linear_ridge} & \le \|\hat{\theta}_1-\theta^*\| + \lambda_{\rm max}(P_1) D_X  \sum\limits_{s=1}^{t-1} \left(|y_s-\mathbb{E}[y_s\mid X_s]| + D_{\rm app}\right) \,.
\end{align}

We apply Lemma 1.4 of \cite{rigollet2015high} in the second line of the following: for any $\mu$ such that $0<\mu<\frac{1}{2\sqrt{2}\sigma}$,
\begin{align*}
    \mathbb{E}\left[\exp(\mu |y_t-\mathbb{E}[y_t\mid X_t]|)\right] & = 1 + \sum\limits_{i\ge 1} \frac{\mu^i \mathbb{E}[|y_t-\mathbb{E}[y_t\mid X_t]|^i]}{i!} \\
    & \le 1 + \sum\limits_{k\ge 1} \frac{\mu^i (2\sigma^2)^{i/2} i\Gamma(i/2)}{i!} \\
    & \le 1 + \sum\limits_{i\ge 1} \left(\sqrt{2}\mu \sigma\right)^i,\qquad \text{because } \Gamma(i/2)\le \Gamma(i)=(i-1)! \\
    & \le 1 + 2\sqrt{2}\mu\sigma, \qquad \text{because } 0<\sqrt{2}\mu\sigma \le \frac{1}{2} \\
    & \le \exp\left(2\sqrt{2}\mu\sigma\right)\,.
\end{align*}
Thus we can apply Lemma \ref{lemma:simultaneous_supermartingale} to the super-martingale $\left(\exp\left(\frac{1}{2\sqrt{2} \sigma}\sum\limits_{s=1}^{t} (|y_s-\mathbb{E}[y_s\mid X_s]| - 2\sqrt{2}\sigma)\right)\right)_t$ in order to obtain, for any $\delta>0$,
\begin{align*}
    & \sum\limits_{s=1}^{t-1}|y_t-\mathbb{E}[y_t\mid X_t]| \le 2\sqrt{2}(t-1)\sigma + 2\sqrt{2}\sigma\ln\delta^{-1},\qquad t\ge 1,
\end{align*}
with probability at least $1-\delta$. The result follows from Equation \eqref{eq:linear_ridge} and $2\sqrt{2}\le 3$.
\end{proof}

\begin{proof}\textbf{of Theorem \ref{th:result_linear}.} 
We first apply Theorem \ref{th:localized_linear}: with probability at least $1-5\delta$, it holds simultaneously
\begin{align*}
    \sum\limits_{t=\tau(\varepsilon,\delta)+1}^{n} L(\hat{\theta}_t) - L(\theta^*) \le\ & \frac{15}{2} d \left(8\sigma^2+ D_{\rm app}^2 + \varepsilon^2 D_X^2\right) \ln\left(1 + (n-\tau(\varepsilon,\delta)) \frac{\lambda_{\rm max}(P_{1})D_X^2}{d}\right) \\
    & + 5 \lambda_{\rm max}\left(P_{\tau(\varepsilon,\delta)+1}^{-1}\right)\varepsilon^2 \\
    & + 115\left(\sigma^2(4+\frac{\lambda_{\rm max}(P_1)D_X^2}{4}) + D_{\rm app}^2 + 2\varepsilon^2D_X^2\right) \ln\delta^{-1},\qquad n\ge \tau(\varepsilon,\delta) \,.
\end{align*}
Moreover, $\lambda_{\rm max}\left(P_{\tau(\varepsilon,\delta)+1}^{-1}\right) \le \lambda_{\rm max}(P_1^{-1}) + \tau(\varepsilon,\delta)D_X^2$.

Then we derive a bound on the first $\tau(\varepsilon,\delta)$ terms.
For any $t\ge 1$, we have $L(\hat{\theta}_t) - L(\theta^*) \le D_X^2 \| \hat{\theta}_t - \theta^* \|^2$, thus, using $(a+b)^2\le 2(a^2+b^2)$ and applying Lemma \ref{lemma:quadratic_first} we obtain the simultaneous property
\begin{align*}
	L(\hat{\theta}_t) - L(\theta^*) \le\ & 2D_X^2(\|\hat{\theta}_1 - \theta^*\| + 3 \lambda_{\rm max}(P_1) D_X\sigma \ln\delta^{-1})^2 \\
	& + 2\lambda_{\rm max}(P_1)^2 D_X^4 (3\sigma+D_{\rm app})^2(t-1)^2,\qquad t\ge 1,
\end{align*}
with probability at least $1-\delta$.

Thus, a summation argument yields, for any $\delta>0$,
\begin{align*}
	\sum\limits_{t=1}^{\tau(\varepsilon,\delta)} L(\hat{\theta}_t) - L(\theta^*) \le\ & 2D_X^2(\|\hat{\theta}_1 - \theta^*\| + 3 \lambda_{\rm max}(P_1) D_X \sigma \ln\delta^{-1})^2 \tau(\varepsilon,\delta) \\
	& + \lambda_{\rm max}(P_1)^2 D_X^4 (3\sigma+D_{\rm app})^2 \frac{(\tau(\varepsilon,\delta)-1)\tau(\varepsilon,\delta)(2\tau(\varepsilon,\delta)-1)}{3} \,,
\end{align*}
with probability at least $1-\delta$.
\end{proof}

\subsection{Definition of $\tau(\varepsilon,\delta)$}\label{app:linear_convergence}
We now focus on the definition of $\tau(\varepsilon,\delta)$. We first transcript the result of \cite{hsu2012random} to our notations in the following lemma.
\begin{lemma}
\label{lemma:resultHsu}
Provided that Assumptions \ref{ass:iid}, \ref{ass:existence} and \ref{ass:subgaussian} are satisfied, starting from any $\hat{\theta}_1\in\R^d$ and $P_1=p_1I,p_1>0$, we have, 
for any $0<\delta<e^{-2.6}$ and $t\ge 6 \frac{D_X^2}{\Lambda_{\rm min}}(\ln{d}+\ln\delta^{-1})$,
\begin{align*}
    \|\hat{\theta}_{t+1}-\theta^*\|_{\Sigma}^2 \le & \frac3t \left(\frac{\|\hat{\theta}_1-\theta^*\|^2}{2p_1} + \frac{D_X^2}{\Lambda_{\rm min}}D_{\rm app}^2\frac{4(1+\sqrt{8\ln\delta^{-1}})}{0.07^2} + \frac{3\sigma^2(d/0.035+\ln\delta^{-1})}{0.07}\right) \\
    & + \frac{12}{0.07^2t^2}\left(\frac{\|\hat{\theta}_1-\theta^*\|^2}{p_1}\frac{D_X^2}{\Lambda_{\rm min}}(1+\sqrt{8\ln\delta^{-1}})\right.\\
    & \qquad\qquad\qquad \left. +  \left(\frac{D_X}{\sqrt{\Lambda_{\rm min}}}(D_{\rm app}+D_X\|\theta^*\|)+\frac{\|\hat{\theta}_1-\theta^*\|}{\sqrt{2p_1}}\right)^2 (\ln\delta^{-1})^2\right) \,,
\end{align*}
with probability at least $1-4\delta$.
\end{lemma}

\begin{proof} 
We first observe that 
\begin{equation*}
	\arg\min\limits_{w\in\R^d} \frac{1}{t} \sum\limits_{s=1}^{t} (y_s-w^TX_s)^2 + \lambda \|w - \hat{\beta}_{1}\|^2 = \arg\min\limits_{w\in\R^d} \frac{1}{t} \sum\limits_{s=1}^{t} (y_s-\hat{\beta}_1^TX_s-w^TX_s)^2 + \lambda \|w\|^2 \,,
\end{equation*}
therefore we apply ridge analysis of \cite{hsu2012random} to $(X_s,y_s-\hat{\beta}_1^TX_s)$. We note that $(y_s-\hat{\beta}_1^TX_s)$ has the same variance proxy and the same approximation error, it only amounts to translate the optimal $w$, that is denoted by $\beta$.

For any $\lambda>0$, we observe that
$d_{2,\lambda}\le d_{1,\lambda}\le d$,
$\rho_{\lambda}\le \frac{D_X}{\sqrt{d_{1,\lambda}\Lambda_{\rm min}}}$ and $b_{\lambda}\le \rho_{\lambda}(D_{\rm app}+D_X\|\beta-\hat{\beta}_1\|)$. Therefore we can apply Theorem 16 of \cite{hsu2012random}: for $0<\delta<e^{-2.6}$ and $t\ge 6 \frac{D_X}{\sqrt{\Lambda_{\rm min}}}(\ln(d)+\ln\delta^{-1})$, the following holds with probability $1-4\delta$: $\|\hat{\beta}_{t+1,\lambda}-\beta\|_{\Sigma}^2=3(\|\beta_{\lambda}-\beta\|_{\Sigma}^2+\varepsilon_{\rm bs}+\varepsilon_{\rm vr})$, with
\begin{align*}
    & \varepsilon_{\rm bs} \le \frac{4}{0.07^2}\Big(\frac{\frac{D_X^2}{\Lambda_{\rm min}}\mathbb{E}[(\mathbb{E}[y\mid X]-\beta^TX)^2]+(1+\frac{D_X^2}{\Lambda_{\rm min}})\|\beta_{\lambda}-\beta\|_{\Sigma}^2}{t}(1+\sqrt{8\ln\delta^{-1}}) \\
    & \qquad \qquad \qquad + \frac{(\frac{D_X}{\sqrt{\Lambda_{\rm min}}}(D_{\rm app}+D_X\|\beta-\hat{\beta}_1\|)+\|\beta_{\lambda}-\beta\|_{\Sigma})^2}{t^2}(\ln\delta^{-1})^2\Big)\,, \\
    & \delta_f \le \frac{1}{\sqrt{t}}\frac{D_X}{\sqrt{\Lambda_{\rm min}}}(1+\sqrt{8\ln\delta^{-1}}) + \frac1t \frac{4\sqrt{\frac{D_X^4}{\Lambda_{\rm min}^2d}+1}}{3}\ln\delta^{-1}\,, \\
    & \varepsilon_{\rm vr} \le \frac{\sigma^2d(1+\delta_f)}{0.07^2t} + \frac{2\sigma^2\sqrt{d(1+\delta_f)\ln\delta^{-1}}}{0.07^{3/2}t} + \frac{2\sigma^2\ln\delta^{-1}}{0.07 t}\,.
\end{align*}

Moreover $\mathbb{E}[(\mathbb{E}[y\mid X]-\beta^TX)^2]\le D_{\rm app}^2$ and $\Lambda_{\rm min}\le D_X^2$, hence, using $\|\beta_{\lambda}-\beta\|_{\Sigma}\le \lambda \|\beta-\hat{\beta}_1\|$ we transfer the result in our KF notations, that is, $\hat{\theta}_t=\hat{\beta}_{t,p_1^{-1}/2(t-1)},\hat{\beta}_1=\hat{\theta}_1,\beta=\theta^*$. We obtain, for any $0<\delta<e^{-2.6}$ and $t\ge 6 \frac{D_X}{\sqrt{\Lambda_{\rm min}}}(\ln(d)+\ln\delta^{-1})$,
\begin{align*}
    & \varepsilon_{\rm bs} \le \frac{4}{0.07^2}\Big(\frac{\frac{D_X^2}{\Lambda_{\rm min}}D_{\rm app}^2+\frac{D_X^2}{\Lambda_{\rm min}}\frac{\|\hat{\theta}_1-\theta^*\|^2}{p_1t}}{t}(1+\sqrt{8\ln\delta^{-1}}) \\
    & \qquad \qquad \qquad + \frac{(\frac{D_X}{\sqrt{\Lambda_{\rm min}}}(D_{\rm app}+D_X\|\theta^*\|)+\frac{\|\hat{\theta}_1-\theta^*\|}{\sqrt{2p_1t}})^2}{t^2}(\ln\delta^{-1})^2\Big)\,, \\
    & \delta_f \le \frac{1}{\sqrt{t}}\frac{D_X}{\sqrt{\Lambda_{\rm min}}}(1+\sqrt{8\ln\delta^{-1}}) + \frac1t \frac{4\sqrt{\frac{D_X^4}{\Lambda_{\rm min}^2d}+1}}{3}\ln\delta^{-1}\,, \\
    & \varepsilon_{\rm vr} \le \frac{\sigma^2d(1+\delta_f)}{0.07^2t} + \frac{2\sigma^2\sqrt{d(1+\delta_f)\ln\delta^{-1}}}{0.07^{3/2}t} + \frac{2\sigma^2\ln\delta^{-1}}{0.07 t}\,, \\
    & \|\hat{\theta}_{t+1}-\theta^*\|_{\Sigma}^2 \le 3\left(\frac{\|\hat{\theta}_1-\theta^*\|^2}{2p_1t}+\varepsilon_{\rm bs}+\varepsilon_{\rm vr}\right)\,,
\end{align*}
with probability at least $1-4\delta$.
For $t\ge \frac{D_X^2}{\Lambda_{\rm min}}\ln\delta^{-1}$, as $\ln\delta^{-1}\ge 1$, we get
\begin{equation*}
	\delta_f \le \frac{1}{\sqrt{6\ln\delta^{-1}}}(1+\sqrt{8\ln\delta^{-1}})+\frac16\frac{4}{3}\sqrt{\frac{1}{d}+1} \le \frac{1 + \sqrt{8}}{\sqrt{6}} + \frac{2\sqrt{2}}{9} \approx 1.9\le 2\,.
\end{equation*}
Thus, as $\sqrt{ab}\le \frac{a+b}{2}$ for any $a,b>0$, we have
\begin{align*}
	\varepsilon_{\rm vr} & \le \frac{\sigma^2}{0.07 t} \left(\frac{3d}{0.07} + 2\sqrt{\frac{3d\ln\delta^{-1}}{0.07}} + 2\ln\delta^{-1}\right) \\
	& \le \frac{\sigma^2}{0.07 t} \left(\frac{6d}{0.07} + 3\ln\delta^{-1}\right)\\
	& \le \frac{3\sigma^2(d/0.035 + \ln\delta^{-1})}{0.07 t}\,.
\end{align*}
It yields the result.
\end{proof}

Lemma \ref{lemma:resultHsu} allows the definition of an explicit value for $\tau(\varepsilon,\delta)$, as displayed in the following Corollary.
\begin{corollary}\label{coro:tau_delta_linear}
Assumption \ref{ass:localized} is satisfied for $\tau(\varepsilon,\delta)=\max(\tau_1(\delta),\tau_2(\varepsilon,\delta),\tau_3(\varepsilon,\delta))$ where we define
\begin{align*}
    & \tau_1(\delta) = \max\left(12 \frac{D_X^2}{\Lambda_{\rm min}}(\ln{d}+\ln\delta^{-1}), \frac{48D_X^2}{\Lambda_{\rm min}} \ln\frac{24D_X^2}{\Lambda_{\rm min}}\right) \,, \\
    & \tau_2(\varepsilon,\delta) = \frac{24\varepsilon^{-1}}{\Lambda_{\rm min}}\left(\frac{\|\hat{\theta}_1-\theta^*\|^2}{2p_1} + \frac{D_X^2}{\Lambda_{\rm min}}D_{\rm app}^2\frac{4(1+\sqrt{8\ln\delta^{-1}})}{0.07^2} + \frac{3\sigma^2(d/0.035+\ln\delta^{-1})}{0.07}\right) \\
    & \qquad\qquad \ln{\frac{12\varepsilon^{-1}}{\Lambda_{\rm min}}\left(\frac{\|\hat{\theta}_1-\theta^*\|^2}{2p_1} + \frac{D_X^2}{\Lambda_{\rm min}}D_{\rm app}^2\frac{4(1+\sqrt{8\ln\delta^{-1}})}{0.07^2} + \frac{3\sigma^2(d/0.035+\ln\delta^{-1})}{0.07}\right)} \,, \\
    & \tau_3(\varepsilon,\delta) = \sqrt{\frac{96\varepsilon^{-1}}{0.07^2\Lambda_{\rm min}}} \Bigg(\frac{\|\hat{\theta}_1-\theta^*\|^2}{p_1}\frac{D_X^2}{\Lambda_{\rm min}}(1+\sqrt{8\ln\delta^{-1}}) \\
    & \qquad\qquad\qquad\qquad\qquad +  \bigg(\frac{D_X}{\sqrt{\Lambda_{\rm min}}}(D_{\rm app}+D_X\|\theta^*\|)+\frac{\|\hat{\theta}_1-\theta^*\|}{\sqrt{2p_1}}\bigg)^2 (\ln\delta^{-1})^2\Bigg)^{1/2} \\
    & \qquad\qquad \ln \frac{96\varepsilon^{-1}}{0.07^2\Lambda_{\rm min}}\Bigg(\frac{\|\hat{\theta}_1-\theta^*\|^2}{2p_1}(1+\frac{D_X^2}{\Lambda_{\rm min}})(1+\sqrt{8\ln\delta^{-1}}) \\
    & \qquad\qquad\qquad\qquad\qquad+  \bigg(\frac{D_X}{\sqrt{\Lambda_{\rm min}}}(D_{\rm app}+D_X\|\theta^*\|)+\frac{\|\hat{\theta}_1-\theta^*\|}{\sqrt{2p_1}}\bigg)^2 (\ln\delta^{-1})^2\Bigg) \,.
\end{align*}
\end{corollary}

We recall that for any $\eta\le 1$, we have $\frac{\ln t}{t}\le \eta$ for $t\ge 2\eta^{-1}\ln(\eta^{-1})$, and we use it in the following proof.
\begin{proof}\textbf{of Corollary \ref{coro:tau_delta_linear}.} 
We define $\delta_t = \delta / t^2$ for any $t\ge 1$. In order to apply Lemma \ref{lemma:resultHsu} with a union bound, we need $t\ge 6 \frac{D_X^2}{\Lambda_{\rm min}}(\ln{d}+\ln\delta_t^{-1})$.
If $t\ge 12 \frac{D_X^2}{\Lambda_{\rm min}}(\ln{d}+\ln\delta^{-1})$ and $t\ge \frac{48D_X^2}{\Lambda_{\rm min}} \ln\frac{24D_X^2}{\Lambda_{\rm min}}$, we obtain
\begin{align*}
	t & \ge \frac{t}{2} + \frac{\sqrt{t}}{2}\sqrt{t}\\
	& \ge 6 \frac{D_X^2}{\Lambda_{\rm min}}(\ln{d}+\ln\delta^{-1}) + \frac{12D_X^2}{\Lambda_{\rm min}} \ln t,\qquad \text{ as } \ln t\le \sqrt{t}\\
	& = 6 \frac{D_X^2}{\Lambda_{\rm min}}(\ln{d}+\ln\delta_t^{-1})\,.
\end{align*}
Therefore, we define $\tau_1(\delta) = \max\left(12 \frac{D_X^2}{\Lambda_{\rm min}}(\ln{d}+\ln\delta^{-1}), \frac{48D_X^2}{\Lambda_{\rm min}} \ln\frac{24D_X^2}{\Lambda_{\rm min}}\right)$, and we apply Lemma \ref{lemma:resultHsu}. We get the simultaneous property
\begin{align*}
    \|\hat{\theta}_{t+1}-\theta^*\|_{\Sigma}^2 \le & \frac3t \left(\frac{\|\hat{\theta}_1-\theta^*\|^2}{2p_1} + \frac{D_X^2}{\Lambda_{\rm min}}D_{\rm app}^2\frac{4(1+\sqrt{8\ln\delta_t^{-1}})}{0.07^2} + \frac{3\sigma^2(d/0.035+\ln\delta_t^{-1})}{0.07}\right) \\
    & + \frac{12}{0.07^2t^2}\left(\frac{\|\hat{\theta}_1-\theta^*\|^2}{p_1}\frac{D_X^2}{\Lambda_{\rm min}}(1+\sqrt{8\ln\delta_t^{-1}})\right.\\
    & \qquad\qquad\qquad \left. +  \bigg(\frac{D_X}{\sqrt{\Lambda_{\rm min}}}(D_{\rm app}+D_X\|\theta^*\|)+\frac{\|\hat{\theta}_1-\theta^*\|}{\sqrt{2p_1}}\bigg)^2 (\ln\delta_t^{-1})^2\right),\qquad t\ge\tau_1(\delta),
\end{align*}
with probability at least $1-4\delta\sum\limits_{t\ge\tau_1(\delta)} t^{-2} \ge 1-\delta$ because $\tau_1(\delta)>4$.

Thus, as $\ln{t}\ge 1$ for $t\ge\tau_1(\delta)$ and $\|\hat{\theta}_{t+1}-\theta^*\|_{\Sigma}^2\ge \Lambda_{\rm min} \|\hat{\theta}_{t+1}-\theta^*\|^2$, we obtain
\begin{align*}
    \|\hat{\theta}_{t+1}-\theta^*\| \le & \frac{6\ln t}{\Lambda_{\rm min} t} \left(\frac{\|\hat{\theta}_1-\theta^*\|^2}{2p_1} + \frac{D_X^2}{\Lambda_{\rm min}}D_{\rm app}^2\frac{4(1+\sqrt{8\ln\delta^{-1}})}{0.07^2} + \frac{3\sigma^2(d/0.035+\ln\delta^{-1})}{0.07}\right) \\
    & + \frac{48(\ln t)^2}{0.07^2\Lambda_{\rm min} t^2}\left(\frac{\|\hat{\theta}_1-\theta^*\|^2}{p_1}\frac{D_X^2}{\Lambda_{\rm min}}(1+\sqrt{8\ln\delta^{-1}})\right.\\
    & \qquad\qquad\qquad\qquad \left. +  \bigg(\frac{D_X}{\sqrt{\Lambda_{\rm min}}}(D_{\rm app}+D_X\|\theta^*\|)+\frac{\|\hat{\theta}_1-\theta^*\|}{\sqrt{2p_1}}\bigg)^2 (\ln\delta^{-1})^2\right),\qquad t\ge\tau_1(\delta),
\end{align*}
with probability at least $1-\delta$. Finally, both terms of the last inequality are bounded by $\varepsilon/2$.
\end{proof}

From Corollary \ref{coro:tau_delta_linear}, we obtain the asymptotic rate by comparing $\tau_2(\delta)$ and $\tau_3(\delta)$.
We write $\tau_2(\delta) = 2 A_2(\delta)\ln A_2(\delta),\tau_3(\delta) = 2 A_3(\delta)\ln A_3(\delta)$ with
\begin{align*}
    & A_2(\delta) \lesssim \frac{\varepsilon^{-1}}{\Lambda_{\rm min}}\left(\frac{\|\hat{\theta}_1-\theta^*\|^2}{p_1} + \frac{D_X^2}{\Lambda_{\rm min}}D_{\rm app}^2 \sqrt{\ln\delta^{-1}} + \sigma^2(d+\ln\delta^{-1})\right) \\
    & A_3(\delta) \lesssim \sqrt{\frac{\varepsilon^{-1}}{\Lambda_{\rm min}} \Bigg(\frac{\|\hat{\theta}_1-\theta^*\|^2}{p_1}\frac{D_X^2}{\Lambda_{\rm min}}\sqrt{\ln\delta^{-1}} +  \bigg(\frac{D_X}{\sqrt{\Lambda_{\rm min}}}(D_{\rm app}+D_X\|\theta^*\|)+\frac{\|\hat{\theta}_1-\theta^*\|}{\sqrt{p_1}}\bigg)^2 (\ln\delta^{-1})^2\Bigg)}  \,.
\end{align*}
where the symbol $\lesssim$ means less than up to universal constants. As $\sqrt{a+b}\lesssim \sqrt{a}+\sqrt{b}$ and $\sqrt{ab}\lesssim a+b$, we obtain
\begin{align*}
    A_3(\delta) & \lesssim \sqrt{\frac{\varepsilon^{-1}}{\Lambda_{\rm min}}} \Bigg(\sqrt{\frac{\|\hat{\theta}_1-\theta^*\|^2}{p_1}\frac{D_X^2}{\Lambda_{\rm min}}\sqrt{\ln\delta^{-1}}} +  \bigg(\frac{D_X}{\sqrt{\Lambda_{\rm min}}}(D_{\rm app}+D_X\|\theta^*\|)+\frac{\|\hat{\theta}_1-\theta^*\|}{\sqrt{p_1}}\bigg)\ln\delta^{-1}\Bigg) \\
    &  \lesssim  \sqrt{\frac{\varepsilon^{-1}}{\Lambda_{\rm min}}} \Bigg(\frac{\|\hat{\theta}_1-\theta^*\|^2}{p_1} +\frac{D_X^2}{\Lambda_{\rm min}}\sqrt{\ln\delta^{-1}} + \bigg(\frac{D_X}{\sqrt{\Lambda_{\rm min}}}(D_{\rm app}+D_X\|\theta^*\|)+\frac{\|\hat{\theta}_1-\theta^*\|}{\sqrt{p_1}}\bigg)\ln\delta^{-1}\Bigg) \,.
\end{align*}

Thus, as long as $\frac{\varepsilon^{-1}}{\Lambda_{\rm min}}\le 1$, we get
\begin{align*}
    & A_2(\delta),A_3(\delta) \lesssim \frac{\varepsilon^{-1}}{\Lambda_{\rm min}}\Bigg(\frac{\|\hat{\theta}_1-\theta^*\|^2}{p_1} + \frac{D_X^2}{\Lambda_{\rm min}}(1+D_{\rm app}^2) \sqrt{\ln\delta^{-1}} + \sigma^2d \\
    & \qquad\qquad\qquad\qquad\qquad + \bigg(\frac{D_X}{\sqrt{\Lambda_{\rm min}}}(D_{\rm app}+D_X\|\theta^*\|)+\frac{\|\hat{\theta}_1-\theta^*\|}{\sqrt{p_1}} + \sigma^2\bigg)\ln\delta^{-1}\Bigg)  \,.
\end{align*}

\end{document}